\documentclass[letterpaper]{article} 
\usepackage{aaai24}  
\usepackage{times}  
\usepackage{helvet}  
\usepackage{courier}  
\usepackage[hyphens]{url}  
\usepackage{graphicx} 
\usepackage{natbib}  
\usepackage{caption} 
\usepackage{algorithm}
\usepackage{algorithmic}

\usepackage{epsfig}
\usepackage{amsmath}
\usepackage{amssymb}
\usepackage{booktabs}
\usepackage{multirow}
\usepackage{makecell}
\usepackage{xcolor,colortbl}
\usepackage{amsthm}

\def\ourmodel{LSAP }
\def\ourmodelc{LSAP-C}
\def\ourmodelh{LSAP-H}
\def\ourmodell{LSAP-L}

\usepackage{newfloat}
\usepackage{listings}
\DeclareCaptionStyle{ruled}{labelfont=normalfont,labelsep=colon,strut=off} 
\floatstyle{ruled}
\newfloat{listing}{tb}{lst}{}
\floatname{listing}{Listing}

\title{Learning to Approximate Adaptive Kernel Convolution on Graphs}
\author{
    Jaeyoon Sim\textsuperscript{\rm 1}, 
    Sooyeon Jeon\textsuperscript{\rm 1}, 
    InJun Choi\textsuperscript{\rm 1}, 
    Guorong Wu\textsuperscript{\rm 2}, 
    Won Hwa Kim\textsuperscript{\rm 1}
}
\affiliations{
    \textsuperscript{\rm 1}Pohang University of Science and Technology, Pohang, South Korea\\
    \textsuperscript{\rm 2}University of North Carolina at Chapel Hill, Chapel Hill, USA\\
    
    \{simjy98, jsuyeon, surung9898, wonhwa\}@postech.ac.kr, guorong\_wu@med.unc.edu
}

\usepackage{bibentry}

\begin{document}

\maketitle

\begin{abstract}
Various Graph Neural Networks (GNNs) have been successful in 
analyzing data in non-Euclidean spaces, however, 
they have limitations such as oversmoothing, i.e., 
information becomes excessively averaged as the number of hidden layers increases.  
The issue stems from the intrinsic formulation of 
conventional graph convolution where the nodal features are 
aggregated from a direct neighborhood per layer across the entire nodes in the graph. 
As setting different number of hidden layers per node is infeasible, recent works
leverage a diffusion kernel to redefine the graph structure and incorporate information from farther nodes. 
Unfortunately, such approaches suffer from heavy diagonalization of a graph Laplacian or 
learning a large transform matrix. 
In this regards, we propose a diffusion learning framework, 
where the range of feature aggregation is controlled by 
the scale of a diffusion kernel. 
For efficient computation, we derive closed-form derivatives of approximations of the graph convolution
with respect to the scale, so that node-wise range can be adaptively learned. 
With a downstream classifier, the entire framework is made trainable in an end-to-end manner. 
Our model is tested on various standard datasets for node-wise classification for the state-of-the-art performance, and 
it is also validated on a real-world brain network data for graph classifications to demonstrate 
its practicality for Alzheimer classification. 
\end{abstract}

\section{Introduction}
\label{sec:intro}

Graph Neural Network (GNN) has been heavily recognized in machine learning and computer vision 
with various practical applications such as 
text classification \cite{huang2019text}, neural machine translation \cite{bastings2017graph}, 3D shape analysis \cite{wei2020view, verma2018feastnet}, semantic segmentation \cite{qi20173d,xie2021scale}, social information systems \cite{lin2021medley} and speech recognition \cite{liu2016graph}. 
At the heart of these GNN models lies the graph convolution, 
which develops useful representation of each node 
with a filtering operation on the graph. 
Given a graph comprised of a set of nodes/edges 
and signal defined on its nodes, in \cite{kipf2016semi}, 
a convolutional layer was introduced as  
a linear combination of the signal within a direct neighborhood of each node using the topology of the graph, and its equivalence with spectral filtering \cite{hammond2011wavelets} has been shown. 
Stacking these convolutional layers (together with non-linear activations) constitutes the fundamental Graph Convolutional Network (GCN) \cite{kipf2016semi}, and testing a newly developed GCN on classifying node-wise labels has become a standard benchmark 
to validate the GCN model \cite{chen2018fastgcn, wu2019simplifying, xu2020graph, chen2020simple, yang2021selfsagcn}. 

Within the architecture of conventional GCN and its variants, there is a fundamental issue:  
each convolution layer gathers information within a direct neighborhood {\em uniformly} across all nodes. 
When information aggregation from direct neighborhood is not sufficient, 
the convolution layers are stacked to seek for useful information from a larger neighborhood. 
Such an architecture with several convolution layers broadens the range of neighborhood for information aggregation {\em uniformly}, again, across the entire nodes in the graph. 
Eventually, when the same information is shared across all the nodes, it leads to ``oversmoothed'' representation of each node, not being able to characterize one from another. 
This behavior can be easily interpreted from the spectral perspective, 
as a filtering operation in the spectral space will uniformly affect all nodes in the graph space. 

Perhaps the most intuitive solution to the oversmoothing is to 
use different ranges of neighborhood per node. 
However, it is difficult to design such a model with conventional graph convolution, 
as it would require different number of convolution layers for each node which is highly impractical. 
Several recent works tried to overcome this locality issue. 
Methods in \cite{velivckovic2017graph, kim2022find} leverage attention 
to capture long-range relationships 
among nodes, authors in \cite{gao2019graph, wu2022structural} develop pooling scheme to compress graphs, and authors in \cite{chen2020simple} improve upon the vanilla GCN with skip connection of residuals as in ResNet \cite{he2016deep}. 

Notably, GraphHeat \cite{xu2020graph} used a diffusion kernel
to 
redefine distances between nodes as a heat diffusion process among the nodes along the graph structure. 
It easily connects many local nodes within a range (i.e., scale) even though they are not directly connected. 
Such an approach is quite effective when the homophily condition is reasonably held 
even if the given edges in a graph may be noisy. 
But still, the scale was defined as a hyperparameter and the same 
range was arranged 
across the entire graph. 
Later, a framework that trains on the scale according to a target task was introduced with 
the gradient of a loss with respect to the scale for each node. 
While using diffusion kernel have shown to be quite effective, 
such an approach is computationally inefficient with heavy diagonalization of graph Laplacian, especially when dealing with a large or a population of graphs. 
Other diffusion-based model such as \cite{zhao2021adaptive} uses diffusion on layers and channels of features instead of nodes, and Graph Neural Diffusion (GRAND) \cite{chamberlain2021grand} 
trains on large weight matrices for attention which can be computationally burdening.

In this regime, we propose an efficient framework 
that learns adaptive scales for each node of a 
graph with approximations, i.e., {\bf L}earning {\bf S}cales via {\bf AP}proximation (LSAP). 
The key idea is to train on the node-wise {\em range of neighborhood} 
instead of excessively stacking convolutional layers. 
For this, we first show that the formulation in \cite{xu2020graph} can be 
defined as a heat kernel convolution, 
which can be approximated with various polynomials \cite{huang2020fast}. 
We then derive the derivatives of the expansion coefficients of these polynomials in the scale space, which can be used to define task-specific gradients  
to train the optimal scales {\em within the approximation} instead of learning exhaustive transform matrices. 
LSAP achieves novel node-wise representation adaptively by leveraging features from other nodes 
within a trained ``range'' defined by the diffusion kernel. 
The ideas above lead to the following contributions: %

\begin{itemize}
\item We propose a GNN with an adaptive diffusion kernel whose approximations are 
trainable in an end-to-end manner at each node,
\item We derive closed-form derivatives of various polynomial coefficients 
with respect to the range (i.e., scale) so that graph convolution can be efficiently trained, 
\item Learning on scales provides interpretable results on the semantics of each node, validated on two independent datasets with different tasks. 

\end{itemize}
\ourmodel demonstrates superior results on the node classification task in a semi-supervised learning setting \cite{shchur2018pitfalls}, 
as well as on a graph classification task performed on a population of brain networks to predict diagnostic labels for Alzheimer's Disease (AD). 
Especially in the AD experiment, the trained scales delineate specific regions highly responsible for the prediction of AD for interpretability. 

\section{Related Works}
\label{sec:related work}

\noindent\textbf{Graph Neural Networks.} 
The vanilla GCN \cite{kipf2016semi} and Variant GNNs utilize graph convolution to perform feature aggregation from neighbors. 
Simplifying Graph Convolution (SGC) \cite{wu2019simplifying} captures high-order information in the graph with K-th power of the adjacency matrix, 
GCNII \cite{chen2020simple} extends a vanilla GCN 
with residual connection and identity mapping \cite{he2016deep}, and 
Graph Attention Network (GAT) \cite{velivckovic2017graph} introduces attention mechanism on graphs 
to assign relationships to different nodes. 
Personalized Propagation of Neural Prediction and its Approximation (APPNP) \cite{klicpera2018predict} improved message propagation based on personalized PageRank \cite{page1999pagerank}, 
Graph Random Neural Network (GRAND) \cite{feng2020graph} developed graph data augmentation, and
Deep Adaptive Graph Neural Network (DAGNN) \cite{liu2020towards} 
disentangled representation transform and message propagation to construct a deep model.

Also, there are recent works that discover useful graph structures from data to adaptively update the structures for message passing. 
DIAL-GNN \cite{chen2019deep} jointly learned graph structure and embeddings by iteratively searching for hidden graph structures,
Bayesian GCNN \cite{zhang2019bayesian} incorporated uncertain graph information through a parametric random graph model, and 
NodeFormer \cite{wu2022nodeformer} proposed an efficient message passing scheme for propagating layer-wise node signals.

\noindent\textbf{Graph Neural Networks with Diffusion.}
There are several previous works that adopt diffusion on graphs for GNNs \cite{lee2023time, zhang2023apegnn, huang2023node}. 
Graph Diffusion Convolution (GDC) \cite{klicpera2019diffusion} introduced spatially localized graph convolution to aggregate information of indirect nodes,
Adaptive Diffusion Convolution (ADC) \cite{zhao2021adaptive} learned a global radius applied on different layers and channels of features,
GRAND \cite{chamberlain2021grand} defined diffusion PDEs on graphs and trained on weight matrices to learn attention 
for diffusivity, and 
Fast and Scalable Network Representation Learning (ProNE) \cite{2019ProNE} focused on effective network embedding with 
spectral propagation for enhancement, 
where they train a single global scale in the spectral propagation.
Our methods differ from these methods in that it uses an adaptive parametric kernel at individual nodes,
and we propose a new optimization scheme 
on the scales within its polynomial approximations. 
Also, many of methods above can be adopted for graph classification as well \cite{velivckovic2017graph,klicpera2019diffusion} with an additional layer transforming node embeddings to a graph embedding. 
Other literature, although not fully discussed in this section, 
will be introduced and used as the baselines to compare the performances with our proposed model for both node and graph classification tasks later. 

\section{Preliminaries}

\label{sec:prelim}

\noindent\textbf{Graph Convolution with Heat Kernel.} 
An undirected graph $G=\{V, E\}$ comprises a node set $V$ with $|V|=N$ and an edge set $E$. 
A graph $G$ is often represented as a symmetric adjacency matrix $A$ of which individual elements $a_{pq}$ encodes 
connectivity information between node $p$ and $q$. 
A graph Laplacian is defined as $L=D-A$ where $D$ is a degree matrix, i.e., a diagonal matrix with  $D_{pp}=\sum_{q}A_{pq}$. 
Since $L$ is positive semi-definite, it has a complete set of orthonormal basis $U = [u_{1} | u_{2} | ... | u_{N}]$ known as Laplacian eigenvectors and corresponding real and non-negative eigenvalues $0=\lambda_{1} \leq \lambda_{2} \leq ... \leq \lambda_{N}$.
A normalized Laplacian 
is defined as 
$\hat{L}=I_{N}-D^{-1/2}AD^{-1/2}$ where $I_N$ is an identity matrix.
Since $\hat{L}$ is real symmetric, it also has a complete set of eigenvectors and eigenvalues.

In \cite{chung1997spectral}, the heat kernel between nodes $p$ and $q$ 
is defined in the spectral domain spanned by $U$ as 
\begin{equation}
\footnotesize
    h_s(p, q) = \sum_{i=1}^{N}e^{-s\lambda_{i}}u_{i}(p)u_{i}(q)
    \label{eq:heat_kernel_between_nodes}
\end{equation}
where $u_i$ is $i$-th eigenvector of the graph Laplacian, and 
the kernel $e^{-s\lambda_i}$ captures smooth transition between the nodes as a diffusion process within the scale $s$. 
Using convolutional theorem \cite{oppenheim1997signals}, 
graph Fourier transform, i.e., $\hat{x} = U^Tx$, 
offers a way to define the graph convolution $\ast$ of a signal $x(p)$ with a filter $h_s$. 
Using Eq. \eqref{eq:heat_kernel_between_nodes}, heat kernel convolution with $h_s$ as a low-pass filter is defined as
\begin{equation}
\footnotesize
    h_s \ast x(p) = \sum_{i=1}^{N}e^{-s\lambda_{i}}\hat{x}(i)u_{i}(p)
    \label{eq:heat_kernel_conv0}
\end{equation}
whose band-width is controlled by the scale $s$.

\noindent\textbf{Approximation of Convolution with Heat Kernel.} 
The exact computation 
of Eq. \eqref{eq:heat_kernel_conv0} requires 
diagonalization of a graph Laplacian which can be computationally challenging. 
Existing literature uses Chebyshev polynomial as a basis to approximate the kernel convolution 
as a linear transform \cite{he2022convolutional}. 
In \cite{huang2020fast}, 
approximation of heat kernel convolution was introduced using several orthogonal polynomials 
such as Chebyshev, Hermite and Laguerre.  
The analytic 
solutions to the polynomial 
coefficients $c_{s,n}$ for scale $s$ were derived for  
Chebyshev polynomial $P^T_n$, Hermite polynomial $P^H_n$ and Laguerre polynomial $P^L_n$, 
where $n$ denotes the degree of each polynomial.

A polynomial $P_n \in \{P^T_n, P^H_n, P^L_n\}$ is often defined by a second order recurrence as 
\begin{equation}
\footnotesize
    P_{n+1}(\lambda) = (\alpha_n\lambda + \beta_n)P_n(\lambda) + \gamma_n P_{n-1}(\lambda)
    \label{eq:second order recurrence}
\end{equation}
where initial conditions $P_{-1}(\lambda)=0$ and $P_0(\lambda)=1$ for $n\geq0$ and 
parameters $\alpha_n$, $\beta_n$ and $\gamma_n$ determine the type of polynomial. 
Then, the heat kernel $e^{-s\lambda}$ can be defined with polynomials $P_n$ and expansion coefficients $c_{s,n}$ as
\begin{equation}
\footnotesize
    e^{-s\lambda}=\sum_{n=0}^{\infty}c_{s,n}P_n(\lambda),
    \label{eq:redefined_heat_kernel}
\end{equation}

Now, the solution to the heat diffusion 
in Eq. \eqref{eq:heat_kernel_conv0} can be expressed in terms of $P_n$ and $c_{s,n}$ 
via Eq. \eqref{eq:redefined_heat_kernel} 
as
\begin{equation}
\footnotesize
    h_s \ast x(p) = \sum_{n=0}^{\infty}c_{s,n}\sum_{j=1}^{N}P_n(\lambda_j)\hat{x}(j)u_j(p). 
    \label{eq:heat_kernel_convolution_approximation0}
\end{equation}
Since $\hat{L}u_j=\lambda_ju_j$, 
the Eq. \eqref{eq:heat_kernel_convolution_approximation0} 
can be further written as 
\begin{equation}
\footnotesize
    h_s \ast x(p) = \sum_{n=0}^{\infty}c_{s,n}P_n(\hat{L})x(p)
    \label{eq:heat_kernel_convolution_approximation1}
\end{equation}
where initial conditions $P_{-1}(\hat{L})x(p)=0$ and $P_0(\hat{L})x(p)=x(p)$ from the second order recurrence. 
Notice that Eq. \eqref{eq:heat_kernel_convolution_approximation1} represents the convolution operation 
as a simple linear combination of $c_{s,n}$ and $P_n$ without $u_j$, and it is often approximated 
at the order of $m$ 
for practical purposes. 

\section{Learning to Approximate Kernel Convolution}

\begin{figure*}[!t]
    \centering  
    \includegraphics[width=.86\linewidth]{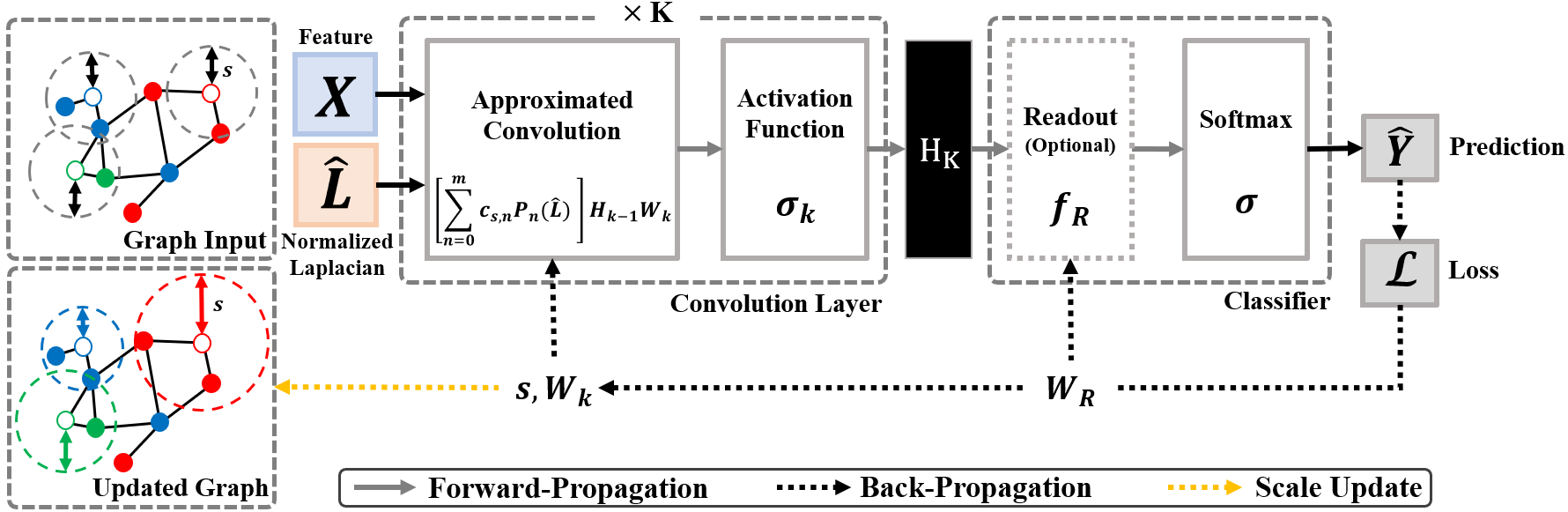}
    \caption{ 
    Illustration of LSAP. 
    A graph (as normalized Laplacian $\hat L$) and node feature $X$ are inputted to the convolution layer. 
    The output $H_K$ 
    is inputted to a 
    downstream classifier which yields a prediction $\hat Y$. The loss from $\hat Y$ is backpropagated to update the classifier and convolution approximation with $\textbf{s} = [s_1,\dots,s_N]$ to adaptively adjust the scale of each node.
    }
    \label{fig:Overall Architecture}
\end{figure*}

We introduce our model, i.e., LSAP, 
that {\em trains on the approximations} for the optimal range of neighborhood at individual nodes. 
The derivatives to train $s$ in Eq.~\eqref{eq:heat_kernel_convolution_approximation1} 
for two separate tasks, i.e., node classification and graph classification, are derived in their closed-forms.

\subsection{Model Architecture}

In most of GCN frameworks \cite{kipf2016semi, yang2021selfsagcn}, 
a convolution operation at the $k$-th layer is given as   
\begin{equation}
\footnotesize
    H_k=\sigma_k(\tilde A H_{k-1}W_k) 
    \label{eq:gc}
\end{equation}
where $\tilde A$ is a normalized adjacency matrix, $W_k$ is a trainable weight matrix, and $\sigma_k$ is a non-linear activation function. 
It takes an input $H_{k-1}$ and outputs a new representation $H_{k}$. 
Each convolution operation takes features from direct neighbors of each node according to $\tilde A$ 
to design the new node representation. 
As more convolution layers are stacked,  
the range of aggregation is uniformly increased for all nodes 
causing the infamous ``oversmoothing'' issue. 
To alleviate oversmoothing, 
we propose to utilize a diffusion kernel, i.e., a heat kernel, on the graph 
to define ranges of neighborhood for individual nodes, so that 
the node-wise range is adaptively 
defined to avoid the oversmoothing. 

The architecture of \ourmodel is given in Fig. \ref{fig:Overall Architecture}. 
The overall components are similar to the original GCN \cite{kipf2016semi}, 
however, \ourmodel redefines the convolution with approximated heat kernel. 
Consider a graph $G$ given as 
a Laplacian $\hat L$, feature $X$ defined on its nodes and 
either node-wise or graph-wise label $Y$. 
Our framework takes $\hat L$ and $X$ as inputs,  
performs approximated heat kernel convolutions and outputs a prediction $\hat Y$. 
The $\hat Y$ is then compared with the ground truth $Y$, and the error is backpropagated 
to update model parameters including the scale $s$.

\subsubsection{Convolution Layer.} 
Based on Eq. \eqref{eq:heat_kernel_convolution_approximation1}, Eq. \eqref{eq:gc} is reformulated 
by replacing the normalized adjacency matrix $\tilde A$ with 
the heat kernel with polynomial approximation as
\begin{equation}
\footnotesize
    H_k=\sigma_k([\sum_{n=0}^m c_{s,n}P_n(\hat{L})]H_{k-1}W_k). 
    \label{eq:Heat_Kernel_Forard_Propagation}
\end{equation}
While $\tilde A$ let the model combine information from nodes within 1-hop distance only, 
Eq. \eqref{eq:Heat_Kernel_Forard_Propagation} let it 
aggregate information within a ``range'' of each node defined by $s$ 
within $c_{s,n}$. 
The Eq. \eqref{eq:Heat_Kernel_Forard_Propagation} defines a convolution layer, 
and we can stack $K$ of them to achieve a better representation of the original $X$. 

\subsubsection{Output Layer.} 
The output layer 
yields a prediction of $Y$, i.e., $\hat Y$, via softmax. 
Depending on the task, it may include a simple Multi-layer Perceptron (MLP) that can be trained. 
A loss $L_{\textbf{err}}$ is computed at this layer which quantifies the error between $\hat Y$ and $Y$ using cross-entropy for
different tasks, e.g., node classification or graph classification. 

\subsubsection{Model Update.} 
The loss $L_{\textbf{err}}$ 
is backpropagated to update the model parameters, i.e., $W_k$ and $s$. 
To maintain the scale positive, an $\ell_1$-norm regularization on $s$ 
is imposed.  
The overall objective function $\mathcal{L}$ is given as 
\begin{equation}
\footnotesize
    \mathcal{L} = \mathcal{L}_{\textbf{err}}+\alpha|{\bf s}|
    \label{eq:Loss_Function}
\end{equation}
where $\alpha$ is a hyperparameter. 
The weight $W_k$ can be easily trained with backpropagation,  
and a multi-variate $\textbf s$ across all nodes can be also
trained 
given a gradient on scale $s$ as
\begin{equation}
\footnotesize
    \textbf{s}\gets \textbf{s}-\beta_{s}\frac{\partial \mathcal{L}}{\partial \textbf s}
    \label{eq:Scale_Update}
\end{equation}
where $\beta_s$ is a learning rate. 
It requires $\frac{\partial \mathcal{L}}{\partial \textbf{s}}$ to make the framework trainable, 
and  
we derive the $\frac{\partial \mathcal{L}}{\partial \textbf{s}}$ 
in a closed-form to train on the approximations of 
Eq. \eqref{eq:Heat_Kernel_Forard_Propagation} in the following. 

\subsection{Gradients of Polynomial Coefficients with Scale}

\label{sec:poly}
We denote expansion coefficients as $c_{\textbf{s},n}^{T}$, $c_{\textbf{s},n}^{H}$ and $c_{\textbf{s},n}^{L}$ that correspond to $P^T_n$, $P^H_n$ and $P^L_n$. 
As introduced in \cite{huang2020fast}, one can obtain a solution to the heat diffusion by obtaining expansion coefficient with each polynomial in $P_n$. 
In order to design a gradient-based ``learning'' framework of node-wise range (i.e., scale) based on these expansions, 
we derived gradients of loss $\frac{\partial \mathcal{L}}{\partial \textbf{s}}$ in closed-forms.   
This is an essential component of \ourmodel as it let the model efficiently train without diagonalization of $\hat L$. 
The $\frac{\partial \mathcal{L}}{\partial \textbf{s}}$ can be achieved using the chain rule in a traditional way, 
and to obtain $\frac{\partial \mathcal{L}}{\partial \textbf{s}}$ in terms of the $H_k$, we compute $\frac{\partial c_{\textbf{s},n}}{\partial \textbf{s}}$ for each coefficient.

\subsubsection{Chebyshev Polynomial.} 
The recurrence relation and expansion coefficient for Chebyshev polynomial are given as
\begin{equation}
\footnotesize
     \begin{aligned}
     P^T_{n+1}(\hat{L}) &= (2-\delta_{n0})\hat{L} P^T_n(\hat{L})-P^T_{n-1}(\hat{L}),\\
     c_{\textbf{s},n}^T&=(2-\delta_{n0})(-1)^ne^{-\frac{\textbf{s}b}{2}}I_n\left(\frac{\textbf{s}b}{2}\right) 
     \end{aligned}
 \label{eq:Chebyshev_Polynomial_Coefficient}
\end{equation}
where $b>0$ is a hyper-parameter, and $I_n$ is the {\em modified Bessel function of the first kind} \cite{olver2010nist}. 

\subsubsection{Hermite Polynomial.}
The recurrence relation and expansion coefficient for Hermite polynomial are written as
\begin{equation}
\footnotesize
 \begin{aligned}
 P^H_{n+1}(\hat{L}) &= 2\hat{L} P^H_n(\hat{L})-2n P^H_{n-1}(\hat{L}),\\
 c_{\textbf{s},n}^H&=\frac{1}{n!}\left(\frac{-\textbf{s}}{2}\right)^n e^{\frac{\textbf{s}^2}{4}},
 \end{aligned}
\label{eq:Hermite_Polynomial_Coefficient}
\end{equation}

\subsubsection{Laguerre Polynomial.}
For Laguerre polynomial, the recurrence relation and expansion coefficient are 
\begin{equation}
\footnotesize
 \begin{aligned}
 P^L_{n+1}(\hat{L})&=\frac{(2n+1-\hat{L})P^L_n(\hat{L}) - n P^L_{n-1}(\hat{L})}{n+1},\\
 c_{\textbf{s},n}^L&=\frac{\textbf{s}^n}{(\textbf{s}+1)^{n+1}}. 
 \end{aligned}
\label{eq:Lagurre_Polynomial_Coefficient}
\end{equation}

Notice that all the expansion coefficients above are defined by $\textbf{s}$. 
If they are differentiable with respect to $\textbf{s}$, then we do not need to learn expensive parameters but 
simply train on these polynomial approximations with $\textbf{s}$ directly. 

\newtheorem{lemma}{Lemma}
\begin{lemma}
Consider an orthogonal polynomial $P_n$ over interval $[a,b]$ with inner product $\int_a^b P_n(\lambda)P_k(\lambda)w(\lambda)d\lambda=\delta_{nk}$, where $w(\lambda)$ is the weight function. If $ P_n$ expands the heat kernel, the expansion coefficients $c_{\textbf{s},n}$ with respect to $\textbf s$ are differentiable and $\frac{\partial{c_{\textbf{s},n}}}{\partial{\textbf{s}}} = -\int_a^b \lambda e^{-\bf{s}\lambda}P_n(\lambda)w(\lambda)d\lambda$.
\label{lem:1}
\end{lemma}

Lemma \ref{lem:1} (with a proof in the supplementary) says that 
the $c_{\textbf{s},n}$ in  Eq. \eqref{eq:Chebyshev_Polynomial_Coefficient}, \eqref{eq:Hermite_Polynomial_Coefficient} and \eqref{eq:Lagurre_Polynomial_Coefficient} have the derivatives with respect to $\textbf{s}$. 
These will be used in the following two sections to define 
gradients on loss functions for different tasks. 

\subsection{Semi-supervised Node Classification}
\label{sec:NC}

The goal of node classification is to predict labels of unlabeled nodes based on the information from other nodes.
The output layer (after $K$-convolution layers) of \ourmodel produces a prediction $\hat Y = \sigma(H_K)$, and 
the training should be performed to reduce the error between $\hat Y$ and the true $Y$. 
\begin{lemma}
Let a graph convolution be operated by Eq. \eqref{eq:Heat_Kernel_Forard_Propagation}, which approximates 
the convolution with $P_n$ and $c_{\textbf{s},n}$. 
If a loss $\mathcal{L}_{\textbf{err}}$ for node-wise classification 
is defined as cross-entropy between a 
prediction $\hat Y = \sigma(H_K)$, where $\sigma(\cdot)$ is a softmax function, and the true $Y$, then
\begin{equation}
\footnotesize
    \begin{aligned}
    \frac{\partial \mathcal{L}_\textbf{err}}{\partial \textbf s}=&(\hat{Y}-Y) \times  \sigma^{'}_k([\sum_{n=0}^mc_{\textbf{s},n}P_n(\hat{L})]H_{k-1}W_k) W_k^{\mathsf{T}}\\ &\times(\sum_{n=0}^m P_n(\hat{L})H_{k-1}^{\mathsf{T}}+[\sum_{n=0}^m c_{\textbf{s},n}P_n(\hat{L})]\frac{\partial H_{k-1}}{\partial c_{\textbf{s},n}}) \frac{\partial c_{\textbf{s},n}}{\partial \textbf{s}}
    \label{eq:dl_ds_NC}
    \end{aligned}
\end{equation}
where $\sigma^{'}_k$ is the derivative of $\sigma_k$.
\label{lem:2}
\end{lemma}
Lemma \ref{lem:2} let \ourmodel backpropagate the error to update $\textbf{s}$ 
to obtain the optimal scale per node for node classification. 

\subsection{Graph Classification}

\label{sec:GC}
Consider a population of graphs $\{G_t\}_{t=1}^T$ with corresponding labels $\{Y_t\}_{t=1}^T$, 
and learning a graph classification model finds a function $f(G_t) = Y_t$. 
For this, 
the output layer 
consists of a readout layer $f_R(\cdot)$ (i.e., MLP with ReLU) and a softmax $\sigma(\cdot)$ at the end to construct pseudo-probability for each class. 
The $f_R(\cdot)$ with weights $W_R$ takes the output $H_K$ from the convolution layers as an input and returns $H_R$ as
\begin{equation}
\footnotesize
    H_R=f_R(H_K;W_R),
    \label{eq:Readout_Layer}
\end{equation}
and prediction $\hat{Y} = \sigma(H_R)$.  

\begin{lemma}
Let $H_k$ from Eq. \eqref{eq:Heat_Kernel_Forard_Propagation} be a graph convolution with a heat kernel with polynomial $P_n$ and coefficients $c_{\textbf{s},n}$. 
If a loss $\mathcal{L}_{\textbf{err}}$ for classifying graph-wise label 
is defined as cross-entropy between a 
prediction $\hat Y = \sigma(H_R)$, where $\sigma(\cdot)$ is a softmax function, and the true $Y$, then
\begin{equation}
\footnotesize
    \begin{aligned}
    \frac{\partial \mathcal{L}_\textbf{err}}{\partial \textbf{s}}=&(\hat{Y}-Y) \times \frac{\partial H_R}{\partial H_K} \times \sigma^{'}_k([\sum_{n=0}^m c_{\textbf{s},n}P_n(\hat{L})]H_{k-1}W_k) W_k^{\mathsf{T}} \\&\times(\sum_{n=0}^m P_n(\hat{L})H_{k-1}^{\mathsf{T}}+[\sum_{n=0}^m c_{\textbf{s},n}P_n(\hat{L})]\frac{\partial H_{k-1}}{\partial c_{\textbf{s},n}}) \frac{\partial c_{\textbf{s},n}}{\partial \textbf{s}}
    \label{eq:dl_ds_GC}
    \end{aligned}
\end{equation}
where $\sigma^{'}_k$ is the derivative of $\sigma_k$.
\label{lem:3}
\end{lemma}

Lemma \ref{lem:3} let \ourmodel adaptively update the scale $\textbf{s}$ across all nodes for each node using Eq.~\eqref{eq:Scale_Update} 
towards the minimal error for predicting graph-wise label. 

\section{Experiments}
\label{sec:experiments}

In this section, we compare the performances of \ourmodel and baselines 
on node classification and graph classification. 
\textbf{\ourmodelc}, \textbf{\ourmodelh} and \textbf{\ourmodell} correspond to approximation frameworks with $P^T_n$, $P^H_n$ and $P^L_n$, 
and the model with exact computation of the heat kernel convolution 
is referred as \textbf{Exact}. 
For the node classification, we used conventional benchmarks for semi-supervised learning task \cite{shchur2018pitfalls}. For the graph classification, we investigated real brain network data from Alzheimer's Disease Neuroimaging Initiative (ADNI) to classify different diagnostic stages towards Alzheimer's Disease (AD) for a practical application. 
These tasks on seven different benchmarks can demonstrate the feasibility of LSAP. 

\subsection{Semi-supervised Node Classification}

\subsubsection{Datasets.} 
We conducted experiments on 
standard node classification datasets (in Table \ref{tab:Datasets_NC})
that provide connected and undirected graphs. 
Cora, Citeseer and Pubmed \cite{sen2008collective} are constructed as citation networks,
Amazon Computer and Amazon Photo \cite{shchur2018pitfalls} define co-purchase networks,
and
Coauthor CS \cite{shchur2018pitfalls} is a co-authorship network. 

\begin{table}[!h]
\centering
    \renewcommand{\arraystretch}{1.0} 
    \renewcommand{\tabcolsep}{0.3cm}
    {
    \scalebox{0.8}{
    \begin{tabular}{c|cccc}
        \Xhline{3\arrayrulewidth}
        \textbf{Dataset} & \textbf{Nodes} & \textbf{Edges} & \textbf{Classes} & \textbf{Features}\\
        \hline
        Cora & 2,708 & 5,429 & 7 & 1,433  \\
        Citeseer & 3,327 & 4,732 & 6 & 3,703 \\
        Pubmed & 19,717 & 44,338 & 3 & 500 \\
        Amazon Computers & 13,752 & 245,861 & 10 & 767 \\
        Amazon Photo  & 7,650 & 119,081 & 8 & 745  \\
        Coauthor CS  & 18,333 & 81,894 & 15 & 6,805 \\
        \Xhline{3\arrayrulewidth}
    \end{tabular}}
    \caption{Summary of node classification datasets.}
    \label{tab:Datasets_NC}
}
\end{table}

\subsubsection{Setup.} 
We used the accuracy as the evaluation metric.
For Cora, Citseer and Pubmed data, eighteen different baselines were used to compare the results for the node classification task as listed in Table \ref{tab:Result_CCP}. These standard benchmarks are provided with fixed split of 20 nodes per class for training, 500 nodes for validation and 1000 nodes for testing 
as in other literature \cite{kim2022find, wu2022structural}. 

For Amazon and Coauthor datasets, seven baselines are used as in Table \ref{tab:Result_AAC}. 
For the MLP with 3-layers, GCN and 3ference, results are obtained from \cite{luo2022inferring}, 
and a result for DSF comes from \cite{guo2023graph}.
For others, the experiments were performed 
by randomly splitting the data as 60\%/20\%/20\% for training/validation/testing datasets as in \cite{luo2022inferring} 
and replicating it 10 times to obtain mean and standard deviation of the evaluation metric.

\begin{table}[!b]
\centering
\renewcommand{\arraystretch}{0.9} 
\renewcommand{\tabcolsep}{0.28cm}
\scalebox{0.77}{
\begin{tabular}{l|ccc}
    \Xhline{3\arrayrulewidth}
    \textbf{Model} & \textbf{Cora} & \textbf{Citeseer} & \textbf{Pubmed}\\
    \hline
    GCN \cite{kipf2016semi} & 81.50 & 70.30 & 78.60 \\
    GAT \cite{velivckovic2017graph} & 83.00 & 72.50 & 79.00 \\
    APPNP \cite{klicpera2018predict} & 83.30 & 71.80 & 80.10 \\
    $\text{GDC}^{\dag}$ \cite{klicpera2019diffusion} & 82.20 & 71.80 & 79.10 \\
    SGC \cite{wu2019simplifying} & 81.70 & 71.30 & 78.90 \\
    Bayesian GCN \cite{zhang2019bayesian} & 81.20 & 72.20 & - \\
    Shoestring \cite{lin2020shoestring} & 81.90 & 69.50 & 79.70 \\
    $\text{GraphHeat}^{\dag}$ \cite{xu2020graph} & 83.70 & 72.50 & 80.50 \\
    g-U-Nets \cite{gao2019graph} & 84.40 & 73.20 & 79.60 \\
    GCNII \cite{chen2020simple} & 85.50 & 73.40 & 80.30 \\
    GRAND \cite{feng2020graph} & 85.40 & 75.40 & 82.70 \\
    DAGNN \cite{liu2020towards} & 84.40 & 73.30 & 80.50 \\
    SelfSAGCN \cite{yang2021selfsagcn} & 83.80 & 73.50 & 80.70 \\
    DIAL-GNN \cite{chen2019deep} & 84.50 & 74.10 & - \\
    SuperGAT \cite{kim2022find} & 84.30 & 72.60 & 81.70 \\
    $\text{GRAND}^{\dag}$ \cite{chamberlain2021grand} & 83.60 & 74.10 & 78.80 \\
    $\text{ADC}^{\dag}$ \cite{zhao2021adaptive} & 84.50 & 74.50 & 83.00 \\
    SEP-N \cite{wu2022structural} & 84.80 & 72.90 & 80.20 \\
    \hline
    \cellcolor{gray!20}\textbf{\ourmodelc} & \cellcolor{gray!20}\textbf{87.90}  & \cellcolor{gray!20}\textbf{76.50} & \cellcolor{gray!20}83.30 \\
    \cellcolor{gray!20}\textbf{\ourmodelh} & \cellcolor{gray!20}85.00 & \cellcolor{gray!20}76.10    & \cellcolor{gray!20}82.60 \\
    \cellcolor{gray!20}\textbf{\ourmodell} & \cellcolor{gray!20}85.90 & \cellcolor{gray!20}75.90 &   \cellcolor{gray!20}\textbf{84.10} \\ 
    \cellcolor{gray!20}Exact & \cellcolor{gray!20}\underline{88.20} &  \cellcolor{gray!20}\underline{78.10} & \cellcolor{gray!20}\underline{85.30} \\
    \Xhline{3\arrayrulewidth}
\end{tabular}
}
\\
{\footnotesize $\dag$: graph diffusion-based models.}
\caption{
Accuracy (\%) on Cora, Citeseer, and Pubmed. 
LSAP yields better performances over existing baselines (in bold)  
similar to Exact achieving the best results (underline). 
}
\label{tab:Result_CCP}
\end{table}

\begin{table}[!t]
\centering
    \renewcommand{\arraystretch}{1.0} 
    \renewcommand{\tabcolsep}{0.45cm}
  {
  \scalebox{0.73}{
  \begin{tabular}{l|ccc}
    \Xhline{3\arrayrulewidth}
    \multirow{2}{*}{\textbf{Model}} & 
    \textbf{Amazon} & \textbf{Amazon} & \textbf{Coauthor} \\ 
    & \textbf{Computer} & \textbf{Photo} & \textbf{CS} \\
    \hline
    MLP (3-layers) & 84.63 & 91.96 & 95.63 \\
    GCN & 90.49 & 93.91 & 93.32 \\
    3ference & 90.74 & 95.05 & {\bf 95.99} \\
    GAT  & 91.18 $\pm$ 0.74  & 94.49 $\pm$ 0.54  & 93.42 $\pm$ 0.31  \\
    GDC & 86.03 $\pm$ 2.26  & 93.28 $\pm$ 1.03  & 92.68 $\pm$ 0.53  \\
    GraphHeat & 89.59 $\pm$ 3.15 & 94.04 $\pm$ 0.75 & 92.93 $\pm$ 0.20 \\
    DSF & 92.84 $\pm$ 0.10 & 95.73 $\pm$ 0.08 & - \\
    \hline 
    \cellcolor{gray!20}\textbf{\ourmodelc} & \cellcolor{gray!20}\underline{\textbf{94.43 $\pm$ 1.16}}  & \cellcolor{gray!20}95.96 $\pm$ 1.65  & \cellcolor{gray!20}\underline{94.81 $\pm$ 0.55}  \\
    \cellcolor{gray!20}\textbf{\ourmodelh} & \cellcolor{gray!20}93.64 $\pm$ 0.86 & \cellcolor{gray!20}\underline{\textbf{96.65 $\pm$ 0.67}} & \cellcolor{gray!20}93.52 $\pm$ 0.97 \\
    \cellcolor{gray!20}\textbf{\ourmodell} & \cellcolor{gray!20}92.76 $\pm$ 0.48 & \cellcolor{gray!20}95.35 $\pm$ 0.85 & \cellcolor{gray!20}93.58 $\pm$ 0.82 \\
    \cellcolor{gray!20}Exact & \cellcolor{gray!20}93.52 $\pm$ 0.65 & \cellcolor{gray!20}96.41 $\pm$ 1.54 & \cellcolor{gray!20}93.71 $\pm$ 1.16 \\
    \Xhline{3\arrayrulewidth}
  \end{tabular}}
\caption{
Mean node classification accuracy (\%) and s.d. on Amazon Computers, Amazon Photo, and Coauthor CS. 
The best results are in bold, and the best results within experiments with replicates are underlined. 
}
\label{tab:Result_AAC}
}
\end{table}

\subsubsection{Results.} 
Table \ref{tab:Result_CCP} and \ref{tab:Result_AAC} show 
the 
performance comparisons between \ourmodel and baseline models. 
As shown in Table \ref{tab:Result_CCP}, on the node classification benchmarks, 
learning node-wise adaptive scale performs the best in both Exact and its approximations. 
\ourmodel showed improved performance over existing models; exceeding previous best baseline performances by 
2.4\% (on Cora) and 1.1\% (on Citeseer and Pubmed) 
The similar performance of \ourmodel with that of Exact 
demonstrates accurate convolution approximation.  
Despite slight decreases, training on adaptive scales using \ourmodel was much faster. 

For additional datasets in Table \ref{tab:Result_AAC}, 
the performance of \ourmodel outperformed the baselines.
The results for MLP, GCN and 3ference were adopted from \cite{luo2022inferring}, 
which reported the best performance out of 10 replicated experiments.
We ran the same experiments 
for GAT, GDC, GraphHeat, Exact and LSAP, 
and the mean and standard deviation of metrics are given.
\ourmodel shows 
significant improvements on the Amazon Computer (94.43\%, \ourmodelc) and Amazon Photo (96.65\%, \ourmodelh). 
On the Coauther CS, 
we also achieve the highest mean accuracy (94.81\%, \ourmodelc) among the experiments with random replicates. 

\begin{table}[!b]
    \centering
    \renewcommand{\arraystretch}{1.1} 
    \renewcommand{\tabcolsep}{0.1cm}
    \scalebox{0.7}{
    \begin{tabular}{c|c||ccccc}
        \Xhline{3\arrayrulewidth}
        \textbf{Biomarker} & \textbf{Category} & \textbf{CN} & \textbf{SMC} & \textbf{EMCI} & \textbf{LMCI} & \textbf{AD}\\
        \hline
        \multirow{3}{*}{\shortstack{Cortical\\Thickness}} & \# of subjects & 359 & 181 & 437 & 180 & 166 \\
        \cline{2-7}
        & Gender (M / F) & 178 / 181 & 69 / 112 & 249 / 188 & 119 / 61 & 102 / 64 \\
        \cline{2-7}
        & Age (Mean$\pm$Std) & 72.8$\pm$1.4 & 72.0$\pm$5.2 & 71.0$\pm$7.9 & 70.9$\pm$6.1 & 74.8$\pm$8.7 \\
        \hline
        \multirow{3}{*}{FDG} & \# of subjects & 345 & 186 & 461 & 231 & 162 \\
        \cline{2-7}
        & Gender (M / F) & 173 / 172 & 66 / 120 & 262 / 199 & 152 / 79 & 102 / 60 \\
        \cline{2-7}
        & Age (Mean$\pm$Std) & 73.0$\pm$1.3 & 71.7$\pm$5.2 & 71.7$\pm$7.8 & 71.1$\pm$7.0 & 74.9$\pm$8.8 \\
        \Xhline{3\arrayrulewidth}
    \end{tabular}}
\caption{Demographics of the ADNI dataset.}
\label{tab:Datasets_ADNI}
\end{table}

\subsection{Graph Classification}
\label{sec:classification}

\subsubsection{Datasets.} 
Using the magnetic resonance images (MRI) from the ADNI data, 
each brain was partitioned into 148 cortical regions and 12 sub-cortical regions using Destrieux atlas \cite{destrieux2010automatic}, 
and tractography on diffusion-weighted imaging (DWI) was applied to calculate the number of white matter fibers connecting the 160 brain regions to construct $160\times 160$ structural network (i.e., graph). 
On the same parcellation, region-wise imaging features such as Standard Uptake Value Ratio (SUVR) of metabolism level from FDG-PET and cortical thickness from MRI were measured. For the SUVR normalization, Cerebellum was used as the reference. 
The dataset consists of 5 AD-specific progressive groups: Control (CN), Significant Memory Concern (SMC), Early Mild Cognitive Impairment (EMCI), Late Mild Cognitive Impairment (LMCI) and AD.
The demographics of ADNI dataset are summarized in Table \ref{tab:Datasets_ADNI}. 

\begin{table}[!t]
\centering
{
\renewcommand{\arraystretch}{1.0} 
\renewcommand{\tabcolsep}{0.15cm}
\scalebox{0.72}{
\begin{tabular}{c|l|ccc}
    \Xhline{3\arrayrulewidth}
    \multirow{2}{*}{\textbf{Feature}} & \multirow{2}{*}{\textbf{Model}} & \multicolumn{3}{c}{\textbf{Classification (ADNI)}}\\ \cline{3-5}
     & & \textbf{Accuracy (\%)} & \textbf{Precision} &  \textbf{Recall} \\ 
    \hline
    \multirow{10}{*}{\shortstack{\textbf{Cortical}\\\textbf{Thickness}}} 
    & SVM (Linear) & 82.39 $\pm$ 2.73 & 0.822 $\pm$ 0.033 & 0.852 $\pm$ 0.025 \\
    & MLP (2-layers) & 78.76 $\pm$ 2.21 & 0.792 $\pm$ 0.036 & 0.799 $\pm$ 0.026  \\
    & GCN & 61.37 $\pm$ 3.09 & 0.598 $\pm$ 0.025 & 0.626 $\pm$ 0.044 \\
    & GAT & 64.17 $\pm$ 5.46 & 0.627 $\pm$ 0.067 & 0.668 $\pm$ 0.046  \\
    & GDC & 77.10 $\pm$ 4.25 & 0.769 $\pm$ 0.050 & 0.785 $\pm$ 0.044  \\
    & GraphHeat & 70.90 $\pm$ 3.17 & 0.703 $\pm$ 0.030 & 0.718 $\pm$ 0.026  \\ 
    & ADC & 82.10 $\pm$ 2.41 & 0.776 $\pm$ 0.019 &	0.728 $\pm$ 0.067\\  \cline{2-5}
    & \cellcolor{gray!20}\textbf{\ourmodelc} & \cellcolor{gray!20}\textbf{87.00 $\pm$ 2.16} & \cellcolor{gray!20}\textbf{0.868 $\pm$ 0.027} & \cellcolor{gray!20}\textbf{0.885 $\pm$ 0.027} \\
    & \cellcolor{gray!20}\textbf{\ourmodelh} & \cellcolor{gray!20}85.41 $\pm$ 2.32 & \cellcolor{gray!20}0.859 $\pm$ 0.031 & \cellcolor{gray!20}0.867 $\pm$ 0.030 \\
    & \cellcolor{gray!20}\textbf{\ourmodell} & \cellcolor{gray!20}85.64 $\pm$ 1.86 & \cellcolor{gray!20}0.859 $\pm$ 0.022 & \cellcolor{gray!20}0.866 $\pm$ 0.022 \\
    & \cellcolor{gray!20}Exact & \cellcolor{gray!20}86.24 $\pm$ 1.96 & \cellcolor{gray!20}0.866 $\pm$ 0.017 & \cellcolor{gray!20}0.867 $\pm$ 0.023 \\
    \hline
    \multirow{10}{*}{\shortstack{\textbf{FDG}}} 
    & SVM (Linear) & 85.27 $\pm$ 2.09 & 0.857 $\pm$ 0.027 & 0.869 $\pm$ 0.021 \\
    & MLP (2-layers) & 87.51 $\pm$ 1.62 & 0.882 $\pm$ 0.024 & 0.882 $\pm$ 0.014  \\
    & GCN & 68.81 $\pm$ 1.95 & 0.677 $\pm$ 0.028 & 0.697 $\pm$ 0.025  \\
    & GAT & 69.24 $\pm$ 7.13 & 0.670 $\pm$ 0.106 & 0.736 $\pm$ 0.037  \\
    & GDC & 86.21 $\pm$ 3.24 & 0.867 $\pm$ 0.033 & 0.870 $\pm$ 0.029  \\
    & GraphHeat & 76.97 $\pm$ 2.42 & 0.775 $\pm$ 0.035 & 0.773 $\pm$ 0.010  \\ 
    & ADC & 88.60 $\pm$ 2.81 & 0.708 $\pm$ 0.062 &	0.753 $\pm$ 0.053\\  \cline{2-5}
    & \cellcolor{gray!20}\textbf{\ourmodelc} & \cellcolor{gray!20}89.24 $\pm$ 2.23 & \cellcolor{gray!20}0.895 $\pm$ 0.022 & \cellcolor{gray!20}0.904 $\pm$ 0.023 \\
    & \cellcolor{gray!20}\textbf{\ourmodelh} & \cellcolor{gray!20}90.11 $\pm$ 2.44 & \cellcolor{gray!20}0.903 $\pm$ 0.027 & \cellcolor{gray!20}0.910 $\pm$ 0.022 \\
    & \cellcolor{gray!20}\textbf{\ourmodell} & \cellcolor{gray!20}\textbf{90.40 $\pm$ 1.38} & \cellcolor{gray!20}\textbf{0.909 $\pm$ 0.018} & \cellcolor{gray!20}\textbf{0.914 $\pm$ 0.015} \\ 
    & \cellcolor{gray!20}Exact & \cellcolor{gray!20}90.18 $\pm$ 2.67 & \cellcolor{gray!20}0.907 $\pm$ 0.028 & \cellcolor{gray!20}0.907 $\pm$ 0.028 \\
    \Xhline{3\arrayrulewidth}
    \end{tabular}}
\caption{
Classification performances on ADNI dataset.
}
\label{tab:Result_ADNI}
}
\end{table}

\subsubsection{Setup.} 
5-way classification 
was designed to classify the different groups in Table \ref{tab:Datasets_ADNI}. 
5-fold cross validation was used to obtain unbiased results, 
and accuracy, precision, and recall in their mean were used as evaluation metrics. 
As the baseline, we adopted 
Linear Support Vector Machine (SVM), Multi-Layer Perceptron (MLP) with 2 layers, GCN \cite{kipf2016semi}, GAT \cite{velivckovic2017graph}, GDC \cite{klicpera2019diffusion}, GraphHeat \cite{xu2020graph} and ADC \cite{zhao2021adaptive}. 
Each sample is given with a graph (i.e., brain network) and two node features (i.e., cortical thickness and FDG measure) 
which are well-known as useful biomarkers for AD diagnosis.

\begin{figure}[!b]
    \centering
    \footnotesize{
    \setlength{\tabcolsep}{0pt}
    \renewcommand{\arraystretch}{0.7}
    \scalebox{0.90}{
        \begin{tabular}{cc}
        \raisebox{0\height}[0pt][0pt]{\textbf{Cora}} & \raisebox{0\height}[0pt][0pt]{\textbf{Citeseer}} \\
        \includegraphics[width=0.48\linewidth]{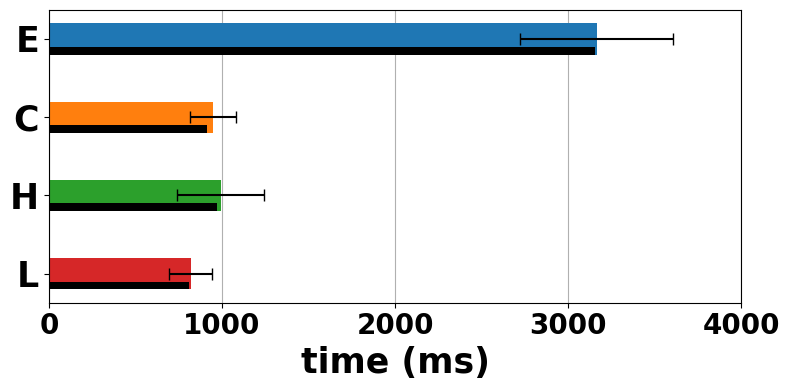} &
        \includegraphics[width=0.48\linewidth]{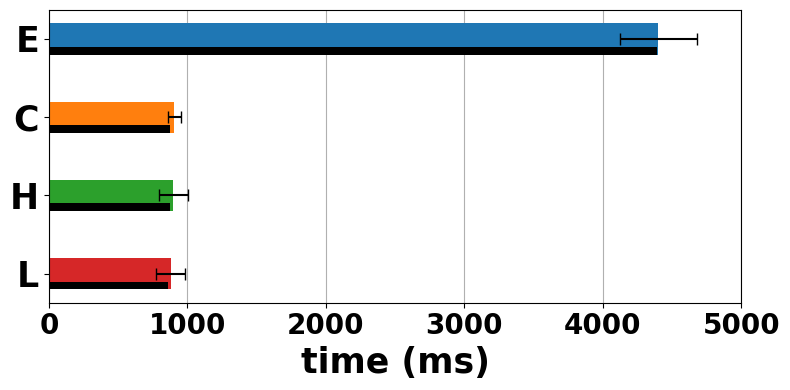} \\
        \raisebox{0\height}[0pt][0pt]{\textbf{Pubmed}} & \raisebox{0\height}[0pt][0pt]{\textbf{ADNI}}\\
        \includegraphics[width=0.48\linewidth]{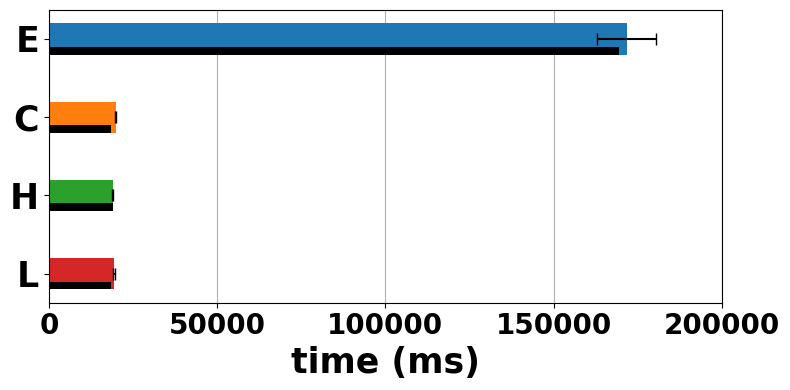} &
        \includegraphics[width=0.48\linewidth]{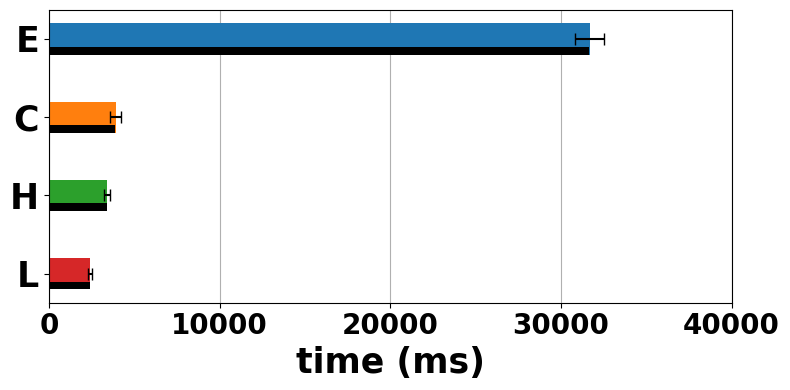} \\
        \multicolumn{2}{c}{\includegraphics[width=0.8\linewidth]{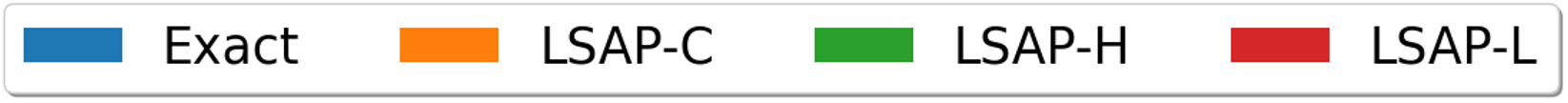}}\\
    \end{tabular}}
    }
    \caption{
    Comparisons of computation time (in {\em ms}) for one epoch (Forward and backpropagation). 
    Within the epoch, time for heat kernel convolution is given in black bar.  
    Results were obtained 
    with 10 repetitions.
    }
    \label{fig:Training_Time_NC}
\end{figure}

\begin{figure*}[t!]
  \centering
  \renewcommand{\arraystretch}{1.0}
  \renewcommand{\tabcolsep}{0.5cm}
  \normalsize{
  \scalebox{0.71}{
  \begin{tabular}{cccccl}
    \raisebox{.1\height}[0pt][0pt]{\textbf{Exact}} 
    & \raisebox{.1\height}[0pt][0pt]{\textbf{\ourmodelc}} 
    & \raisebox{.1\height}[0pt][0pt]{\textbf{\ourmodelh}} 
    & \raisebox{.1\height}[0pt][0pt]{\textbf{\ourmodell}} 
    & \\ 
    
    \includegraphics[width=0.21\linewidth]{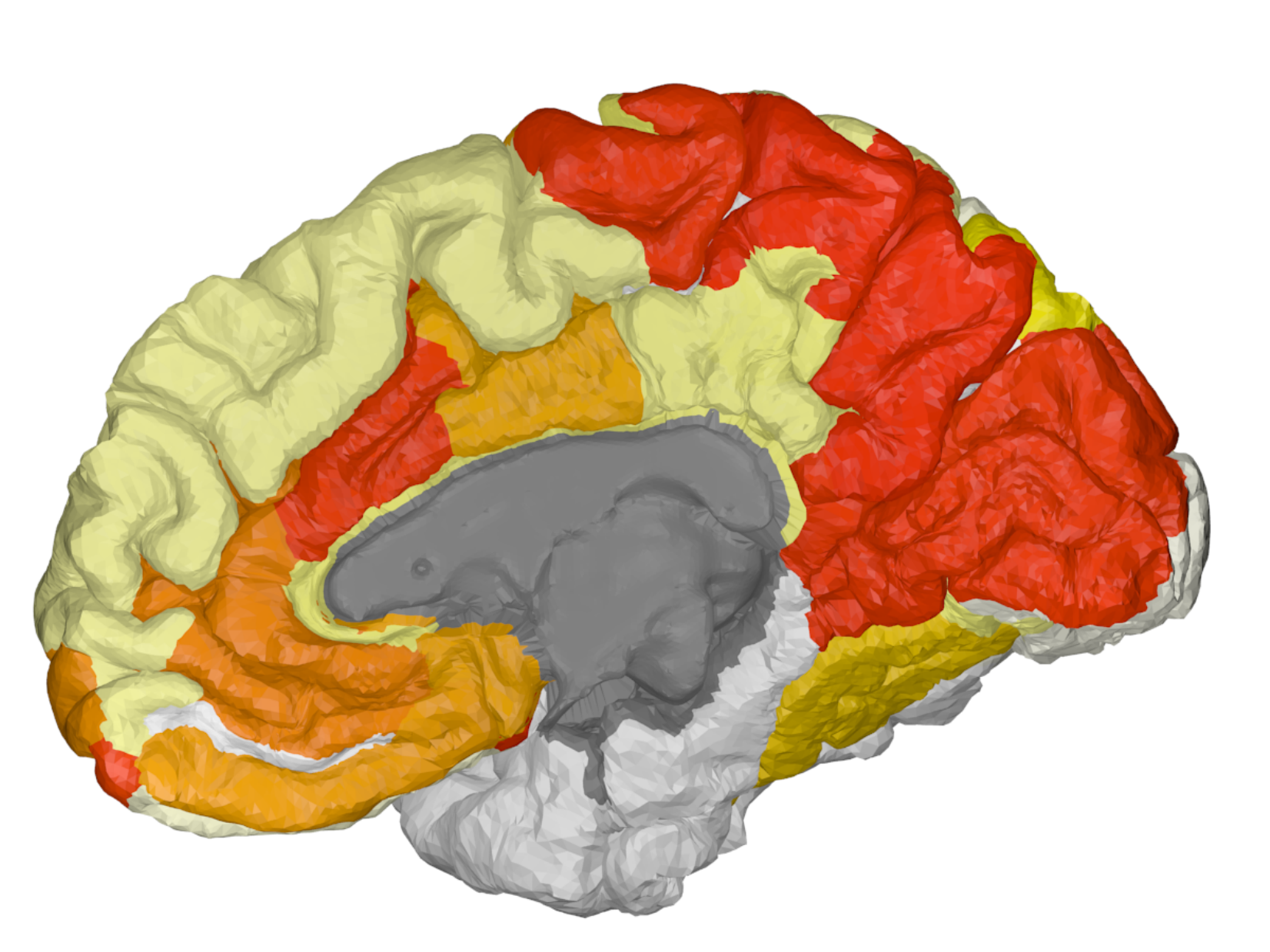} &
    \includegraphics[width=0.21\linewidth]{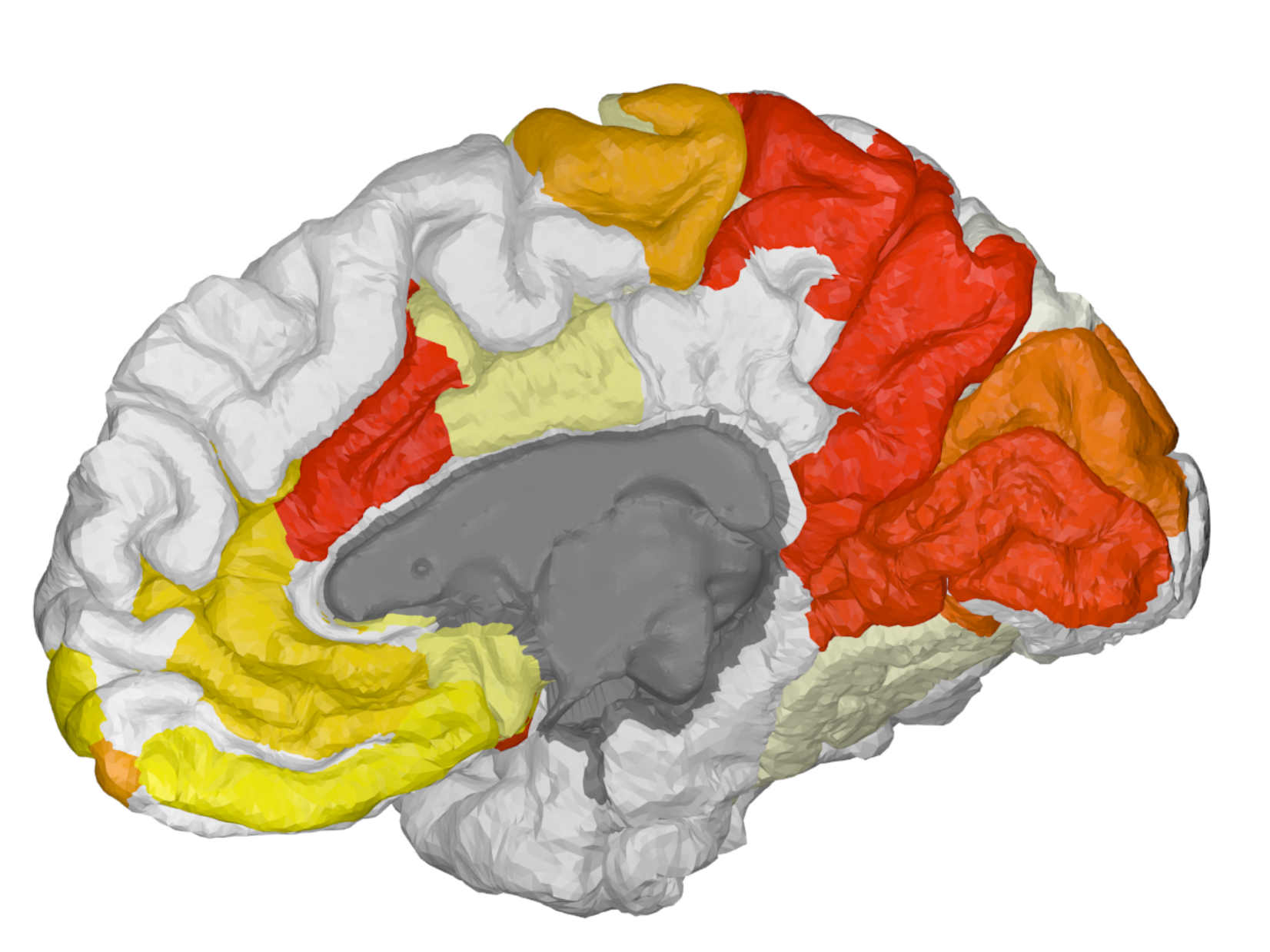} &
    \includegraphics[width=0.21\linewidth]{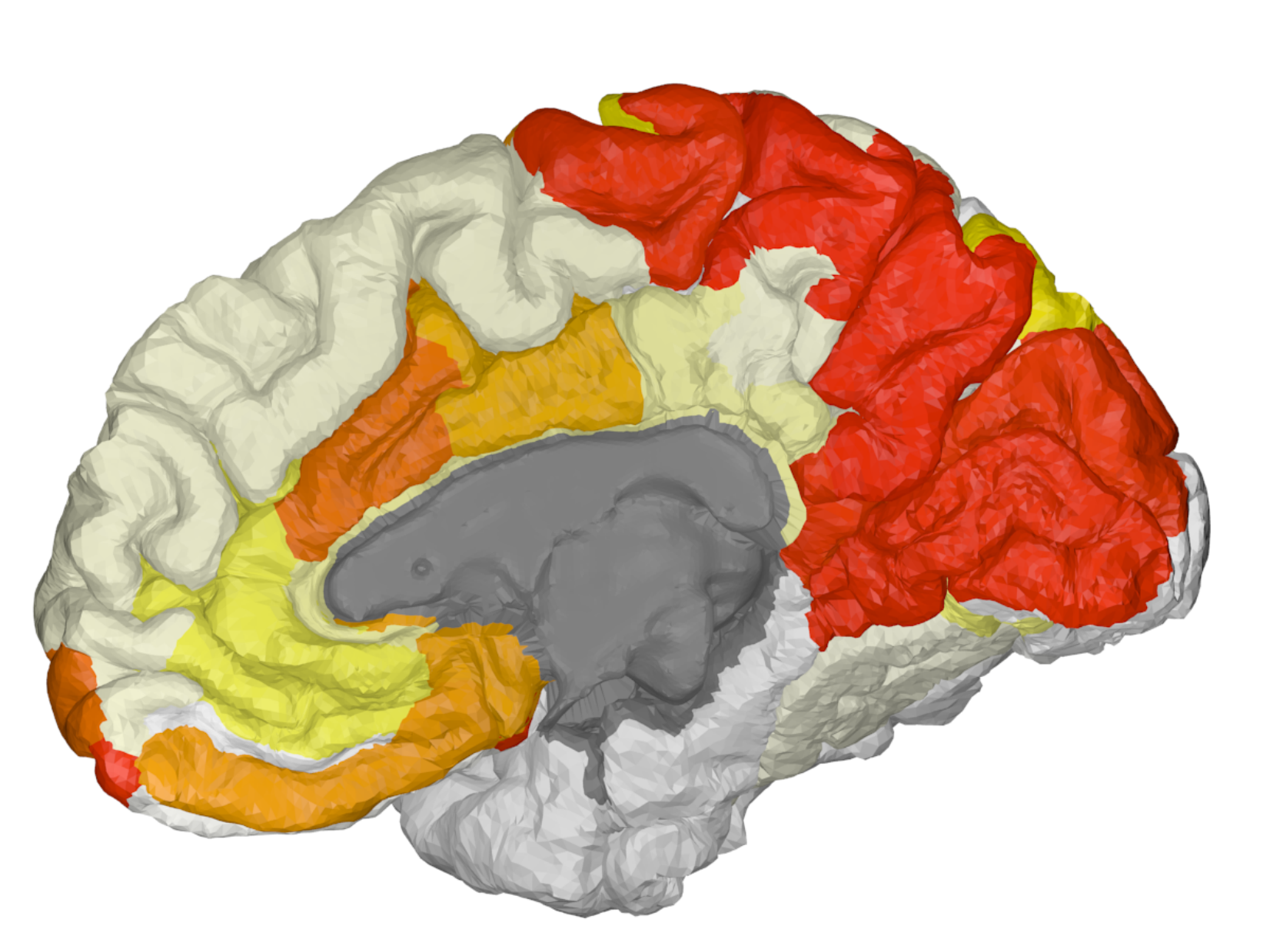} &
    \includegraphics[width=0.21\linewidth]{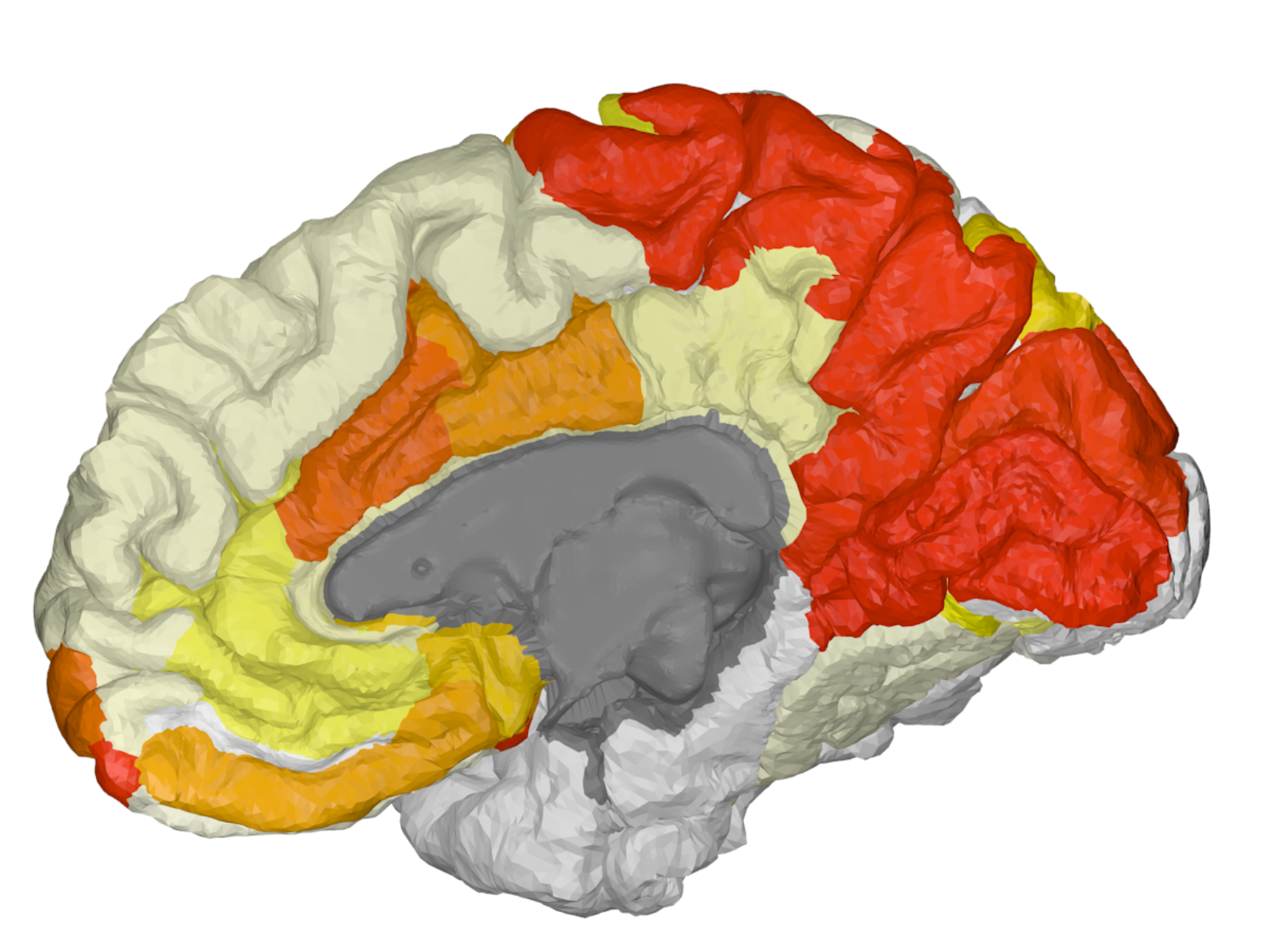} & \\
    
    \includegraphics[width=0.21\linewidth]{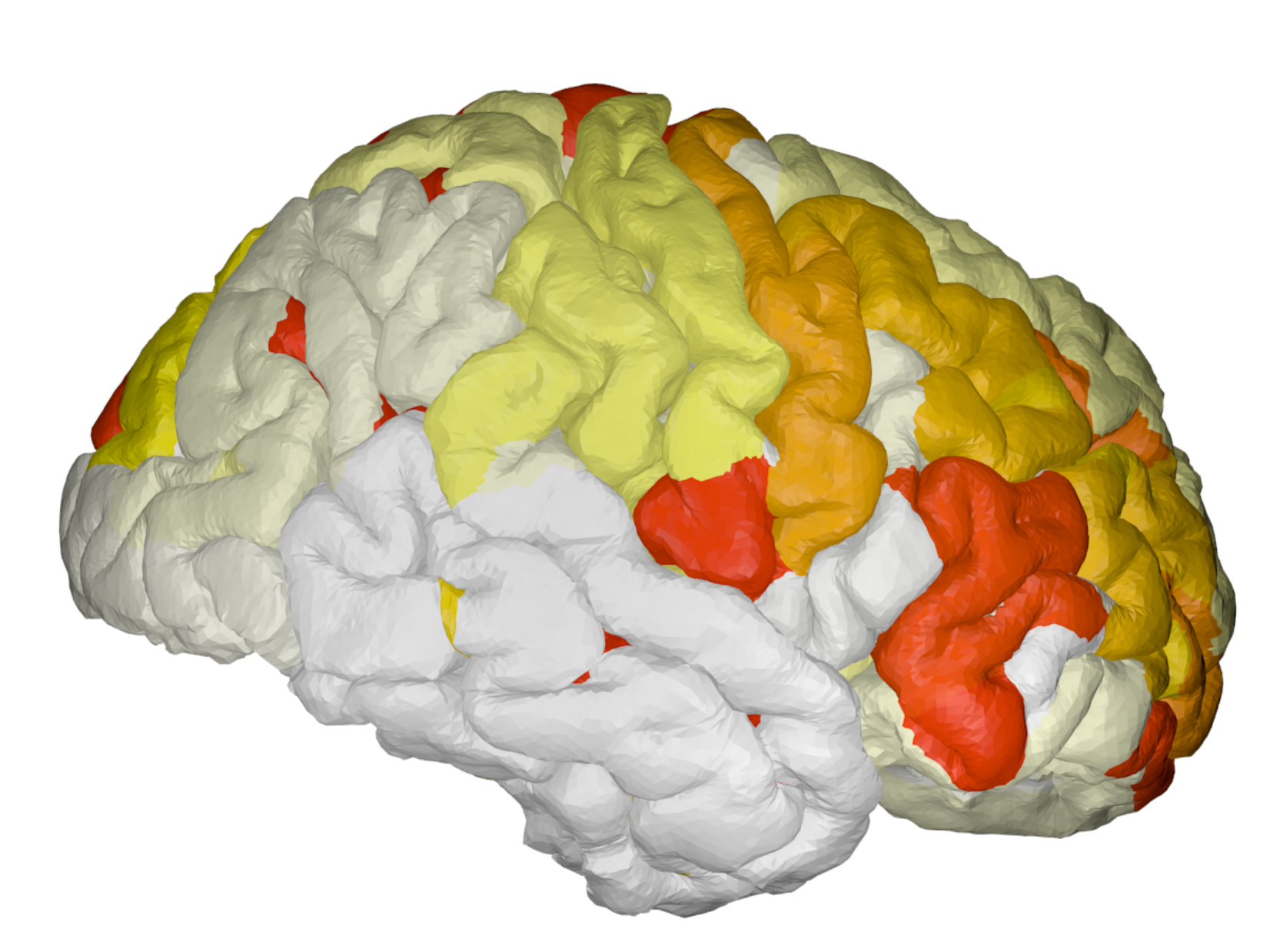}&
    \includegraphics[width=0.21\linewidth]{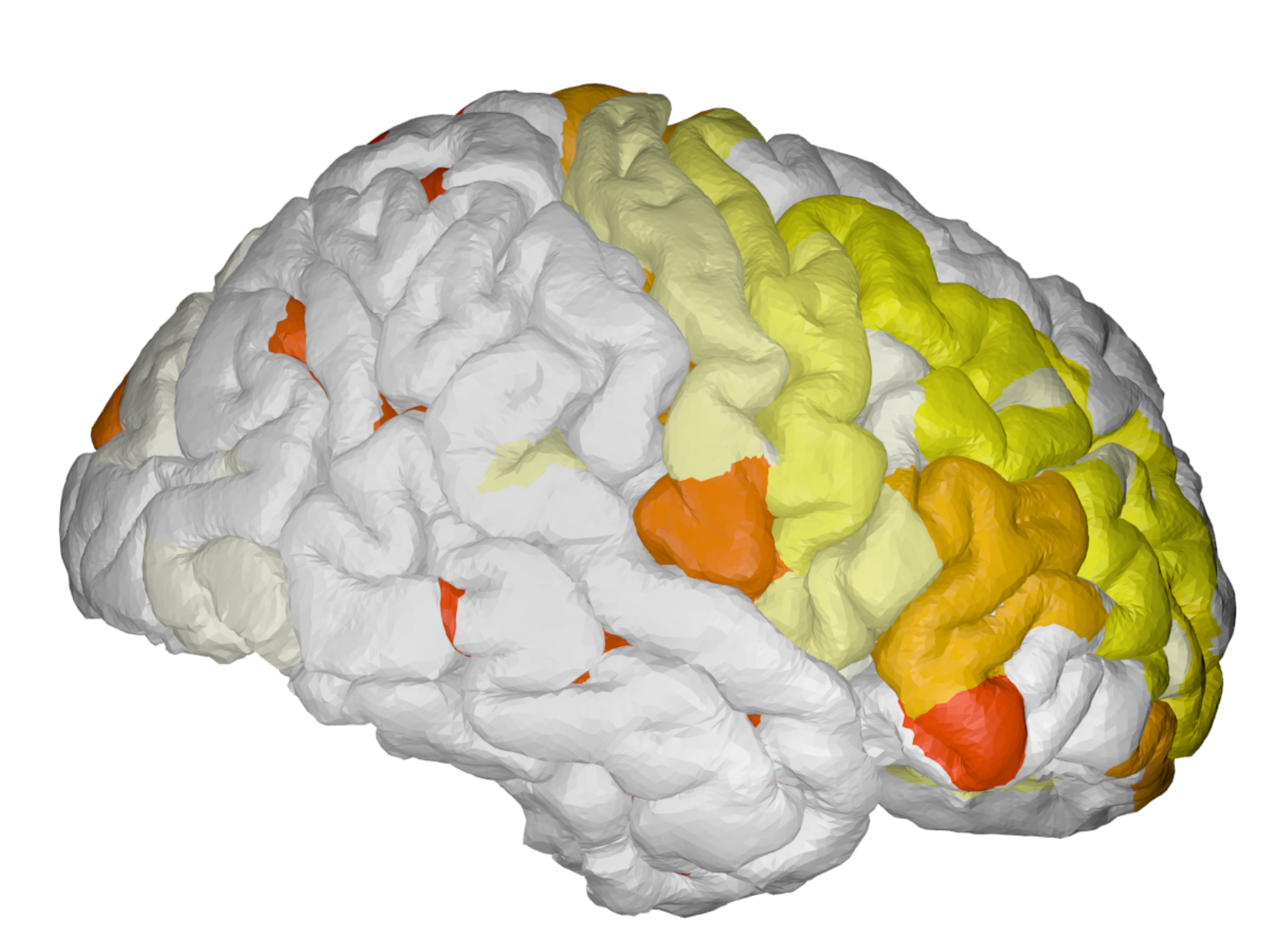} &
    \includegraphics[width=0.21\linewidth]{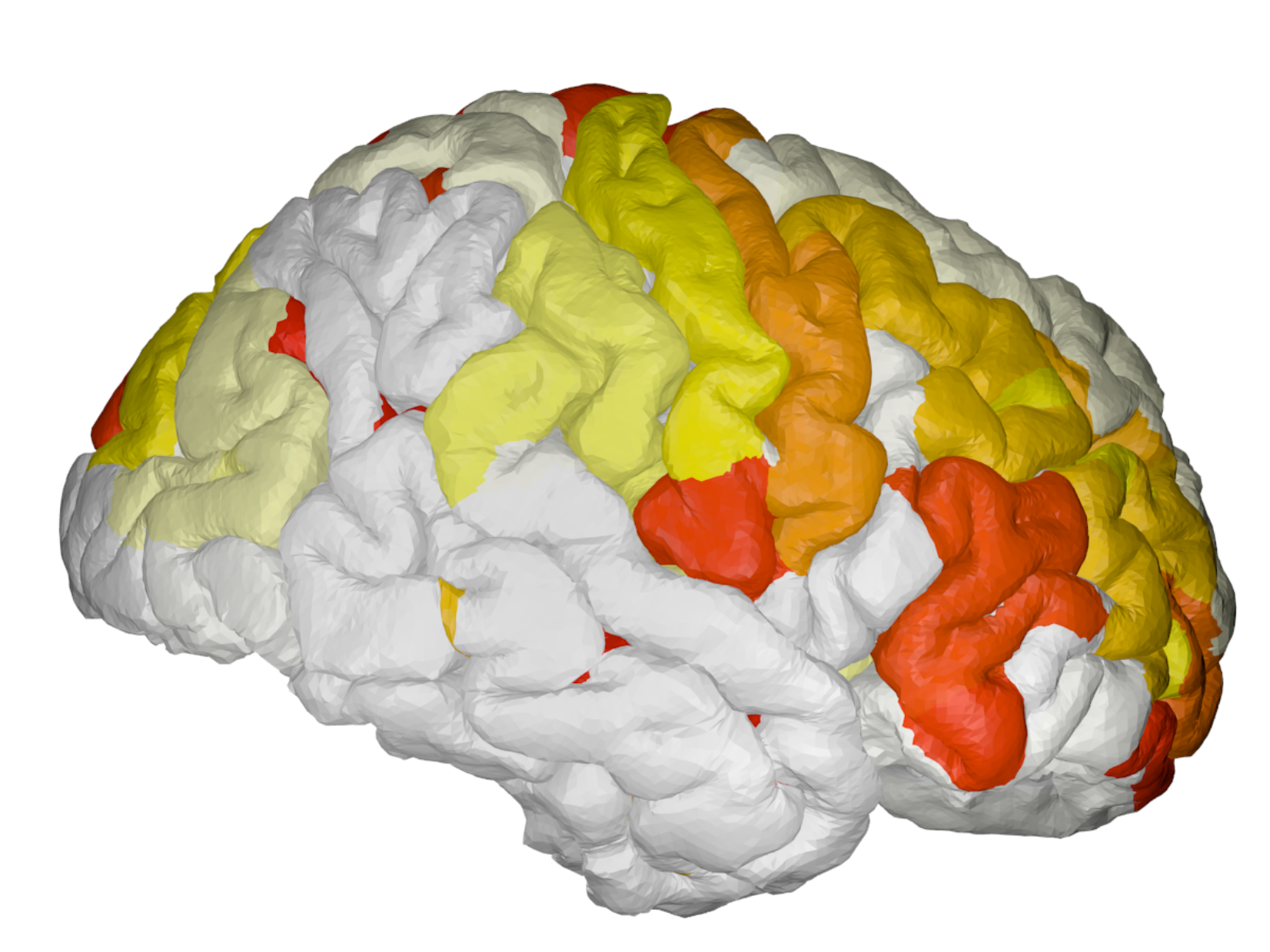} &
    \includegraphics[width=0.21 \linewidth]{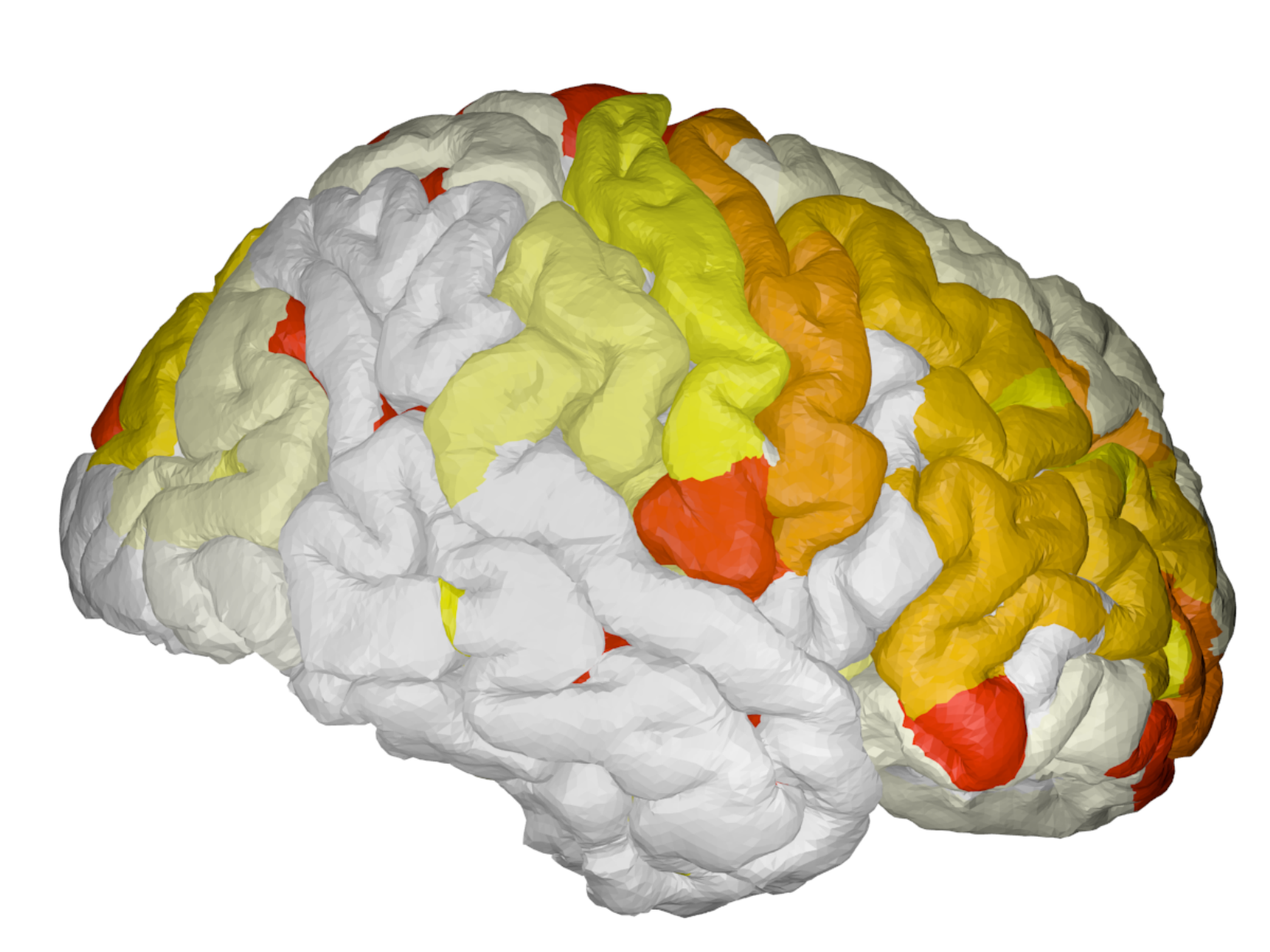} & 
    \raisebox{0\height}[0pt][0pt]{\includegraphics[width=0.09\linewidth]{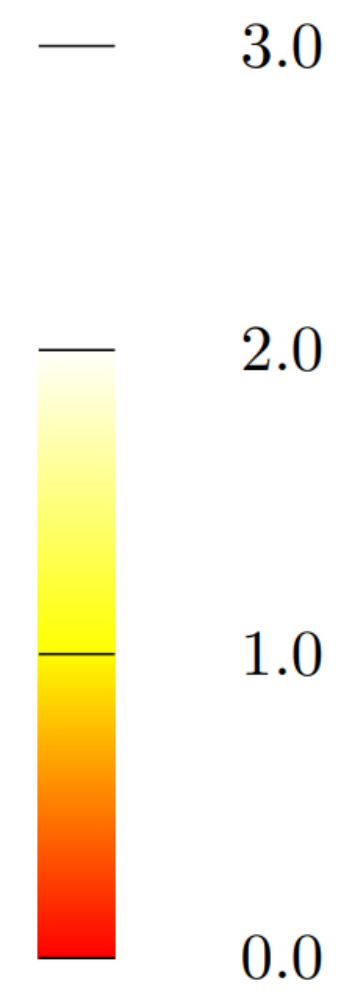}}\\
    
\end{tabular}}}
    \caption{
    Visualization of the learned scales on the cortical regions of a brain.
    This visualization shows the scale of each ROI from the classification result using FDG feature.
    Top: Inner part of right hemisphere, Bottom: Outer part of right hemisphere.
    }
    \label{fig:Brain_Visualization}
\end{figure*}

\subsubsection{Results.} 
The performances including accuracy, precision, and recall between \ourmodel and seven baselines on the ADNI dataset are shown in Table \ref{tab:Result_ADNI}.
\ourmodel showed accuracy $\sim$86\% using cortical thickness and $\sim$90\% with FDG in classifying the 5 diagnostic stages of AD, 
with precision and recall $\sim$0.86 and 0.91, respectively. 
These numbers are approximately the same with the results from Exact 
and low standard deviations from \ourmodel demonstrate feasibility of the approximation. 
\ourmodel performed the best surpassing the second best methods by 4.61\%p and 1.80\%p in accuracy for cortical thickness and FDG experiments, respectively.

\subsection{Model Behavior Analysis}
\label{sub:analysis}

\subsubsection{Computation Time with Kernel Convolution.} 
Fig. \ref{fig:Training_Time_NC} compares averaged empirical time (in {\em ms}) spent for one epoch of training process of Exact and \ourmodel 
on node classification task (on Cora, Citesser and Pubmed) and graph classification task (on ADNI) with 10 replicates. 
The colors denote the type of methods, and as seen in Fig. \ref{fig:Training_Time_NC}, \ourmodel takes far less time than {Exact} computation. 
Notice that the computation of kernel convolution takes 
the majority of time (in black bar), and the approximations 
make this process efficient. 
For the node classification, approximation on Pubmed showed the best efficiency as its graph had the largest number of nodes (19717) compared to Cora (2708) and Citeseer (3327). 
For the graph classification, approximations were even more efficient as eigendecomposition of $\hat L$ from all subjects had to be performed for Exact. 
Comparing the computation time of a single epoch on Exact and \ourmodell, the time is saved by $\sim$93\%.

\begin{table}[!t]
    \centering
    \renewcommand{\arraystretch}{1.0} 
    \renewcommand{\tabcolsep}{0.2cm}
    \scalebox{0.76}{
    {
    \begin{tabular}{l||c|c|c|c}
        \Xhline{3\arrayrulewidth}
        \textbf{ROI \cite{destrieux2010automatic}} & \textbf{Exact} & \textbf{\ourmodelc} & \textbf{\ourmodelh} & \textbf{\ourmodell}\\
        \hline
        (L) G\&S.paracentral & 0.034 & 0.052 & 0.049 & 0.066 \\ 
        (L) G.front.inf.Orbital & 0.036 & 0.071 & 0.060 & 0.043 \\
        (R) G.precuneus & 0.041 & 0.044 & 0.034 & 0.051 \\
        (R) S.ortibal.med.olfact & 0.047 & 0.078 & 0.054 & 0.059 \\
        (R) G.cingul.Post.ventral & 0.055 & 0.056 & 0.055 & 0.051 \\
        (R) S.oc.temp.lat& 0.055 & 0.065 & 0.045 & 0.063 \\
        (R) G.oc.temp.med.Lingual& 0.055 & 0.076 & 0.043 & 0.040 \\
        (L) Sub.put & 0.058 & 0.077 & 0.047 & 0.060 \\
        (L) S.postcentral & 0.061 & 0.069 & 0.060 & 0.018 \\
        (R) G.front.inf.Orbital& 0.063 & 0.093 & 0.069 & 0.050 \\
        \Xhline{3\arrayrulewidth}
    \end{tabular}}}
    \caption{
    10 ROIs with the smallest trained scales for AD classification. 
    (L)/(R) denote the left/right hemisphere.
    }    
\label{tab:RoIs}
\end{table}

\subsubsection{Discussions on the Scales for Graph Classification.} 
In AD classification, we performed graph classification to distinguish the different diagnostic labels of AD. 
The trained model yields node-wise optimized scale 
where the node corresponds to specific region of interest (ROI) in the brain. 
The trained scales denote the optimal ranges of neighborhood for each ROI.
Therefore, if the trained scale is small for a specific node, 
it means that the node does not have to look far to contribute to the classification. 
On the other hand, the nodes with large scales need to aggregate information from far distances 
to constitute an effective embedding as it is not very useful on its own. 
The trained scales on the brain network classification with Exact and \ourmodel are visualized in Fig. \ref{fig:Brain_Visualization} 
conveying two important perspectives. 
First, the scales 
delineate which of the ROIs are independently behaving to classify AD-specific labels. 
Second, the trained scales with \ourmodel are quite similar to the result from Exact meaning that 
the approximation is feasibly accurate for practical uses. 

In Table \ref{tab:RoIs}, 10 ROIs with the smallest scales that appear in common across Exact and LSAPs are listed. 
{\em Inferior frontal orbital gyrus} on both 
hemispheres is captured, and 
several {\em temporal/orbital regions, precuneous}, and {\em left putamen} are shown to yield small scales. 
These ROIs are known as highly AD-specific by various literature \cite{galton2001differing, bailly2015precuneus, de2008strongly, van2000orbitofrontal}. 

\subsubsection{Effect of $K$.} 
We examined the performance of \ourmodel with respect to the number of convolution layers $K$ on Cora and ADNI experiments.   
When $K$ was varied from 1 to 4 under the same setting, 
$K$=2 showed the best performance in both experiments. 
The performance decreases when $K$=3 and 4 may be due to lack of training samples 
as the model sizes are drastically increased. 
We observed the same pattern across {Exact} and LSAP, which demonstrates
\ourmodel with approximation is able to train on the scales properly.

\begin{figure}[!t]
  \centering
  \footnotesize{
  \setlength{\tabcolsep}{1pt}
  \renewcommand{\arraystretch}{0.7}
  \scalebox{1.0}{
  \begin{tabular}{ccc}
    \raisebox{0\height}[0pt][0pt]{\textbf{$\quad$Cora}} & \raisebox{0\height}[0pt][0pt]{\textbf{$\quad$ADNI}} &\\ 
    \includegraphics[width=0.42\linewidth]{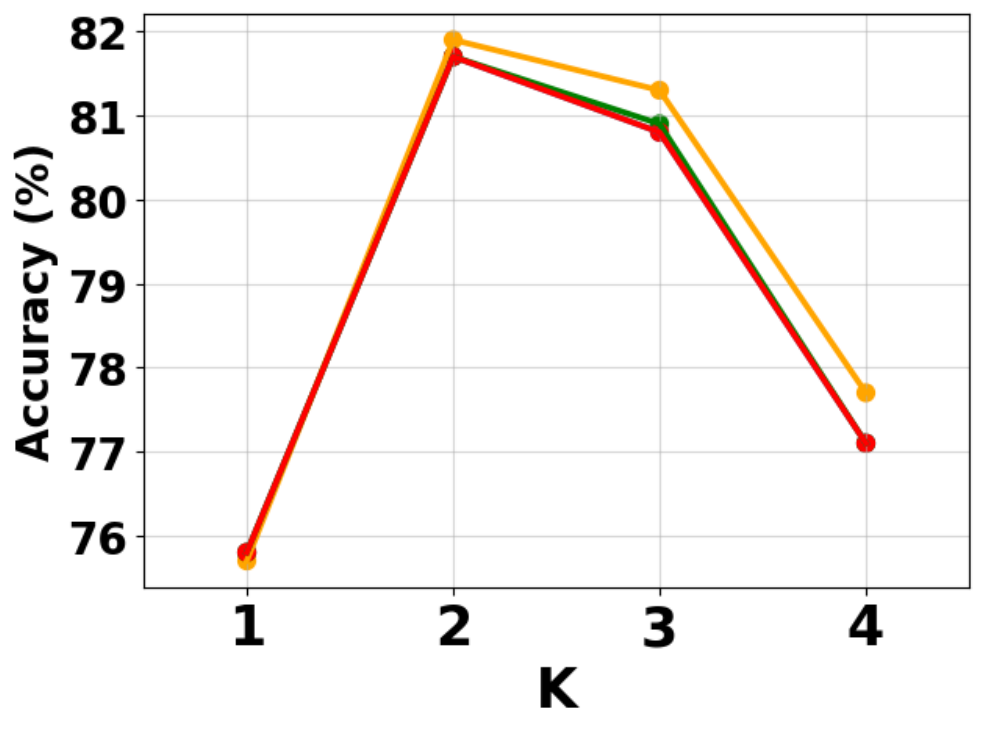} &
    \includegraphics[width=0.42\linewidth]{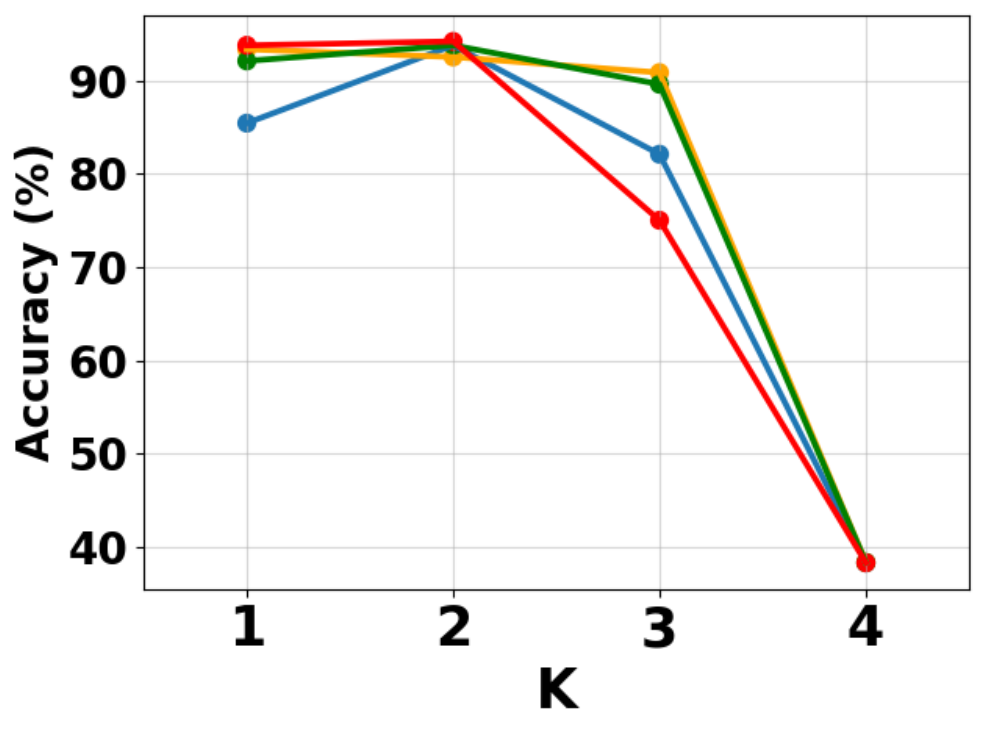} &\\ 
    \multicolumn{2}{c}{\includegraphics[width=0.8\linewidth]{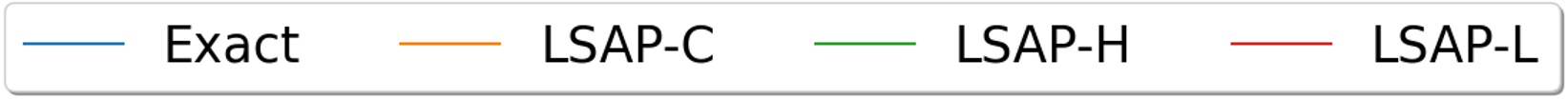}} \\
  \end{tabular}}}
    \caption{ 
    Effect of the number of layers $K$ on model performance. Left: accuracy of node classification on Cora, Right: accuracy of graph classification on ADNI. 
    }
    \label{fig:Ablation}
\end{figure}

\section{Conclusions}
\label{sec:conclusions}

In this work, we proposed  
efficient trainable methods to bypass exact computation of spectral kernel convolution 
that define adaptive ranges of neighbor for each node. 
We have derived closed-form derivatives on polynomial coefficients to train the scale 
with conventional backpropagation, and the developed framework \ourmodel 
demonstrates {\em SOTA} performance on node classification and brain network classification. 
The brain network analysis provides neuroscientifically interpretable results corroborated by 
previous AD literature.

\section{Acknowledgments}
This research was supported by NRF-2022R1A2C2092336 (50\%), IITP-2022-0-00290 (20\%), IITP-2019-0-01906 (AI Graduate Program at POSTECH, 10\%) funded by MSIT, HU22C0171 (10\%), HU22C0168 (10\%) funded by MOHW from South Korea, and NSF IIS CRII 1948510 from the U.S.

\bibliography{aaai24}

\clearpage

\section{Summary}
\label{sec:summary}
This material presents the supplementary paper from the main paper due to the space limitation.
In Section 1, the detailed calculation process of gradient for coefficients are presented.
In Section 2, the proofs of each Lemma from the main paper are provided.
Additional details for LSAP architecture that were omitted from the main manuscript are presented in Section 3. 
In Section 4, hyperparameters of LSAP and baselines for both node classification and graph classification and more results on brain network classification on ADNI data are shown.  
The implementation details are written in Section 5, and
detailed model complexity is in Section 6.
Adjustment of the trained scales at training phase is visualized in Section 7. 
Finally, we additionally discuss our framework in Section 8.

\section{1. Gradients of Polynomial Coefficients with Scale}
\label{sec:gradients}
In this section, we are going to show the detailed 
derivation of gradient $\frac{\partial c_{s,n}}{\partial s}$ for the three polynomials, i.e., Chebyshev, Hermite and Laguerre.

\noindent\textbf{Chebyshev Polynomial.} The expansion coefficient for Chebyshev polynomial is given as
\begin{equation}\footnotesize
    \begin{aligned}
        c_{s,n}^T&=(2-\delta_{n0})(-1)^ne^{-\frac{sb}{2}}I_n\left(\frac{sb}{2}\right),
    \end{aligned}
    \label{eq:che_coeff_supple}
\end{equation}
and the derivative of $c_{s,n}^T$ with respect to $s$ is obtained as
\begin{equation}
    {\footnotesize
    \begin{aligned}
    \frac{\partial c_{s,n}^T}{\partial s} 
    &= 2(-1)^n e^{\frac{-sb}{2}}\left(\frac{-b}{2}\right)I_n\left(\frac{sb}{2}\right)
    \\& \quad +2(-1)^ne^{\frac{-sb}{2}}I_n\left(\frac{sb}{2}\right)\frac{b}{2}\\
    &=2(-1)^ne^{\frac{-sb}{2}}\left(-\frac{b}{2}I_n\left(\frac{sb}{2}\right)+\frac{b}{2}I_n^{'}\left(\frac{sb}{2}\right)\right)\\
    &=(-1)^nbe^{\frac{-sb}{2}}\left(I_n^{'}\left(\frac{sb}{2}\right)-I_n\left(\frac{sb}{2}\right)\right)\\
    &=(-1)^nbe^{\frac{-sb}{2}}\left(I_{n-1}\left(\frac{sb}{2}\right)-\frac{n}{\frac{sb}{2}}I_n\left(\frac{sb}{2}\right)-I_n\left(\frac{sb}{2}\right)\right)\\
    &=(-1)^nbe^{\frac{-sb}{2}}\left(I_{n-1}\left(\frac{sb}{2}\right)-\left(\frac{2n}{sb}+1\right)I_n\left(\frac{sb}{2}\right)\right)\\
    &\left(\because I_n^{'}(x)=I_{n-1}(x)-\frac{n}{x}I_n(x)\right).
    \end{aligned}}
\label{eq:dc_ds_che_supple}
\end{equation}

\noindent\textbf{Hermite Polynomial.} The expansion coefficient for Hermite polynomial is written as
\begin{equation}\footnotesize
    \begin{aligned}
        c_{s,n}^H&=\frac{1}{n!}\left(\frac{-s}{2}\right)^n e^{\frac{s^2}{4}},
    \end{aligned}
    \label{eq:her_coeff_supple}
\end{equation}
and the derivative of $c_{s,n}^H$ with respect to $s$ is computed as
\begin{equation}\footnotesize
    \begin{aligned}
    \frac{\partial c_{s,n}^H}{\partial s}
    &=\frac{n}{n!}\left(\frac{-s}{2}\right)^{n-1}\left(\frac{-1}{2}\right)e^{\frac{s^2}{4}}+\frac{1}{n!}\left(\frac{-s}{2}\right)^ne^{\frac{s^2}{4}}\frac{s}{2}\\
    &=\frac{1}{n!}e^{\frac{s^2}{4}}\left(\frac{-s}{2}\right)^n\left(\frac{-n}{2}\left(\frac{-s}{2}\right)^{-1}+\frac{s}{2}\right)\\
    &=\frac{1}{n!}e^{\frac{s^2}{4}}\left(\frac{-s}{2}\right)^n\left(\frac{n}{s}+\frac{s}{2}\right)\\
    &=\frac{1}{n!}e^{\frac{s^2}{4}}\left(\frac{-s}{2}\right)^n\frac{n}{s}\frac{s}{2}\frac{2}{s}+\frac{1}{n!}e^{\frac{s^2}{4}}\left(\frac{-s}{2}\right)^n\frac{s}{2}\\
    &=\frac{se^{\frac{s^2}{4}}}{2n!}\left(\frac{-s}{2}\right)^n\left(\frac{2n}{s^2}+1\right).
    \end{aligned}
\label{eq:dc_ds_her_supple}
\end{equation}

\noindent\textbf{Laguerre Polynomial.} The expansion coefficient for Laguerre polynomial is written as
\begin{equation}\footnotesize
    \begin{aligned}
        c_{s,n}^L&=\frac{s^n}{(s+1)^{n+1}},
    \end{aligned}
    \label{eq:lag_coeff_supple}
\end{equation}
and the derivative of $c_{s,n}^L$ with respect to $s$ is calculated as
\begin{equation}\footnotesize
    \begin{aligned}
    \frac{\partial c_{s,n}^L}{\partial s}
    &=\frac{ns^{n-1}\left(s+1\right)^{n+1}-s^n\left(n+1\right)\left(s+1\right)^n}{\left\{\left(s+1\right)^{n+1}\right\}^2}\\
    &=\frac{ns^{n-1}\left(s+1\right)^{n+1}-s^n\left(n+1\right)\left(s+1\right)^n}{\left(s+1\right)^{2n}\left(s+1\right)^2}\\
    &=\frac{ns^{n-1}\left(s+1\right)-s^n\left(n+1\right)}{\left(s+1\right)^{n}\left(s+1\right)^2}\\
    &=\frac{ns^n+ns^{n-1}-ns^n-s^n}{\left(s+1\right)^{n+2}}\\
    &=\frac{ns^{n-1}-s^n}{\left(s+1\right)^{n+2}}\\
    &=\frac{\left(\frac{n}{s}-1\right)s^n}{\left(s+1\right)^{n+2}}\\
    &=\frac{s^{n-1}(n-s)}{(s+1)^{n+2}}.
    \end{aligned}
\label{eq:dc_ds_lag_supple}
\end{equation}

\section{2. Proofs of Lemma}
\label{sec:proofs}
\newtheorem{lemmaa}{Lemma}
\begin{lemmaa}
Consider an orthogonal polynomial $P_n$ over interval $[a,b]$ with inner product $\int_a^b P_n(\lambda)P_k(\lambda)w(\lambda)d\lambda=\delta_{nk}$, where $w(\lambda)$ is the weight function. If $ P_n$ expands the heat kernel, the expansion coefficients $c_{\textbf{s},n}$ with respect to $\textbf s$ are differentiable and $\frac{\partial{c_{\textbf{s},n}}}{\partial{\textbf{s}}} = -\int_a^b \lambda e^{-\bf{s}\lambda}P_n(\lambda)w(\lambda)d\lambda$.
\end{lemmaa}
\begin{proof}
    From \cite{huang2020fast} and Eq. (4) in the main paper, the exponential weight of the heat kernel can be expanded by polynomials $P_n$ and coefficients $c_{\textbf{s},n}$ as
\begin{equation}\footnotesize
    e^{-\textbf{s}\lambda}=\sum_{n=0}^{\infty}c_{\textbf{s},n}P_n(\lambda).
    \label{eq:redefined_heat_kernel_supple}
\end{equation}
Multiplying both sides of Eq. \eqref{eq:redefined_heat_kernel_supple} by $P_k(\lambda)w(\lambda)$, and taking the integral over $\lambda$ from a to b,
\begin{equation}\footnotesize
    \begin{aligned}
    \int_a^b e^{-\textbf{s}\lambda}P_k(\lambda)w(\lambda)d\lambda=\int_a^b \sum_{n=0}^{\infty}c_{\textbf{s},n}P_n(\lambda)P_k(\lambda)w(\lambda)d\lambda.
    \label{eq:coeff_1_supple}
    \end{aligned}
\end{equation}
Since $P_n$ is an orthogonal polynomial and satisfies the inner product equation, $\int_a^b P_n(\lambda)P_k(\lambda)w(\lambda)d\lambda=\delta_{nk}$, so the right-hand side of Eq. \eqref{eq:coeff_1_supple} follows as 
\begin{equation}\footnotesize
    \begin{aligned}
    &\int_a^b \sum_{n=0}^{\infty}c_{\textbf{s},n}P_n(\lambda)P_k(\lambda)w(\lambda)d\lambda
    \\&=\sum_{n=0}^{\infty}c_{\textbf{s},n}\int_a^b P_n(\lambda)P_k(\lambda)w(\lambda)d\lambda
    \\&=\sum_{n=0}^{\infty}c_{\textbf{s},n}\delta_{nk}.
    \label{eq:coeff_2_supple}
    \end{aligned}
\end{equation}
$\delta_{nk}$ in Eq. \eqref{eq:coeff_2_supple} is equal to one when $n=k$, otherwise zero. So, $c_{\textbf{s},n}$ can be expressed as
\begin{equation}\footnotesize
    c_{\textbf{s},n}=\int_a^b e^{-\textbf{s}\lambda}P_n(\lambda)w(\lambda)d\lambda.
    \label{eq:redefined_coefficient_supple}
\end{equation}
From the Eq. \eqref{eq:redefined_coefficient_supple}, the $\textbf{s}$ only depends on the exponential term, $e^{-\textbf{s}\lambda}$.
Therefore, the expansion coefficients of the heat kernel are differentiable, and the derivative is
\begin{equation}\footnotesize
    \frac{\partial{c_{\textbf{s},n}}}{\partial{\textbf{s}}} = -\int_a^b \lambda e^{-\bf{s}\lambda}P_n(\lambda)w(\lambda)d\lambda.
    \label{eq:coefficient_derivative_supple}
\end{equation}
\end{proof}
\begin{lemmaa}
Let a graph convolution be operated by Eq. (8) (in the main paper), which approximates 
the convolution with $P_n$ and $c_{\textbf{s},n}$. 
If a loss $\mathcal{L}_{\textbf{err}}$ for node-wise classification 
is defined as cross-entropy between a prediction $\hat Y = \sigma(H_K)$, where $\sigma(\cdot)$ is a softmax function, and the true $Y$, then
\begin{equation}\footnotesize
    \begin{aligned}
    \frac{\partial \mathcal{L}_\textbf{err}}{\partial \textbf{s}}=&(\hat{Y}-Y) \times  \sigma^{'}_k([\sum_{n=0}^mc_{\textbf{s},n}P_n(\hat{L})]H_{k-1}W_k) W_k^{\mathsf{T}}
    \\ &\times(\sum_{n=0}^m P_n(\hat{L})H_{k-1}^{\mathsf{T}}+[\sum_{n=0}^m c_{\textbf{s},n}P_n(\hat{L})]\frac{\partial H_{k-1}}{\partial c_{\textbf{s},n}})\frac{\partial c_{\textbf{s},n}}{\partial \textbf{s}}
    \label{eq:dl_ds_NC_full_supple}
    \end{aligned}
\end{equation}
where $\sigma^{'}_k$ is the derivative of $\sigma_k$.
\end{lemmaa}
\begin{proof}
Let $N$ be the number of nodes, and $J$ be the number of labels. The loss $\mathcal{L}_{\textbf{err}}$ across all labeled nodes $V^L$ is defined as
\begin{equation}\footnotesize
    \mathcal{L}_{\textbf{err}} = -\frac{1}{N}\sum_{i\in V^L}\sum_{j=1}^J Y_{ij} \text{ln} \hat{Y}_{ij}.
    \label{eq:Loss_Function_CE_NC_supple}
\end{equation}
The derivative of $\mathcal{L}$ can be calculated with respect to scale $\textbf{s}$ of $c_{\textbf{s},n}$ via chain rule as,
\begin{equation}\footnotesize
    \frac{\partial \mathcal{L}_{\textbf{err}}}{\partial \textbf{s}}=\frac{\partial\mathcal{L}_{\textbf{err}}}{\partial H_K}\frac{\partial H_K}{\partial c_{\textbf{s},n}}\frac{\partial c_{\textbf{s},n}}{\partial \textbf{s}}.
    \label{eq:dl_ds_NC_supple}
\end{equation}
Since we use softmax to produce pseudo-probability,
the 
$\frac{\partial \mathcal{L}_{\textbf{err}}}{\partial H_K}$ is derived as
\begin{equation}\footnotesize
    \frac{\partial\mathcal{L}_{\textbf{err}}}{\partial H_K}=\hat{Y}-Y.
    \label{eq:dl_dh_supple}
\end{equation}
From Eq. (8) in the main paper, the 
$\frac{\partial H_K}{\partial c_{\textbf{s},n}}$ becomes
\begin{equation}\footnotesize
    \begin{aligned}
    \frac{\partial H_K}{\partial c_{\textbf{s},n}}&=\sigma^{'}_k([\sum_{n=0}^m c_{\textbf{s},n}P_n(\hat{L})]H_{k-1}W_k)\times W_k^{\mathsf{T}}
    \\&\times(\sum_{n=0}^m P_n(\hat{L})H_{k-1}^{\mathsf{T}}+[\sum_{n=0}^m c_{\textbf{s},n}P_n(\hat{L})]\frac{\partial H_{k-1}}{\partial c_{\textbf{s},n}})
    \label{eq:dh_dc_supple}
    \end{aligned}
\end{equation}
where the derivative of $H_k$ is recursively defined with $H_{k-1}$ along the hidden layer.
Using chain rule with Eq. \eqref{eq:dl_dh_supple}, \eqref{eq:dh_dc_supple} and $\frac{\partial c_{\textbf{s},n}}{\partial \textbf{s}}$ that depends on the type of polynomial, the gradient on $\mathcal{L}_{\textbf{err}}$ for node classification is written as 
\begin{equation}\footnotesize
    \begin{aligned}
    \frac{\partial \mathcal{L}_\textbf{err}}{\partial \textbf s}=&(\hat{Y}-Y) \times  \sigma^{'}_k([\sum_{n=0}^mc_{\textbf{s},n}P_n(\hat{L})]H_{k-1}W_k) W_k^{\mathsf{T}}
    \\ &\times(\sum_{n=0}^m P_n(\hat{L})H_{k-1}^{\mathsf{T}}+[\sum_{n=0}^m c_{\textbf{s},n}P_n(\hat{L})]\frac{\partial H_{k-1}}{\partial c_{\textbf{s},n}}) \frac{\partial c_{\textbf{s},n}}{\partial \textbf{s}}.
    \end{aligned}
\end{equation}
\end{proof}
\begin{lemmaa}
Let $H_k$ from Eq. (8) (in the main paper) be a graph convolution operation, which is operated by the heat kernel with polynomial $P_n$ and coefficients $c_{\textbf{s},n}$. 
If a loss $\mathcal{L}_{\textbf{err}}$ for classifying graph-wise label 
is defined as cross-entropy between a prediction $\hat Y = \sigma(H_R)$, where $\sigma(\cdot)$ is a softmax function, and the true $Y$, then
\begin{equation}\footnotesize
    \begin{aligned}
    \frac{\partial \mathcal{L}_\textbf{err}}{\partial \textbf{s}}=&(\hat{Y}-Y) \times \frac{\partial H_R}{\partial H_K} \times \sigma^{'}_k([\sum_{n=0}^m c_{\textbf{s},n}P_n(\hat{L})]H_{k-1}W_k) W_k^{\mathsf{T}} \\&\times(\sum_{n=0}^m P_n(\hat{L})H_{k-1}^{\mathsf{T}}+[\sum_{n=0}^m c_{\textbf{s},n}P_n(\hat{L})]\frac{\partial H_{k-1}}{\partial c_{\textbf{s},n}}) \frac{\partial c_{\textbf{s},n}}{\partial \textbf{s}}
    \label{eq:dl_ds_GC_supple}
    \end{aligned}
\end{equation}
where $\sigma^{'}_k$ is the derivative of $\sigma_k$.
\end{lemmaa}
\begin{proof}
Let $T$ be the sample size, and $J$ be the set of class labels. As in the node classification, the loss $\mathcal{L}_{\textbf{err}}$ is defined with cross-entropy as
\begin{equation}\footnotesize
    \mathcal{L}_{\textbf{err}} = -\frac{1}{T}\sum_{t=1}^T\sum_{j\in J} Y_{tj} \text{ln} \hat{Y}_{tj}
    \label{eq:Loss_Function_CE_GC_supple}
\end{equation}
Also, the derivative of $\mathcal{L}$ can be calculated in terms of scale $s$ of $c_{s,n}$ via chain rule. 
As we use an additional MLP for graph classification, i.e., Eq. (15) in the main paper, the gradient of $\mathcal{L}_{\textbf{err}}$ along $H_R$ must be computed and plugged into Eq. \eqref{eq:dl_ds_NC_full_supple} as 
\begin{equation}\footnotesize
    \frac{\partial \mathcal{L}_{\textbf{err}}}{\partial \textbf{s}}=\frac{\partial\mathcal{L}_{\textbf{err}}}{\partial H_R}\frac{\partial H_R}{\partial H_K}\frac{\partial H_K}{\partial c_{\textbf{s},n}}\frac{\partial c_{\textbf{s},n}}{\partial \textbf{s}}.
    \label{eq:dl_ds_GC_proof_supple}
\end{equation}\\
Since the activation function of output layer is a softmax, the 
$\frac{\partial \mathcal{L}_{\textbf{err}}}{\partial H_R}$ are the same as Eq. ~\eqref{eq:dl_dh_supple},
\begin{equation}\footnotesize
    \frac{\partial\mathcal{L}_{\textbf{err}}}{\partial H_R}=\hat{Y}-Y.
    \label{eq:dl_dH_RC_supple}
\end{equation}
The 
$\frac{\partial H_R}{\partial H_k}$ varies depending on the choice of $H_R$, and the $\frac{\partial H_k}{\partial c_{\textbf{s},n}}$ is the same as Eq. \eqref{eq:dh_dc_supple},
\begin{equation}\footnotesize
    \begin{aligned}
    \frac{\partial H_K}{\partial c_{\textbf{s},n}}&=\sigma^{'}_k([\sum_{n=0}^m c_{\textbf{s},n}P_n(\hat{L})]H_{k-1}W_k)\times W_k^{\mathsf{T}}
    \\&\times(\sum_{n=0}^m P_n(\hat{L})H_{k-1}^{\mathsf{T}}+[\sum_{n=0}^m c_{\textbf{s},n}P_n(\hat{L})]\frac{\partial H_{k-1}}{\partial c_{\textbf{s},n}}).
    \label{eq:dh_dc_GC_supple}
    \end{aligned}
\end{equation}
With these components and $\frac{\partial c_{\textbf{s},n}}{\partial \textbf{s}}$ that depends on the type of polynomial, the gradient on $\mathcal{L}_\textbf{err}$ is computed as
\begin{equation}\footnotesize
    \begin{aligned}
    \frac{\partial \mathcal{L}_\textbf{err}}{\partial \textbf{s}}&=(\hat{Y}-Y) \times \frac{\partial H_R}{\partial H_K} \times \sigma^{'}_k([\sum_{n=0}^m c_{\textbf{s},n}P_n(\hat{L})]H_{k-1}W_k) W_k^{\mathsf{T}} \\&\times(\sum_{n=0}^m P_n(\hat{L})H_{k-1}^{\mathsf{T}}+[\sum_{n=0}^m c_{\textbf{s},n}P_n(\hat{L})]\frac{\partial H_{k-1}}{\partial c_{\textbf{s},n}})\frac{\partial c_{\textbf{s},n}}{\partial \textbf{s}}.
    \label{eq:dl_ds_GC_full_supple}
    \end{aligned}
\end{equation}
\end{proof}
The $\frac{\partial H_R}{\partial H_K}$ for the MLP used within our framework is given in the following Section. 

\section{3. LSAP Architecture}
\label{sec:architecture}

\textbf{Readout for Graph Classification.} 
The $f_R(\cdot)$ with weights $W_R$ takes $H_K$ from the convolution layers as an input and returns $H_R$. In our model architecture, $H_R$ is chosen as 2-layer Multi-layer perceptron (MLP) as 
\begin{equation}\footnotesize
    \begin{aligned}
    H_R=\sigma_{R_2}(\sigma_{R_1}(H_KW_{R_1})W_{R_2}) 
    \end{aligned}
\label{eq:Readout_supple}
\end{equation}
where $R_1$ and $R_2$ correspond to first and second layer for MLP structure, respectively. $W_{R_1}$ and $W_{R_2}$ denote weights and Rectified Linear Unit (ReLU) was used for the non-linear activation functions $\sigma_{R_1}$ and $\sigma_{R_2}$ for each layer. 
To make our model end-to-end trainable, the derivative of $H_R$ with respect to $H_K$ is computed as
\begin{equation}\footnotesize
    \begin{aligned}
    \frac{\partial H_R}{\partial H_K}
    =\sigma_{R_2}^{'}(\sigma_{R_1}(H_KW_{R_1})W_{R_2})W_{R_2}\sigma_{R_1}^{'}(H_KW_{R_1})W_{R_1}. 
    \end{aligned}
\label{eq:Readout_Back_supple}
\end{equation}

\begin{table}[!t]
\caption{\small
Node classification accuracy (\%) on Amazon Computers, Amazon Photo, and Coauthor CS.
The number is the best accuracy from 10 replicated experiments like \cite{luo2022inferring}, and the best results are in {\bf bold} except {\em Exact} method.
}
\vspace{-10pt}
\centering
    \renewcommand{\arraystretch}{1.0} 
    \renewcommand{\tabcolsep}{0.3cm}
  {\small
  \scalebox{0.95}{
  \begin{tabular}{l|ccc}
    \Xhline{3\arrayrulewidth}
    \multirow{2}{*}{\textbf{Model}} & 
    \textbf{Amazon} & \textbf{Amazon} & \textbf{Coauthor} \\ 
    & \textbf{Computer} & \textbf{Photo} & \textbf{CS} \\
    \hline
    MLP (3-layers) & 84.63 & 91.96 & 95.63 \\
    GCN & 90.49 & 93.91 & 93.32 \\
    3ference & 90.74 & 95.05 & 95.99 \\
    GAT  & 92.51 & 95.16 & 93.70 \\
    GDC & 92.00 & 94.57 & 93.45 \\
    GraphHeat & 92.62 & 94.77 &93.29 \\
    \hline 
    \textbf{LSAP-C} & 95.53 & 97.25 & {\bf 96.37} \\
    \textbf{LSAP-H} & {\bf 96.01} & {\bf 97.38} & {\bf 96.37}\\
    \textbf{LSAP-L} & 93.67 & 96.66 & 95.91 \\
    Exact & 95.34 & 98.17 & 97.14 \\
    \Xhline{3\arrayrulewidth}
  \end{tabular}}
\label{tab:Result_AAC_supple}
}
\vspace{-10pt}
\end{table}

\begin{table*}[t!]
    \caption{Hyperparameters for classification tasks.}
    \scalebox{0.6}{
    \centering
    \renewcommand{\arraystretch}{1.2}
    \renewcommand{\tabcolsep}{0.4cm}
        \begin{tabular}{c|c|c||c|c|c|c|c|c}
            \Xhline{3\arrayrulewidth}
            \textbf{Task} & \textbf{Dataset} & \textbf{Model} & \textbf{Hidden Units} & \textbf{Learning Rate} & \textbf{Dropout Rate} & \textbf{Regularization} ($\alpha$) & \textbf{Scale's Learning Rate} ($\beta_s$) & \textbf{LSAP-C} (b)\\
            \hline
            \multirow{15}{*}{\shortstack{\textbf{Node}\\ \textbf{Classification}}}&\multirow{1}{*}{\textbf{Cora}} 
            &\textbf{LSAP} $\&$ Exact & 64 & 0.01 & 0.5 & 0.1 & 1 & 1.48\\ \cline{2-9}
            &\multirow{1}{*}{\textbf{Citeseer}} 
            &\textbf{LSAP} $\&$ Exact & 32 & 0.01 & 0.5 & 1 & 10 & 1.50\\ \cline{2-9}
            &\multirow{1}{*}{\textbf{Pubmed}} 
            &\textbf{LSAP} $\&$ Exact & 64 & 0.1 & 0.5 & 1 & 10 & 1.65\\ \cline{2-9}
            &\multirow{4}{*}{\shortstack{\textbf{Amazon}\\ \textbf{Computer}}} 
            &GAT & 32 & 0.01 & 0.5 & - & - & -\\ \cline{3-9}
            &&GDC & 32 & 0.0001 & 0.5 & - & - & -\\ \cline{3-9}
            &&GraphHeat & 32 & 0.01 & 0.5 & - & - & -\\ \cline{3-9}
            &&\textbf{LSAP} $\&$ Exact & 32 & 0.001 & 0.5 & 1 & 1 & 1.60\\ \cline{2-9}
            &\multirow{4}{*}{\shortstack{\textbf{Amazon}\\ \textbf{photo}}} 
            &GAT& 32 & 0.01 & 0.5 & - & - & -\\ \cline{3-9}
            &&GDC & 16 & 0.0001 & 0.5 & - & - & -\\ \cline{3-9}
            &&GraphHeat & 32 & 0.01 & 0.5 & - & - & -\\ \cline{3-9}
            &&\textbf{LSAP} $\&$ Exact & 32 & 0.01 & 0.5 & 1 & 10 & 1.59\\ \cline{2-9}
            &\multirow{4}{*}{\shortstack{\textbf{Coauthor}\\ \textbf{CS}}} 
            &GAT & 32 & 0.01 & 0.5 & - & - & -\\ \cline{3-9}
            &&GDC & 16 & 0.0001 & 0.5 & - & - & -\\ \cline{3-9}
            &&GraphHeat & 32 & 0.01 & 0.5 & - &- & -\\ \cline{3-9}
            &&\textbf{LSAP} $\&$ Exact & 32 & 0.01 & 0.5 & 1 & 1 & 1.41\\
            \hline
            \multirow{16}{*}{\shortstack{\textbf{Graph}\\ \textbf{Classification}}}&\multirow{8}{*}{\shortstack{\textbf{Cortical}\\ \textbf{Thickness}}} 
            &SVM (Linear) & - & - & - & - & - & -\\ \cline{3-9}
            &&MLP (2-layers) & 16 & 0.01 & 0.5 & - & - & -\\ \cline{3-9}
            &&GCN & 16 & 0.01 & 0.5 & - & - & -\\ \cline{3-9}
            &&GAT & 16 & 0.01 & 0.1 & - & - & -\\ \cline{3-9}
            &&GDC & 16 & 0.01 & 0.5 & - & - & -\\ \cline{3-9}
            &&GraphHeat & 32 & 0.01 & 0.5 & - & - & -\\ \cline{3-9}
            &&ADC & 16 & 0.01 & 0.5 & - & - & -\\ \cline{3-9}
            &&\textbf{LSAP} $\&$ Exact & 16 & 0.01 & 0.5 & 1 & 1 & 1.20\\\cline{2-9}
            &\multirow{8}{*}{\textbf{FDG}} 
            &SVM (Linear) & -& -& -& -& -& -\\\cline{3-9}
            &&MLP (2-layers) & 16& 0.01& 0.5& -& -& -\\\cline{3-9}
            &&GCN & 16& 0.01& 0.5& -& -& -\\ \cline{3-9}
            &&GAT & 16& 0.01& 0.1& -& -& -\\ \cline{3-9}
            &&GDC & 16& 0.01& 0.5& -& -& - \\ \cline{3-9}
            &&GraphHeat & 16& 0.01& 0.5& -&- &- \\ \cline{3-9}
            &&ADC & 16 & 0.01 & 0.5 & - & - & -\\ \cline{3-9}
            &&\textbf{LSAP} $\&$ Exact & 16 & 0.01 & 0.5 & 1 & 1 & 1.20\\
            \Xhline{3\arrayrulewidth}
        \end{tabular}
        }
\label{tab:hyperparameter_supple}
\end{table*}

\begin{figure*}[t!]
  \centering
  \renewcommand{\arraystretch}{1.0}
  \renewcommand{\tabcolsep}{0.05cm}
  \small{
  \scalebox{0.9}{
  \begin{tabular}{ccccccccl}
    & 
    \raisebox{1\height}[0pt][0pt]{\textbf{Top}} & \raisebox{1\height}[0pt][0pt]{\textbf{Bottom}} &
    \raisebox{1\height}[0pt][0pt]{\textbf{Outer-Left}} & \raisebox{1\height}[0pt][0pt]{\textbf{Outer-Right}} & \raisebox{1\height}[0pt][0pt]{\textbf{Inner-Left}} & \raisebox{1\height}[0pt][0pt]{\textbf{Inner-Right}} & \raisebox{1\height}[0pt][0pt]{\textbf{Sub-Cortical}} & \\ \vspace{-0.2cm}
    \raisebox{4\height}[0pt][0pt]{\textbf{Exact}} & 
    \includegraphics[width=0.13\linewidth]{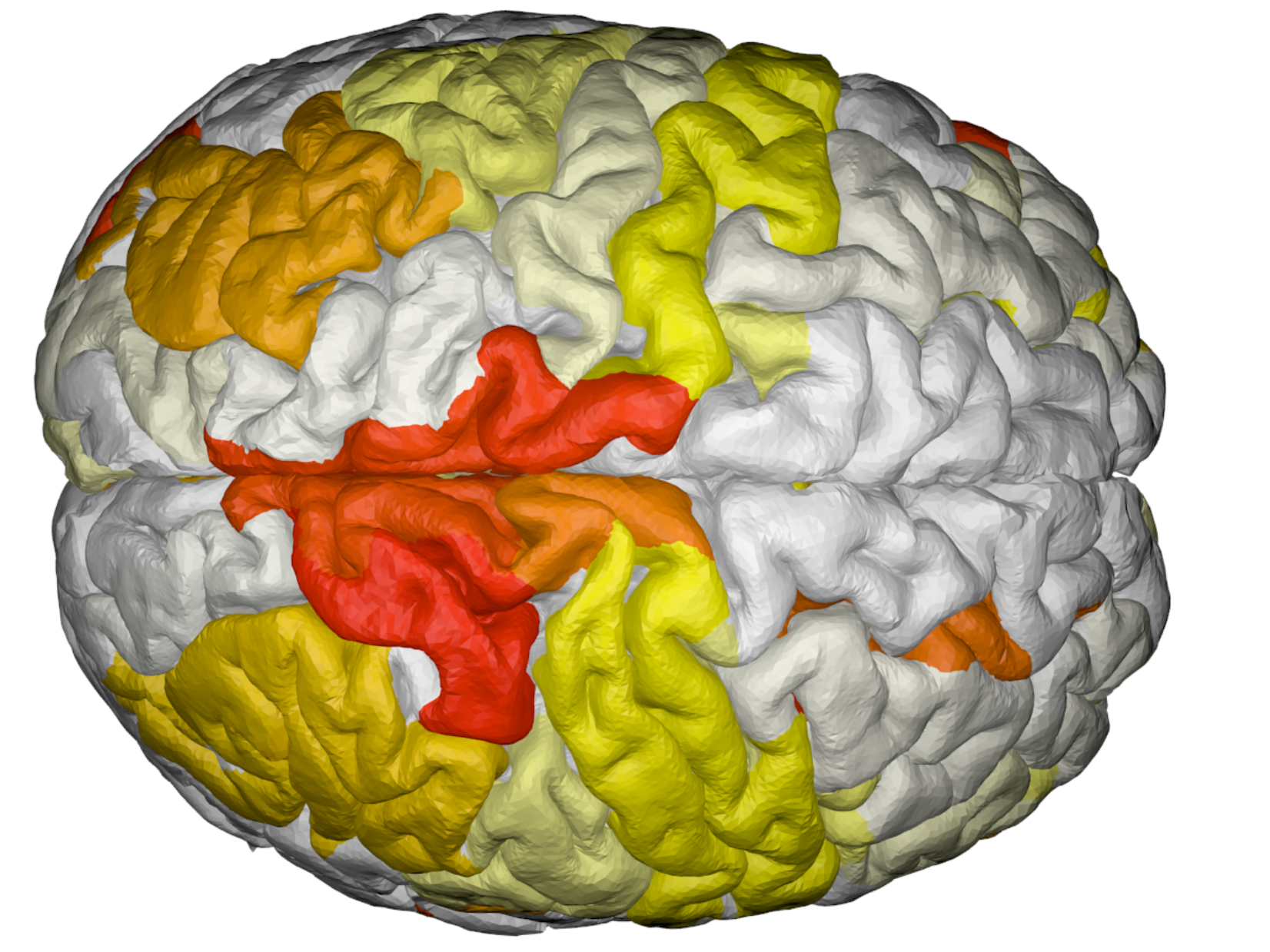} &
    \includegraphics[width=0.13\linewidth]{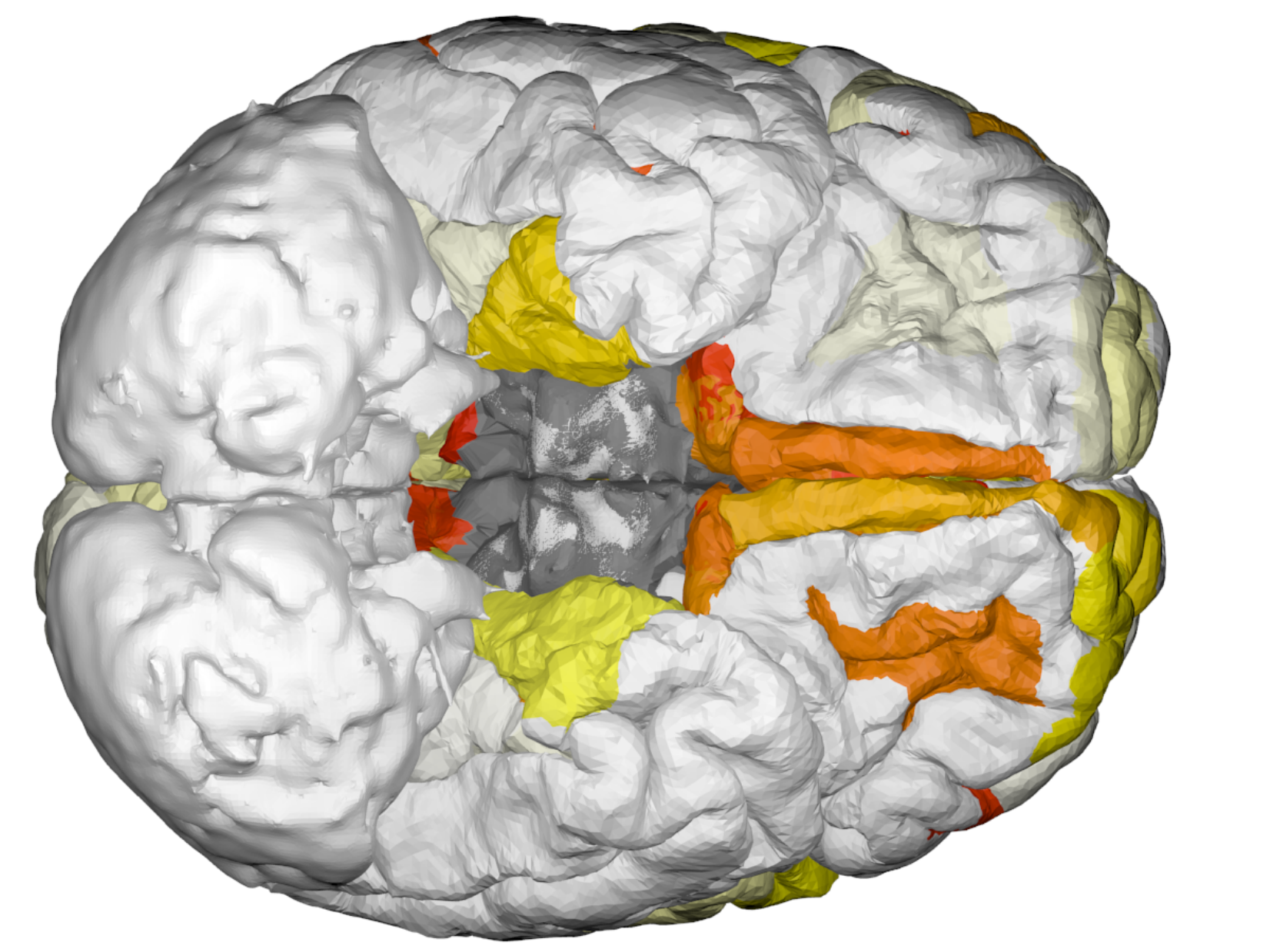} &
    \includegraphics[width=0.13\linewidth]{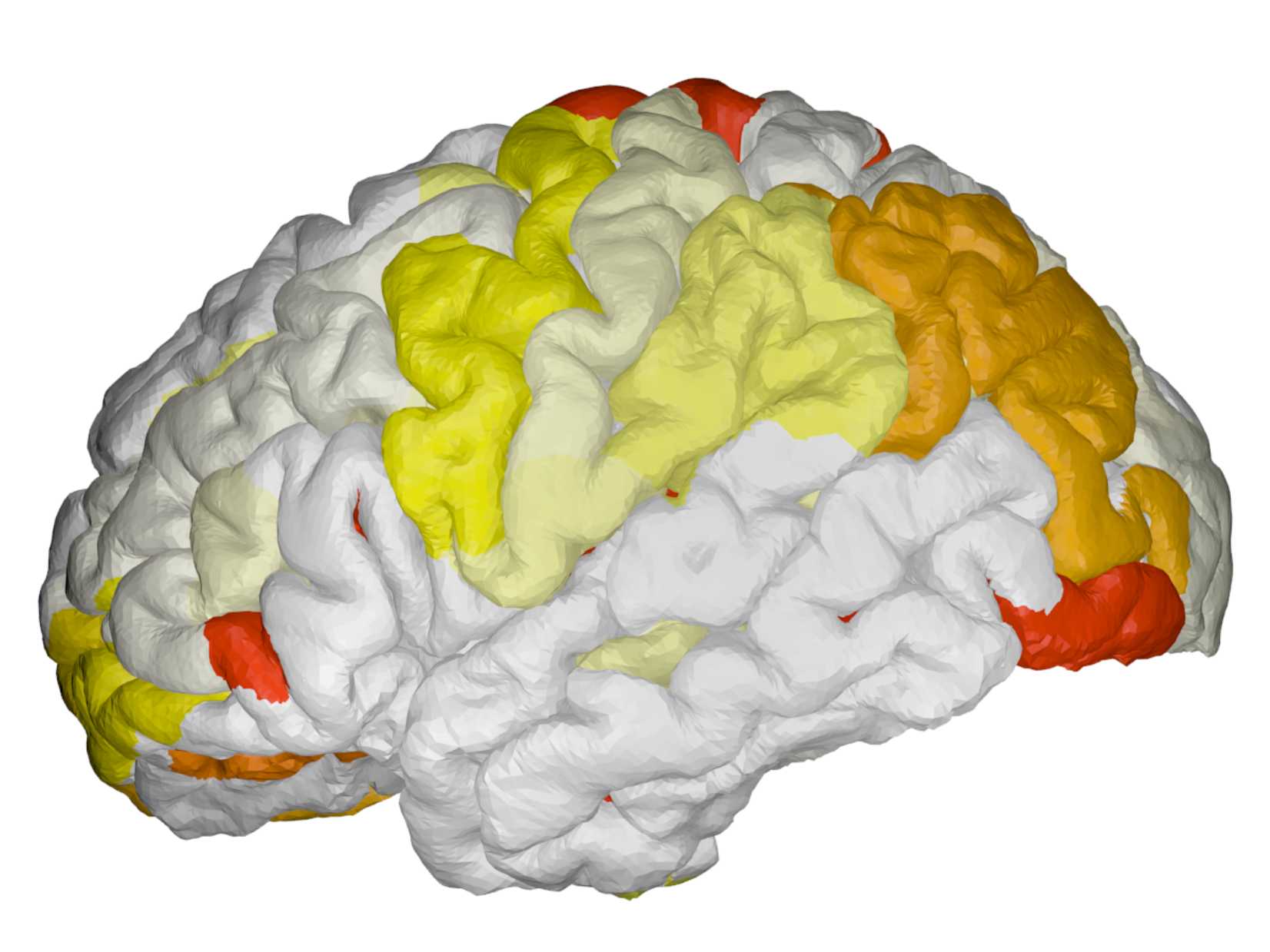} &
    \includegraphics[width=0.13\linewidth]{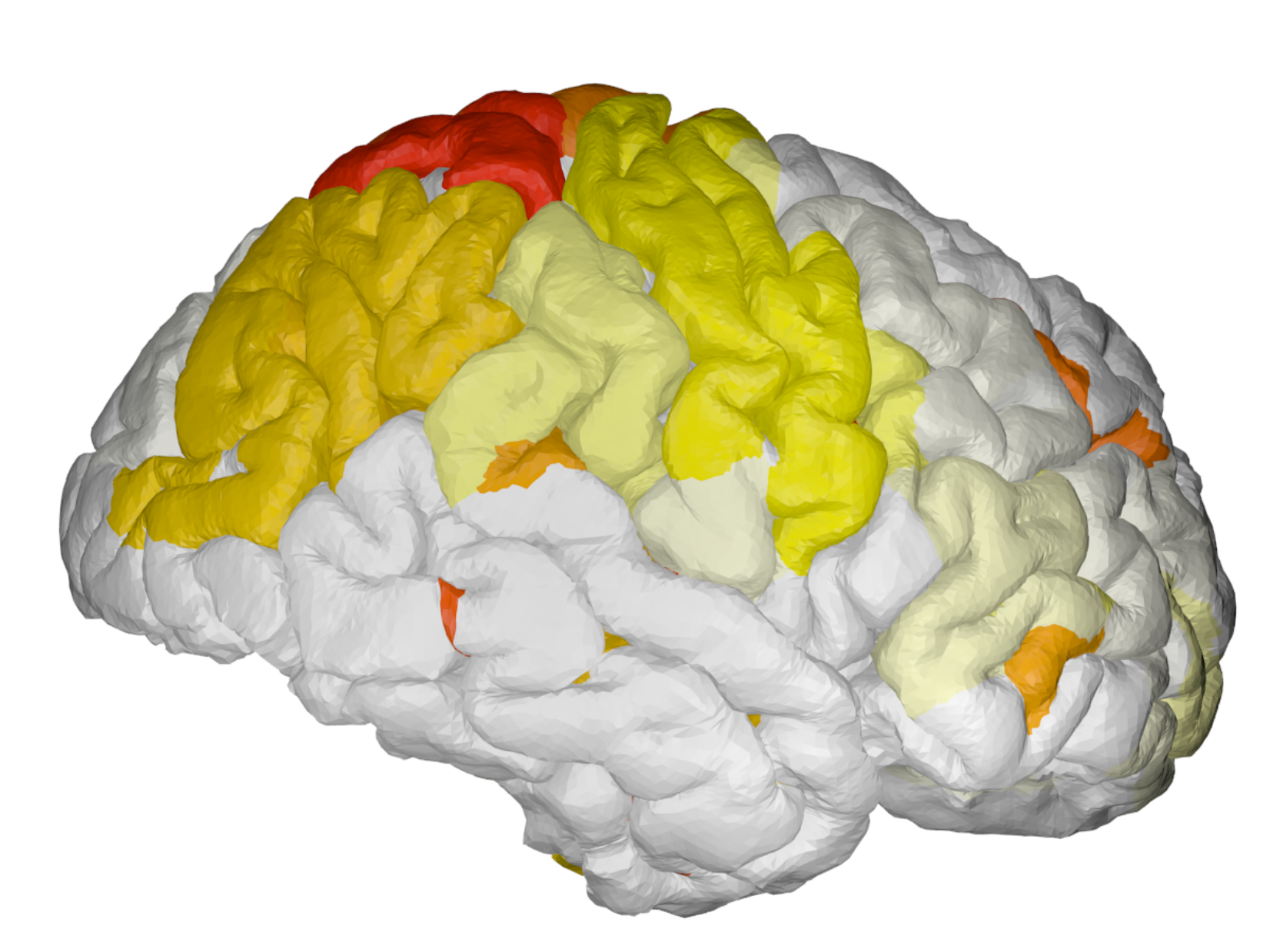} &
    \includegraphics[width=0.13\linewidth]{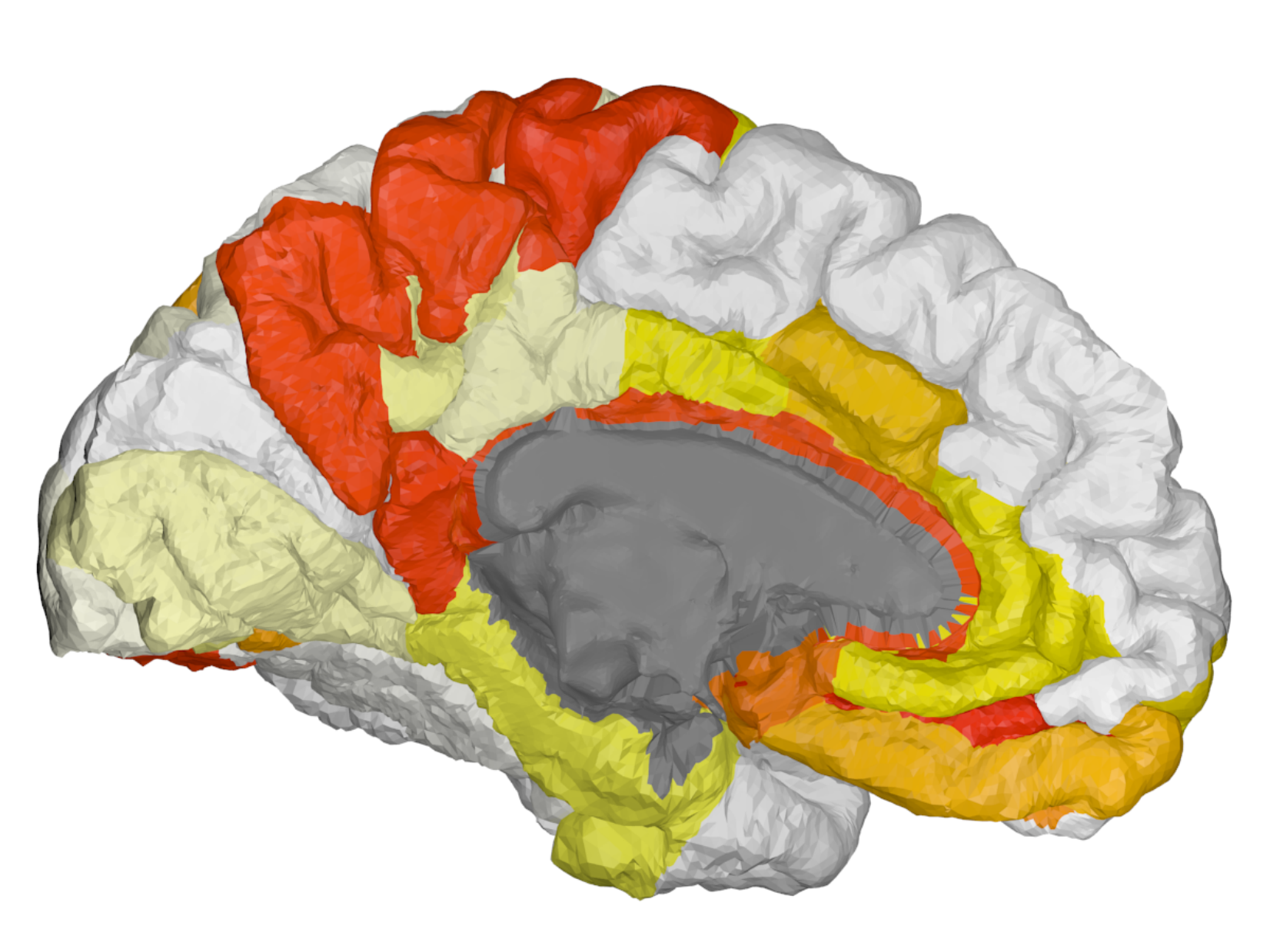} &
    \includegraphics[width=0.13\linewidth]{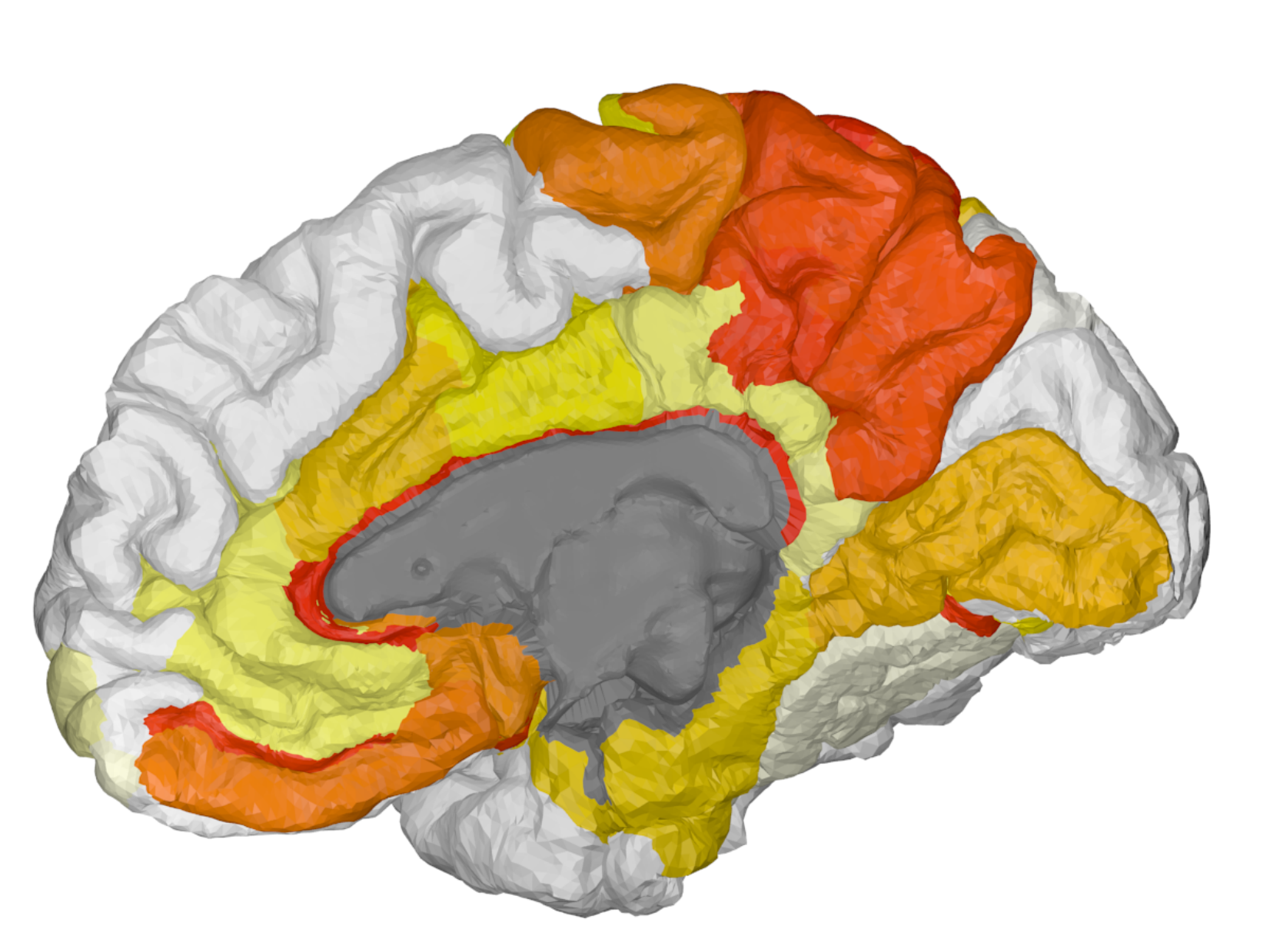} & 
    \includegraphics[width=0.13\linewidth]{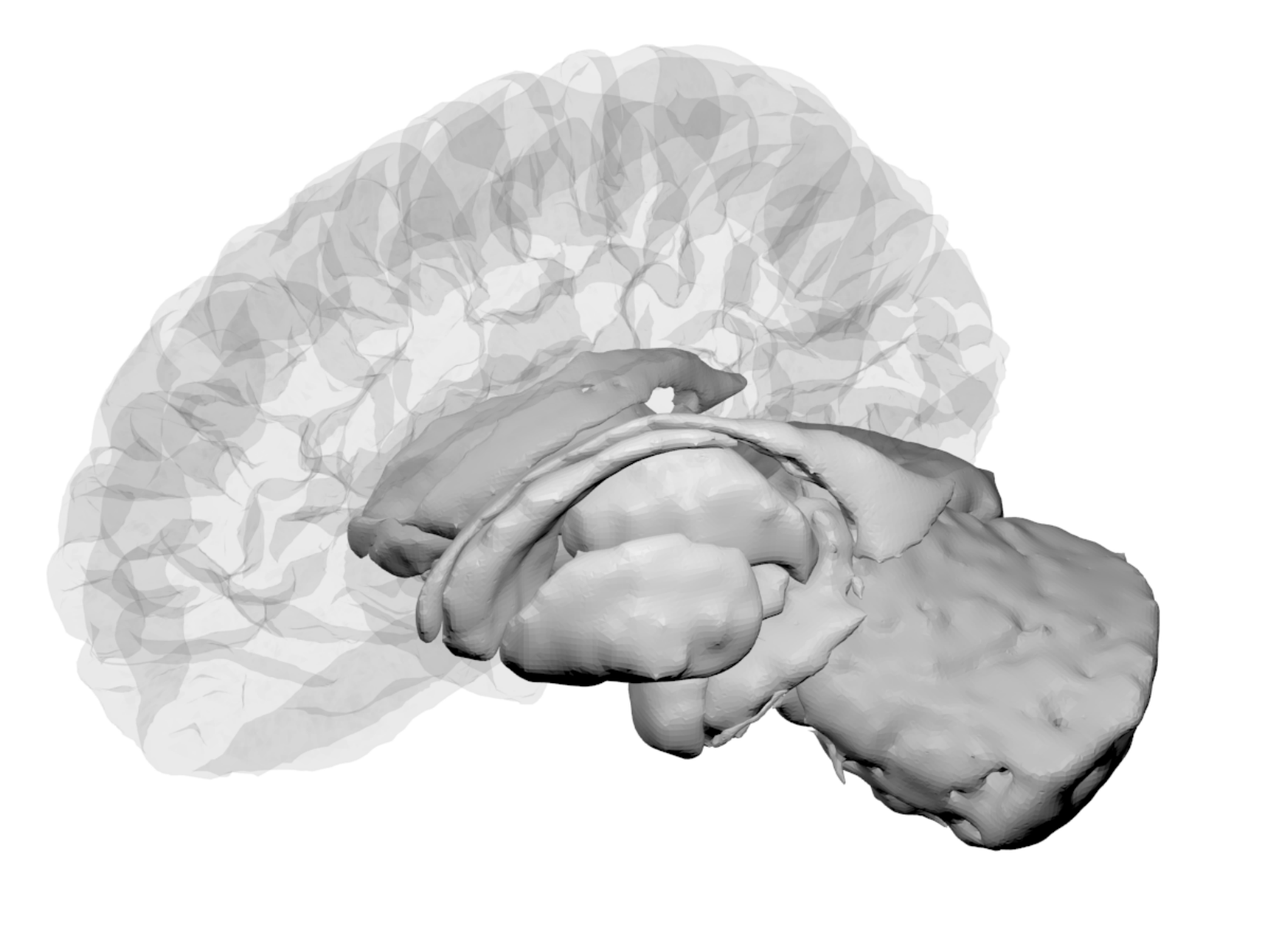} &\\ \vspace{-0.2cm}
    \raisebox{4\height}[0pt][0pt]{\textbf{LSAP-C}} &
    \includegraphics[width=0.13\linewidth]{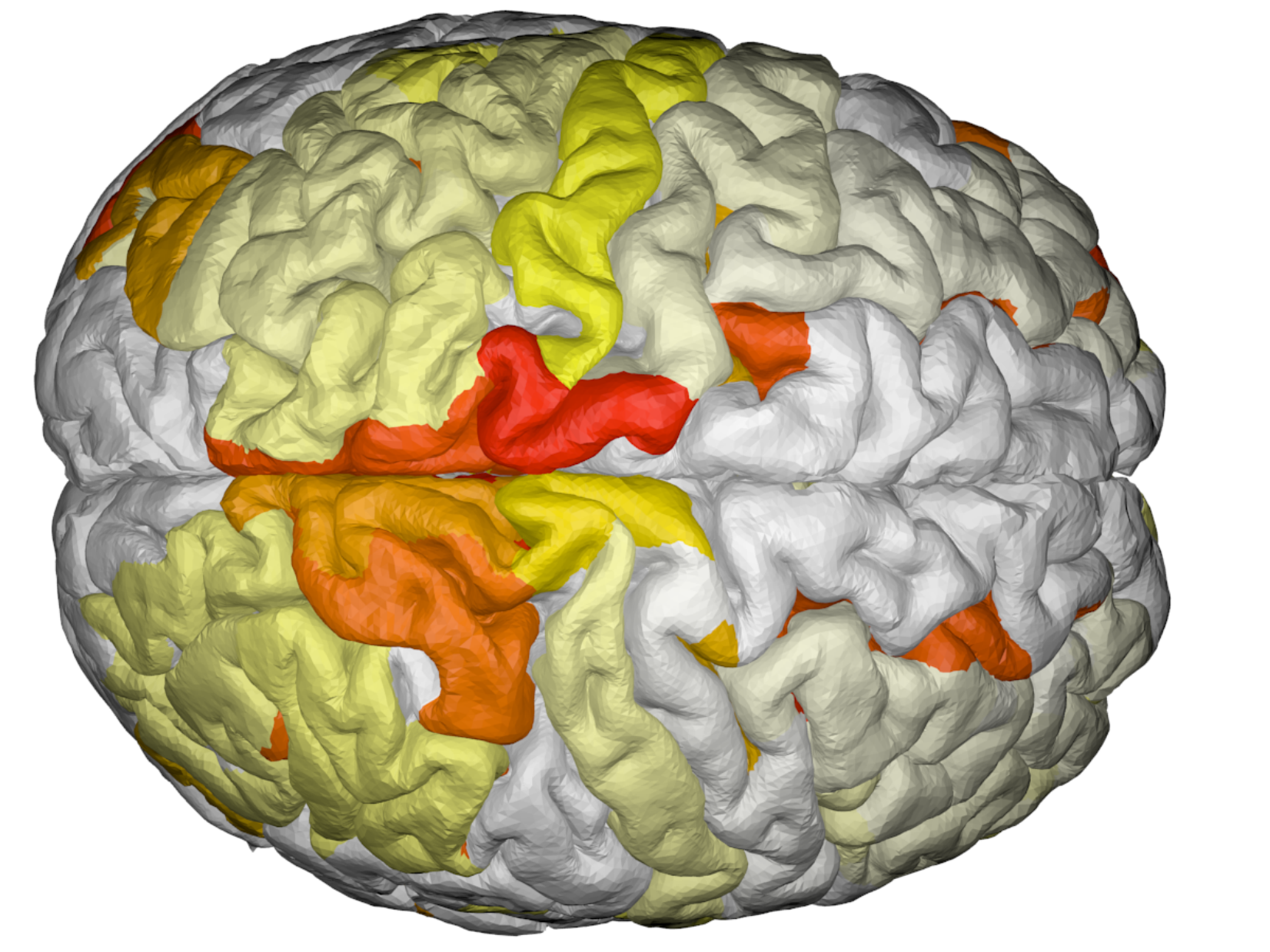} &
    \includegraphics[width=0.13\linewidth]{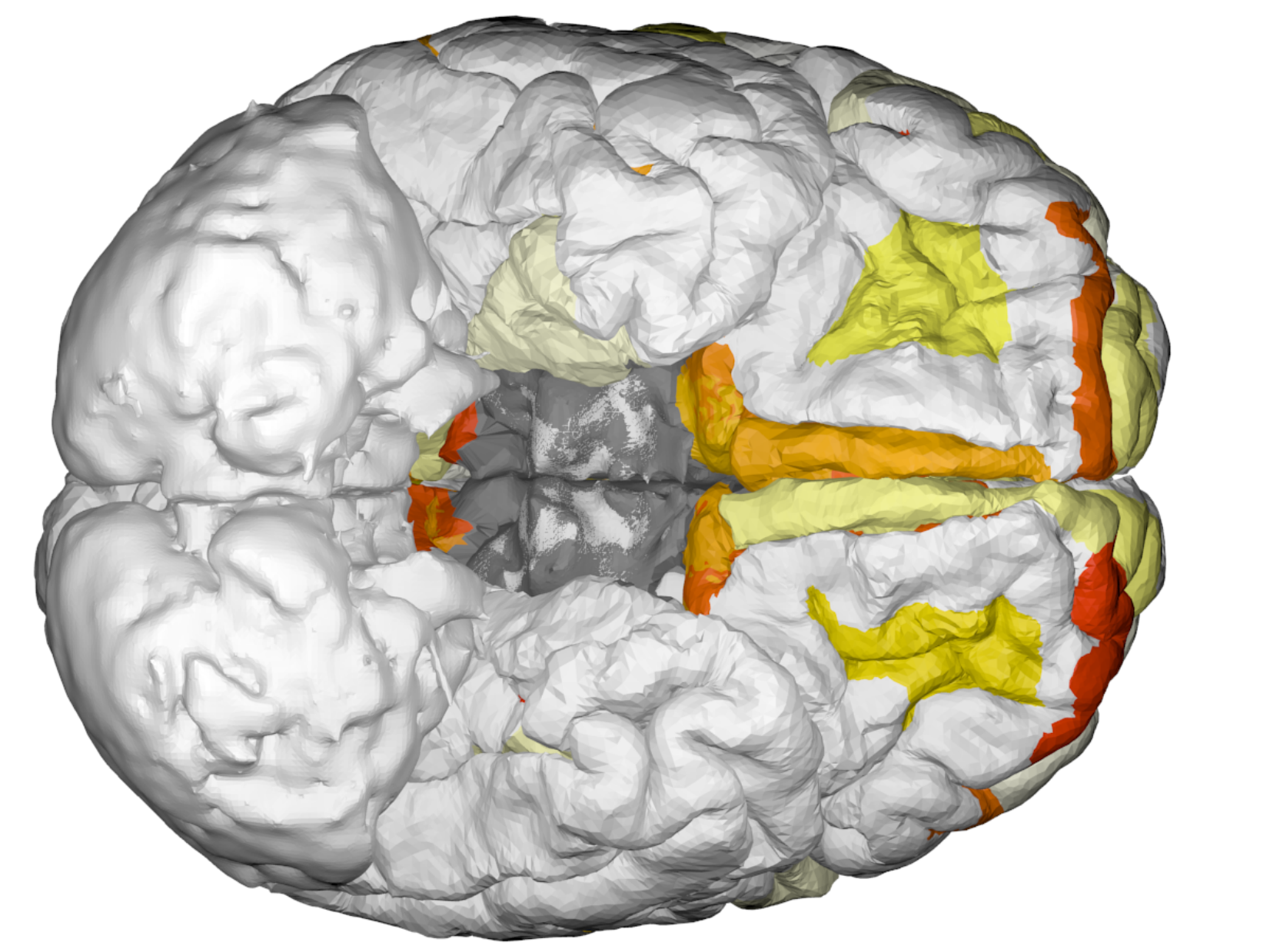} &
    \includegraphics[width=0.13\linewidth]{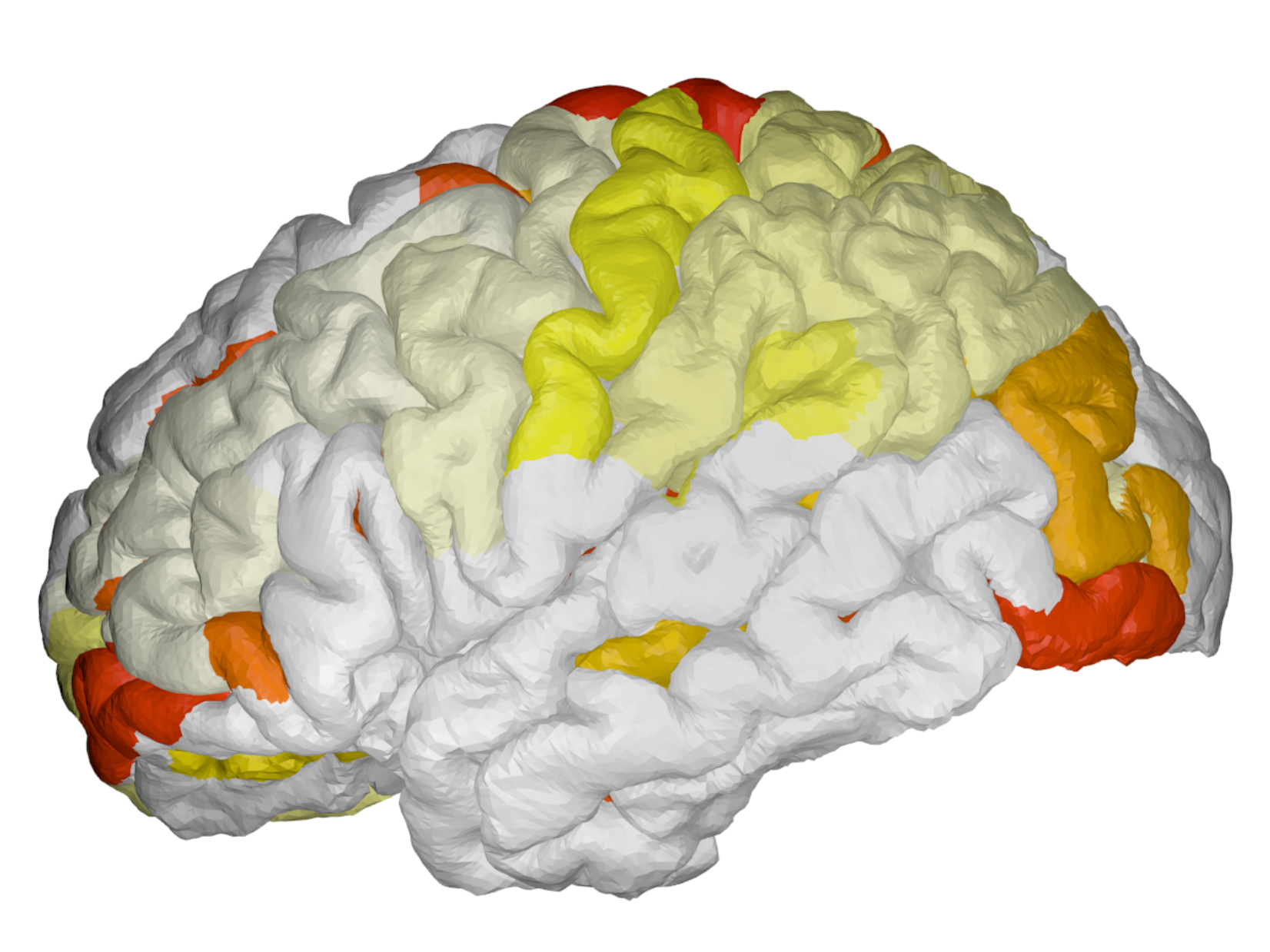} &
    \includegraphics[width=0.13\linewidth]{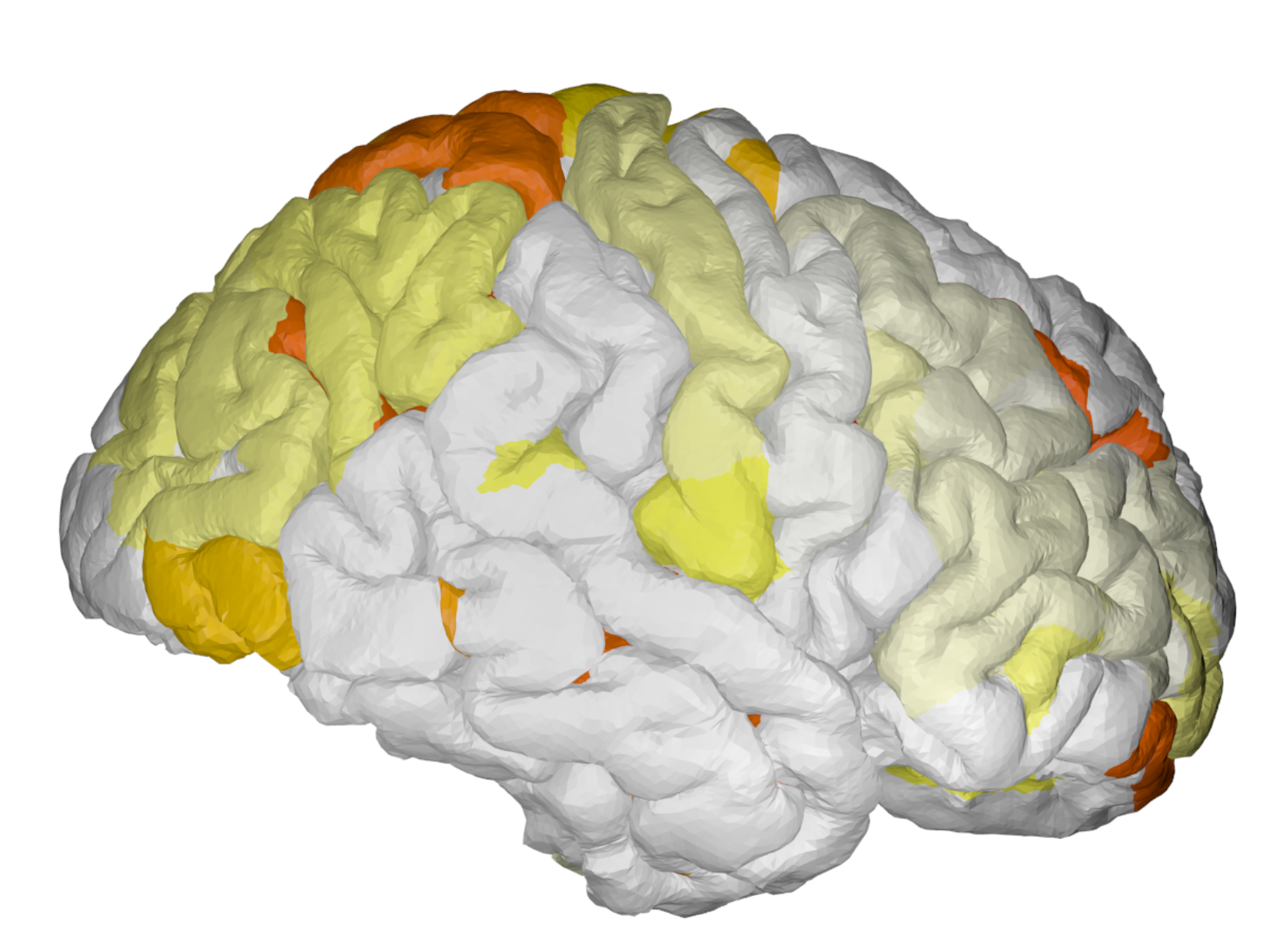} &
    \includegraphics[width=0.13\linewidth]{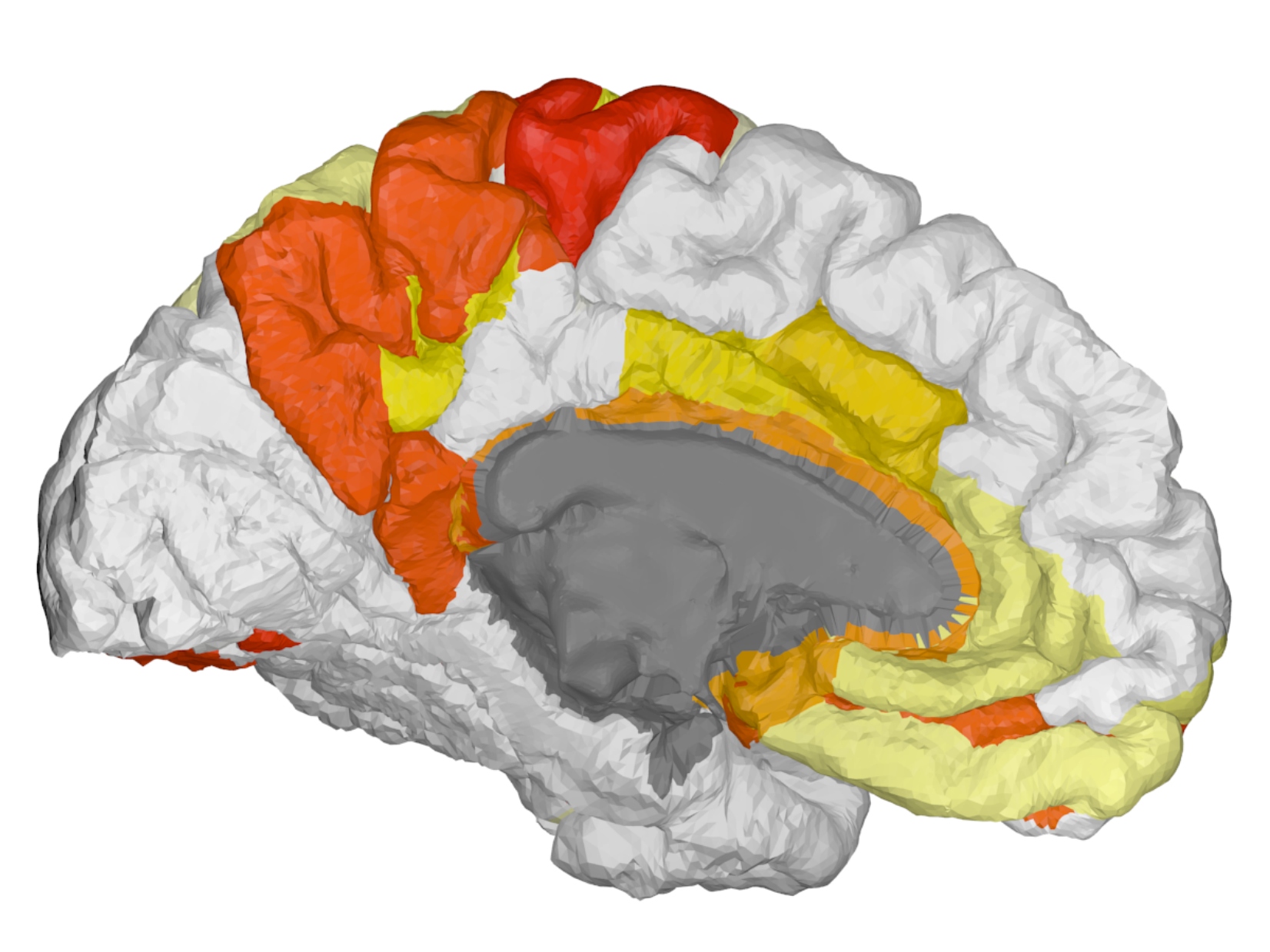} &
    \includegraphics[width=0.13\linewidth]{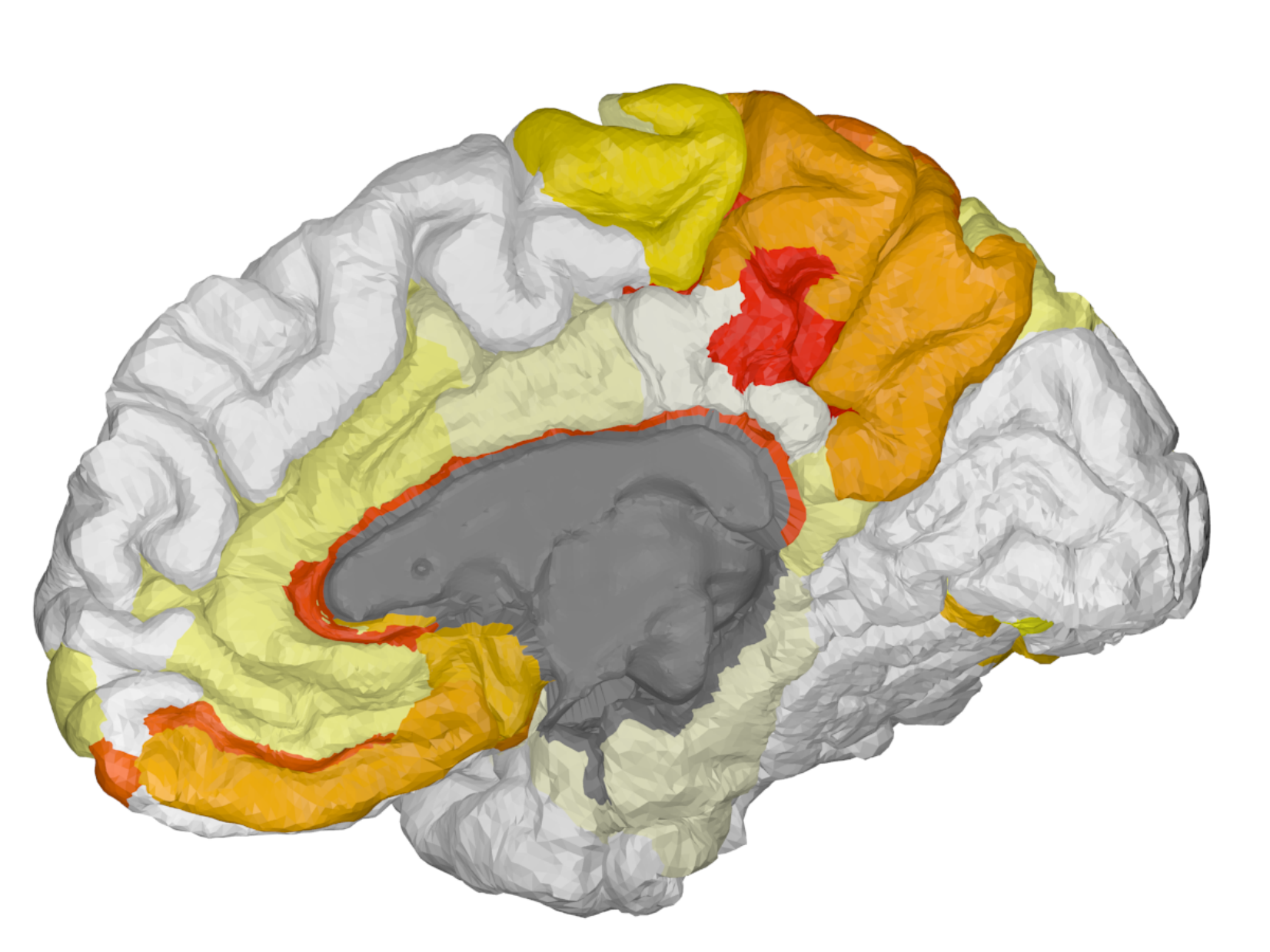} & 
    \includegraphics[width=0.13\linewidth]{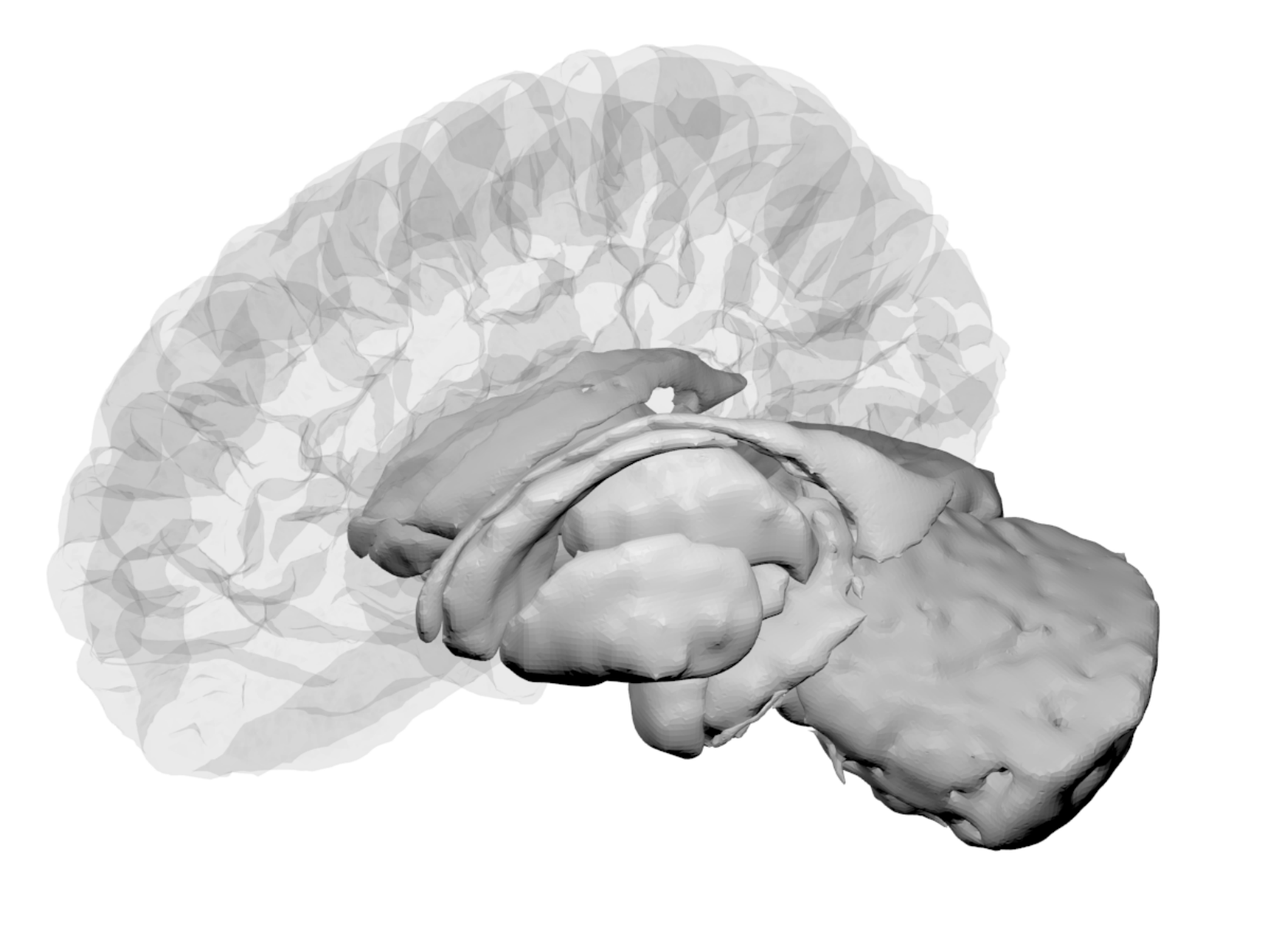} &
    \raisebox{-0.5\height}[0pt][0pt]{\includegraphics[width=0.06\textwidth]{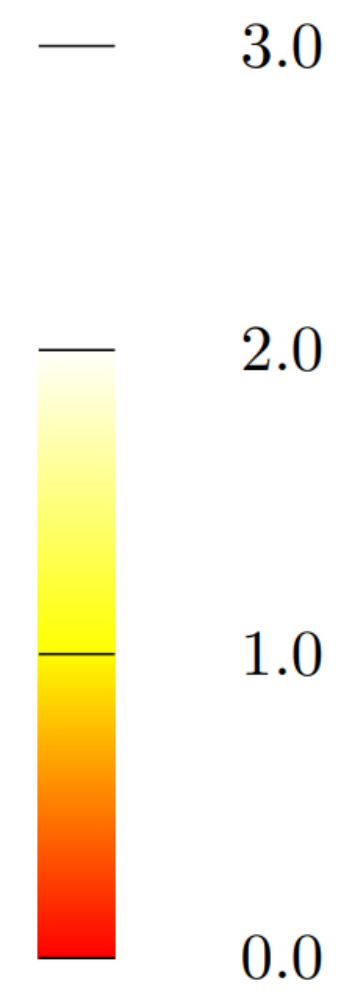}}\\ \vspace{-0.2cm}
    \raisebox{4\height}[0pt][0pt]{\textbf{LSAP-H}} &
    \includegraphics[width=0.13\linewidth]{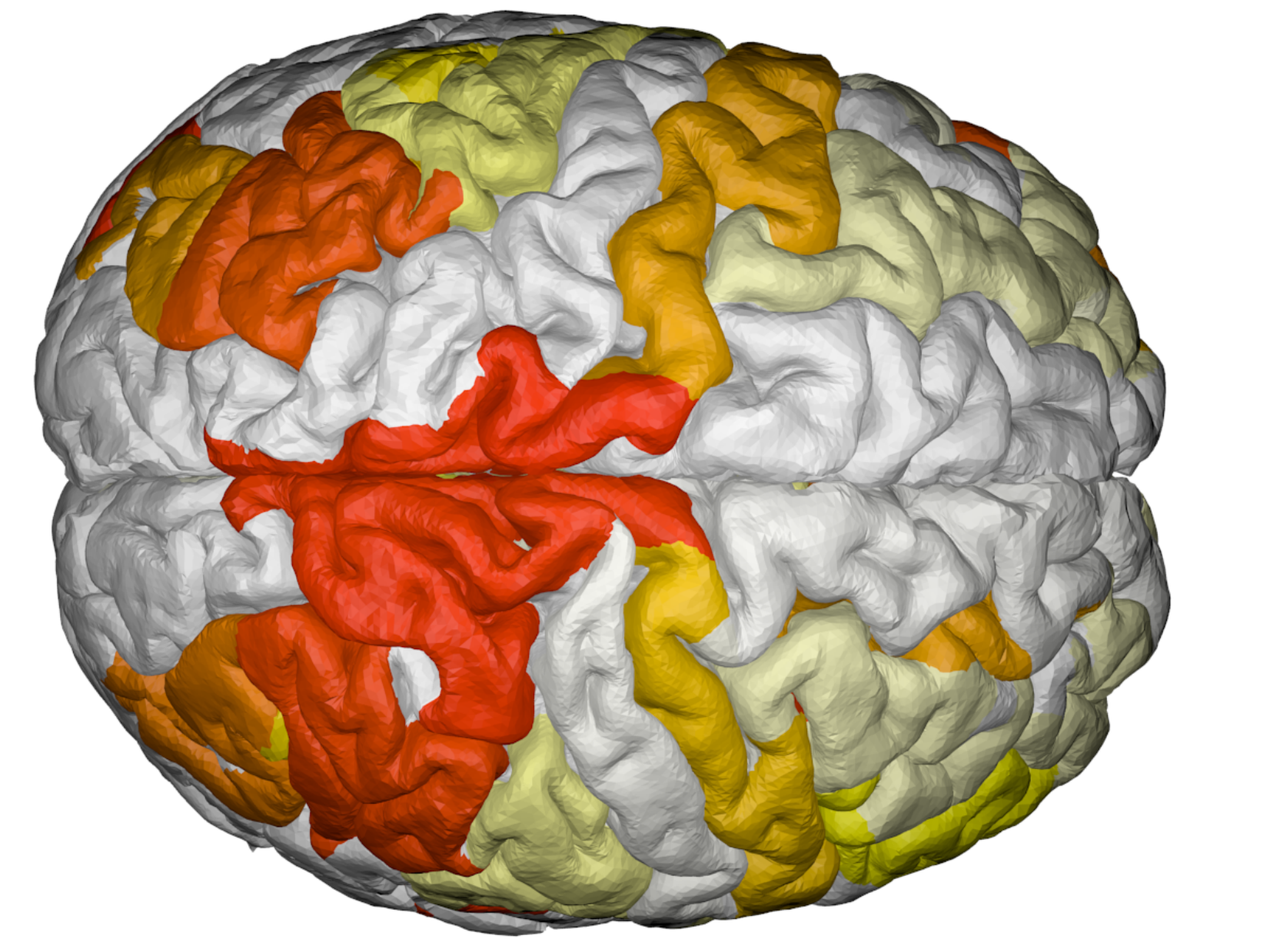} &
    \includegraphics[width=0.13\linewidth]{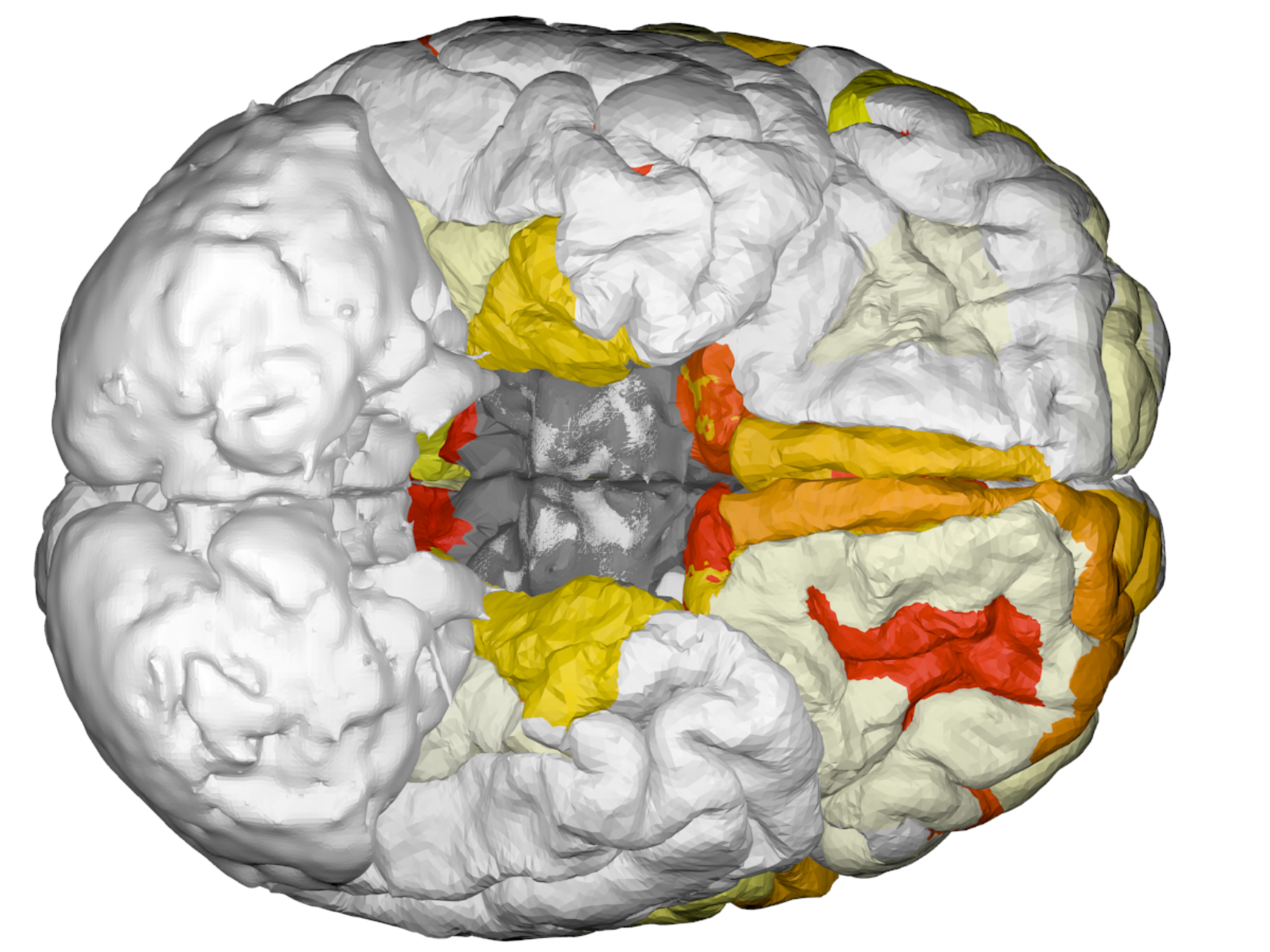} &
    \includegraphics[width=0.13\linewidth]{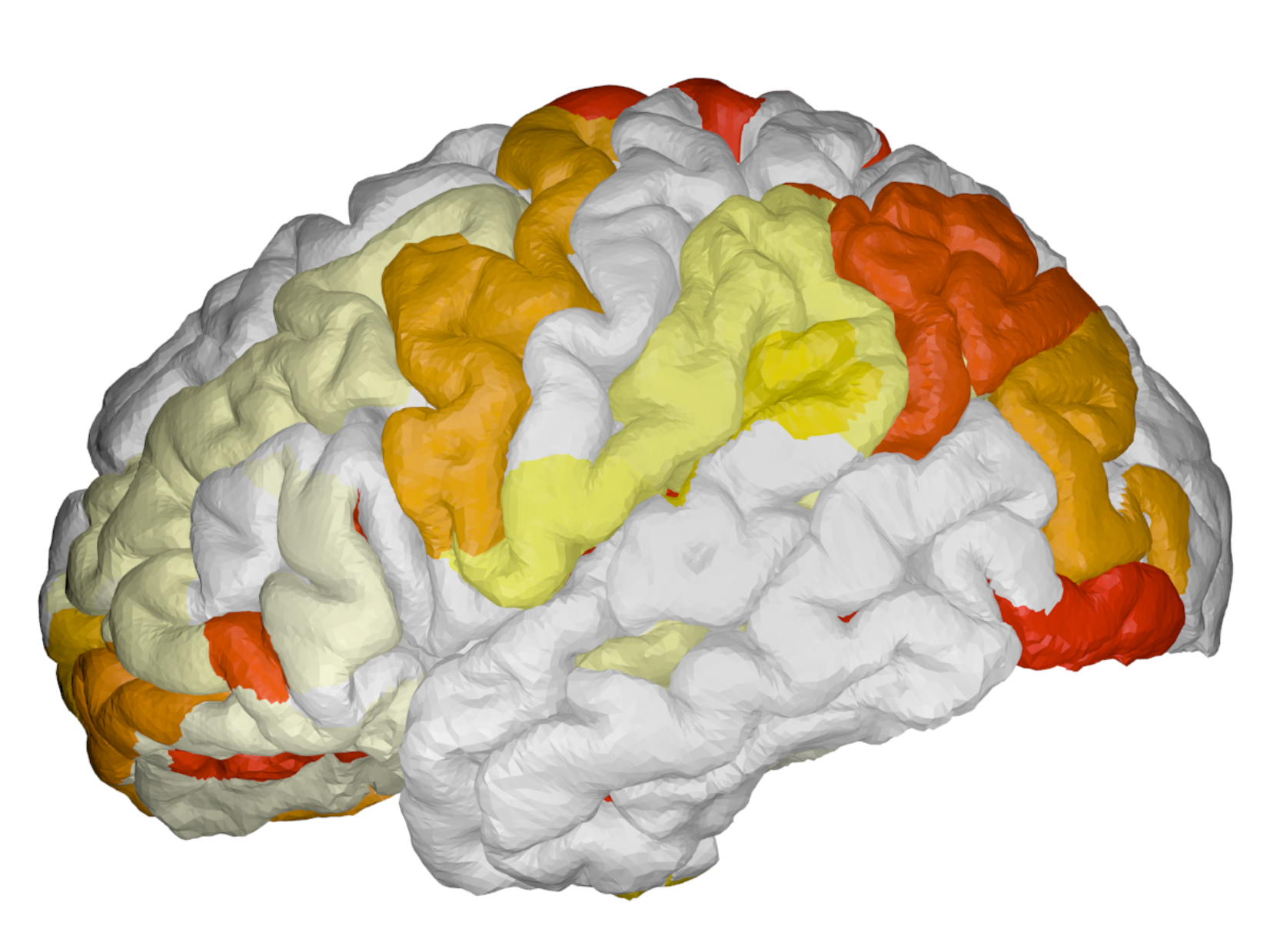} &
    \includegraphics[width=0.13\linewidth]{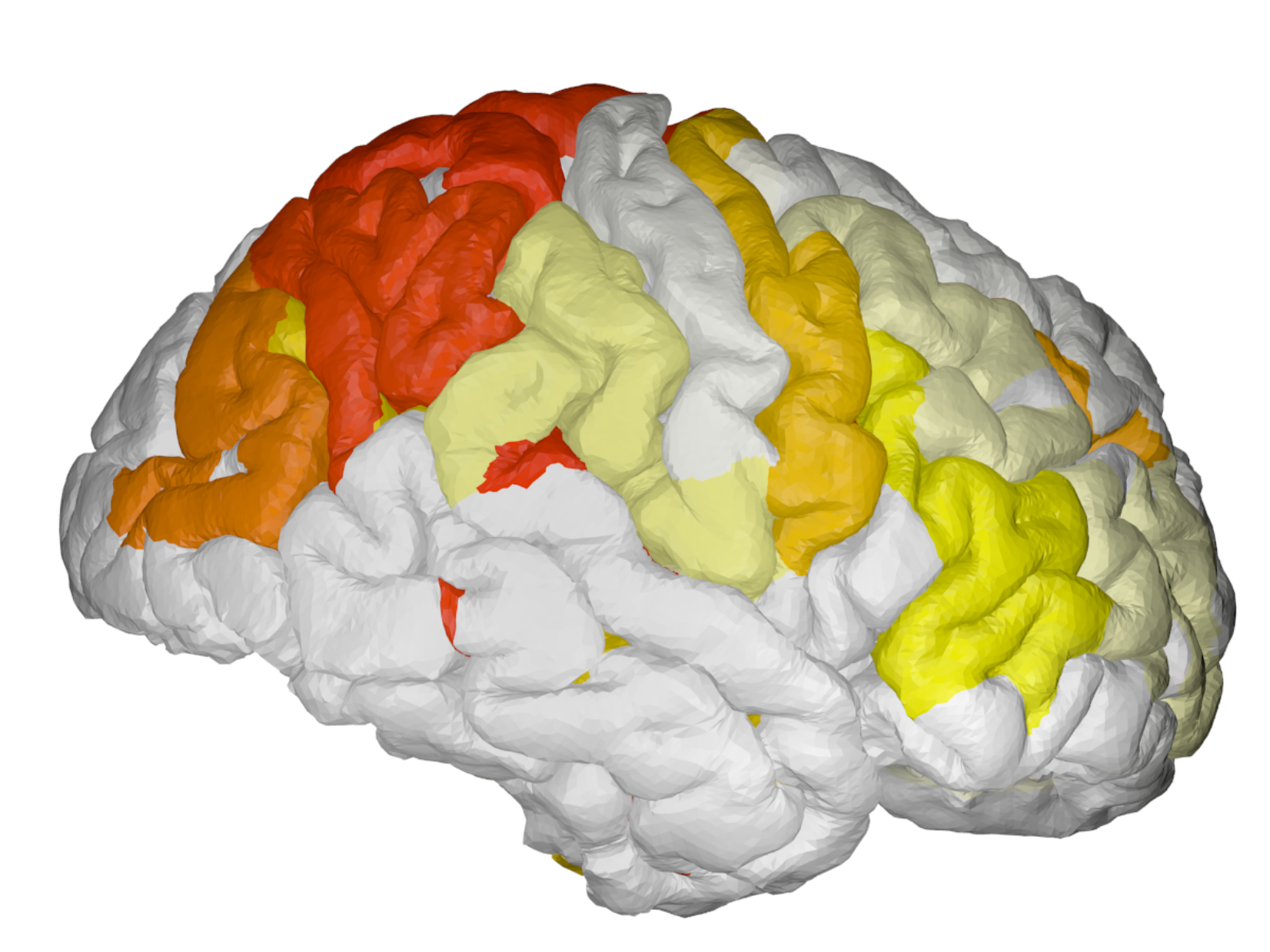} &
    \includegraphics[width=0.13\linewidth]{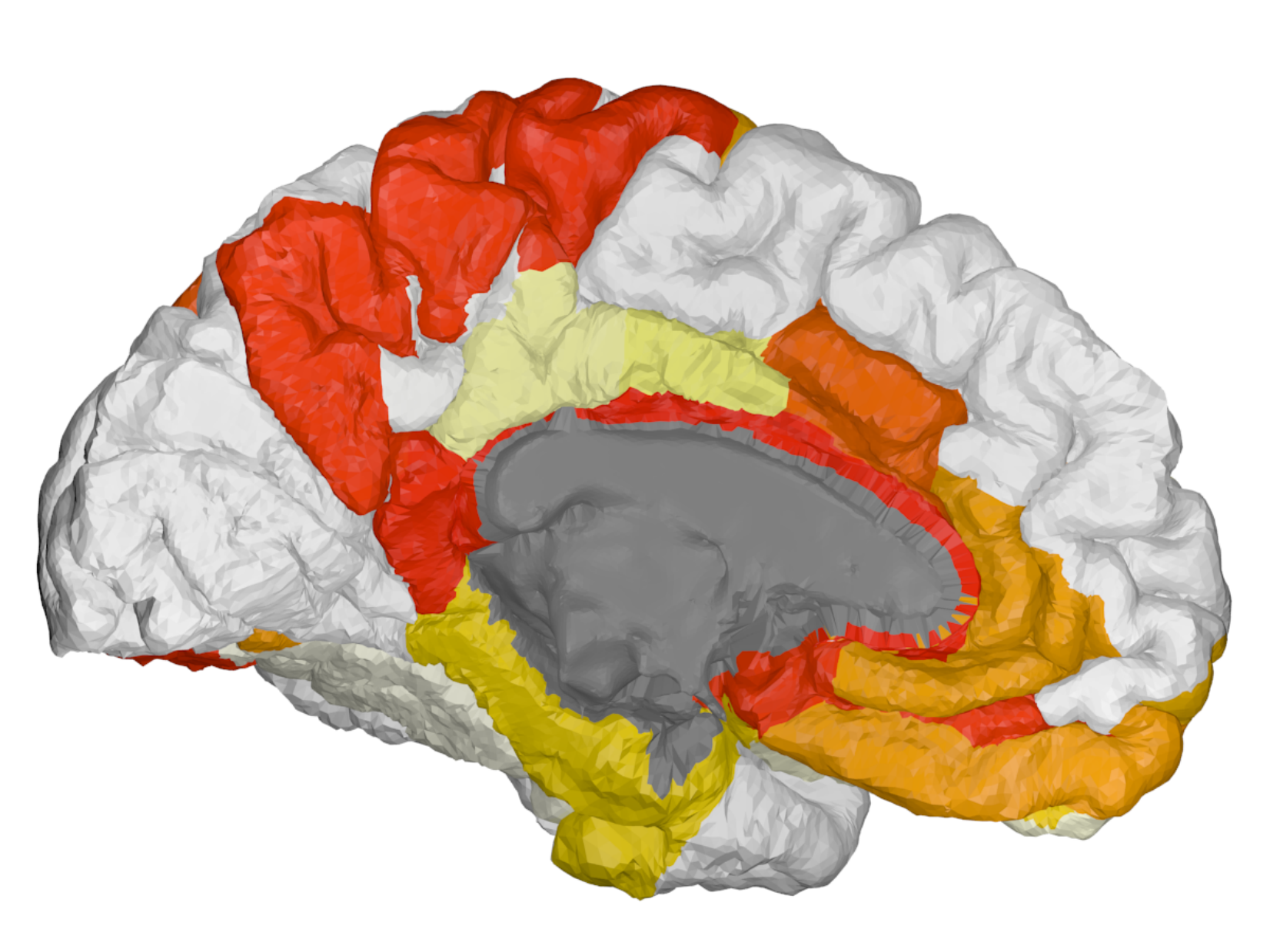} &
    \includegraphics[width=0.13\linewidth]{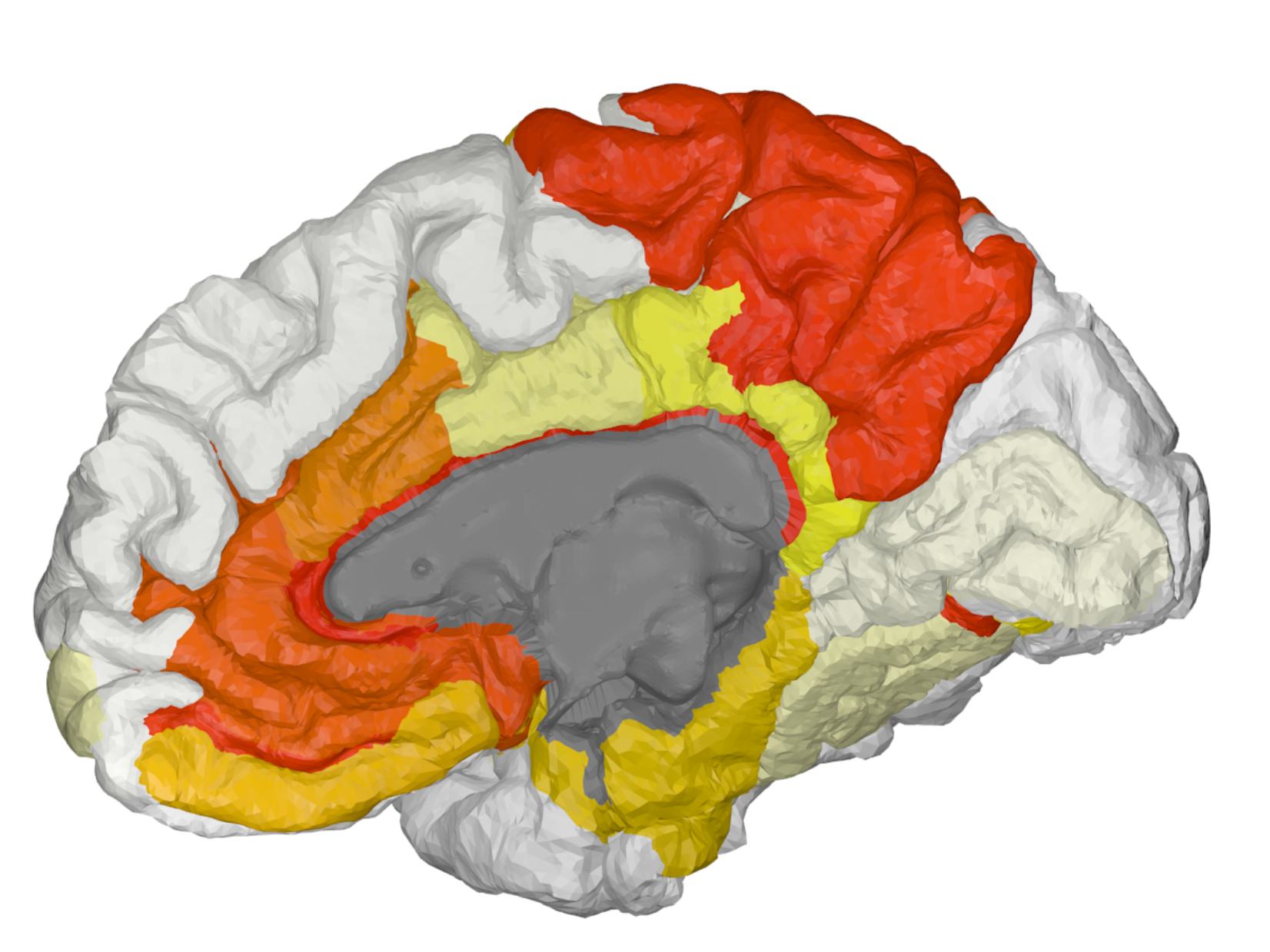} & 
    \includegraphics[width=0.13\linewidth]{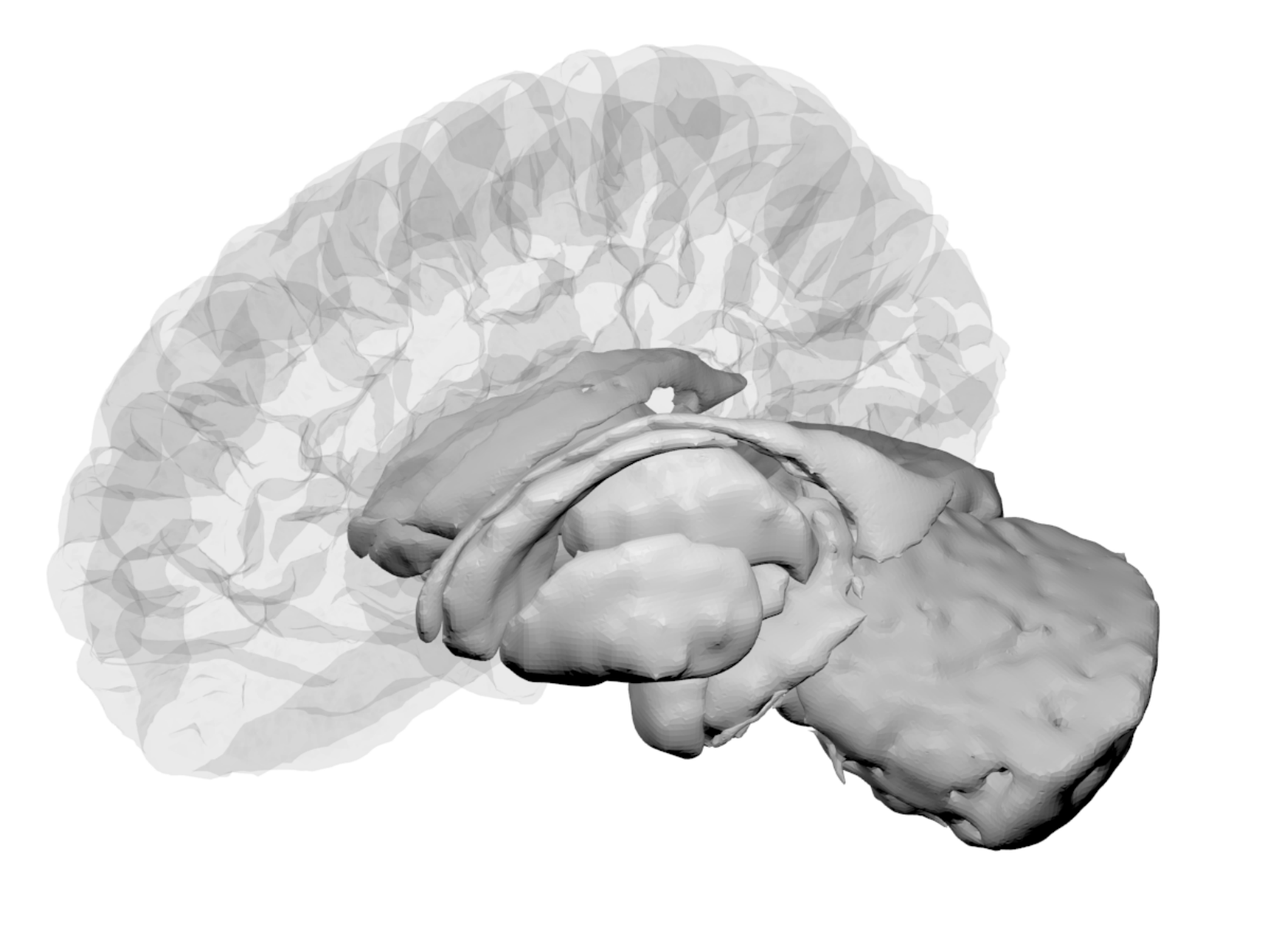} &
    \\ \vspace{-0.2cm}
    \raisebox{4\height}[0pt][0pt]{\textbf{LSAP-L}} &
    \includegraphics[width=0.13\linewidth]{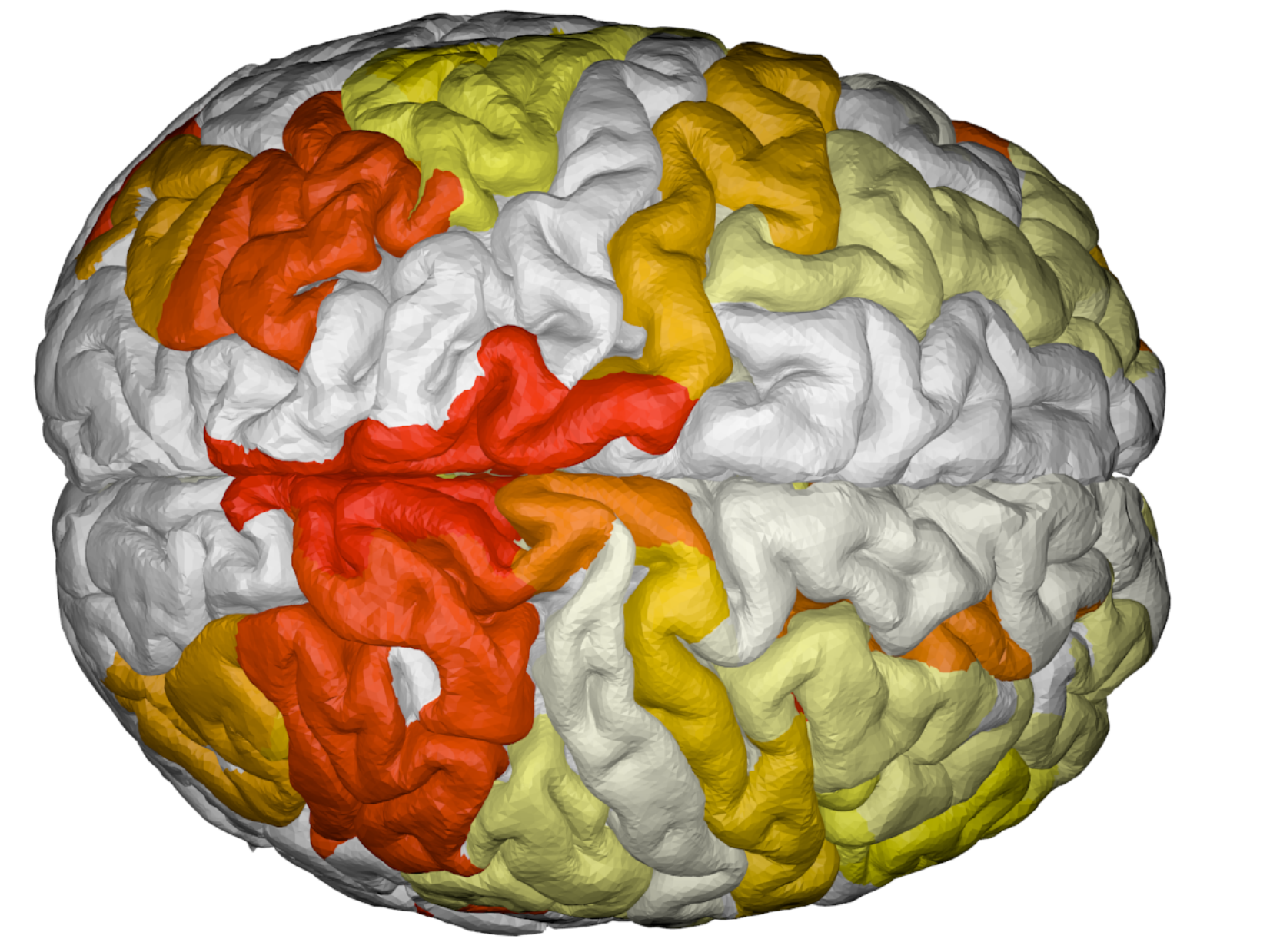} &
    \includegraphics[width=0.13\linewidth]{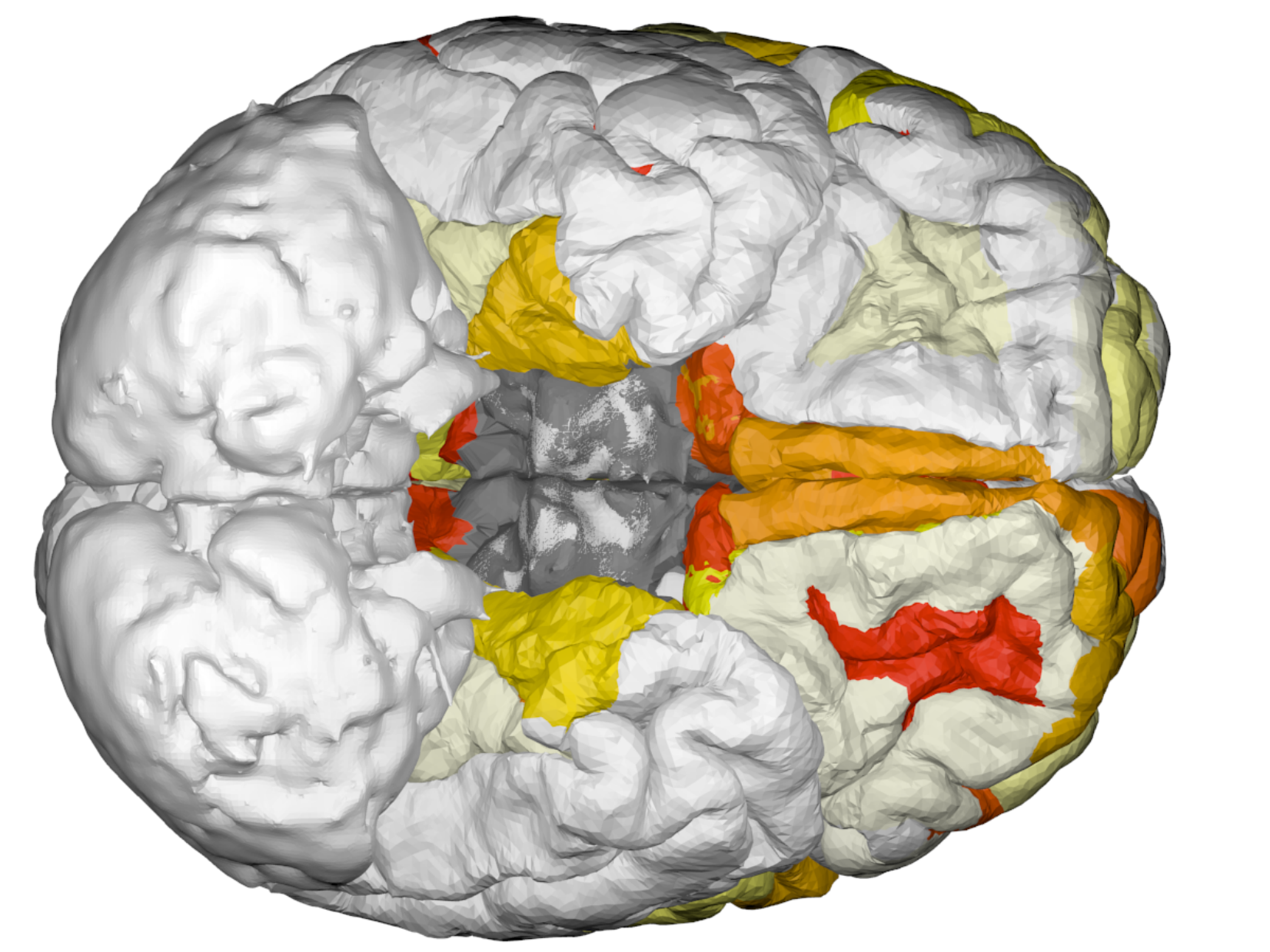} &
    \includegraphics[width=0.13\linewidth]{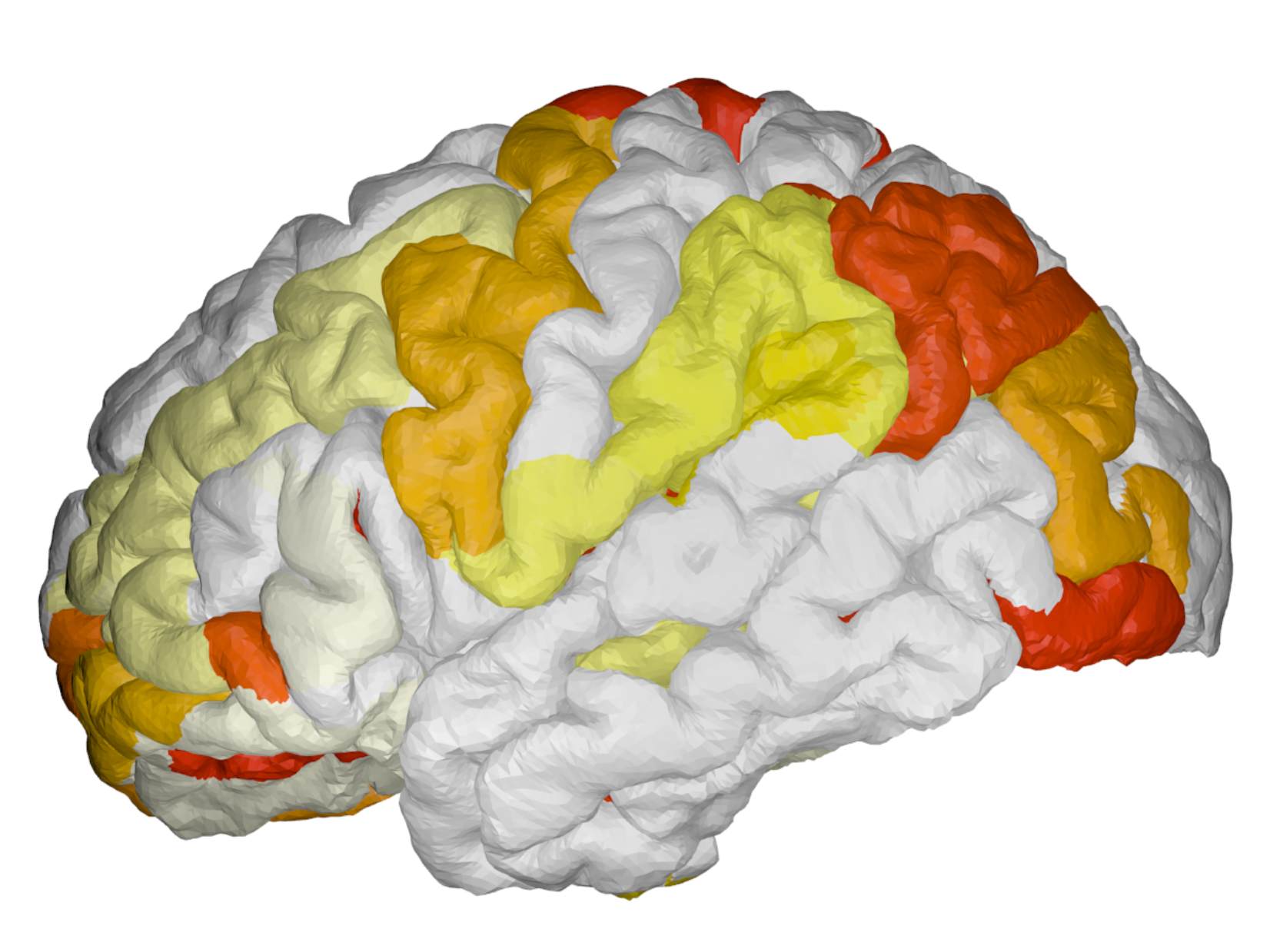} &
    \includegraphics[width=0.13\linewidth]{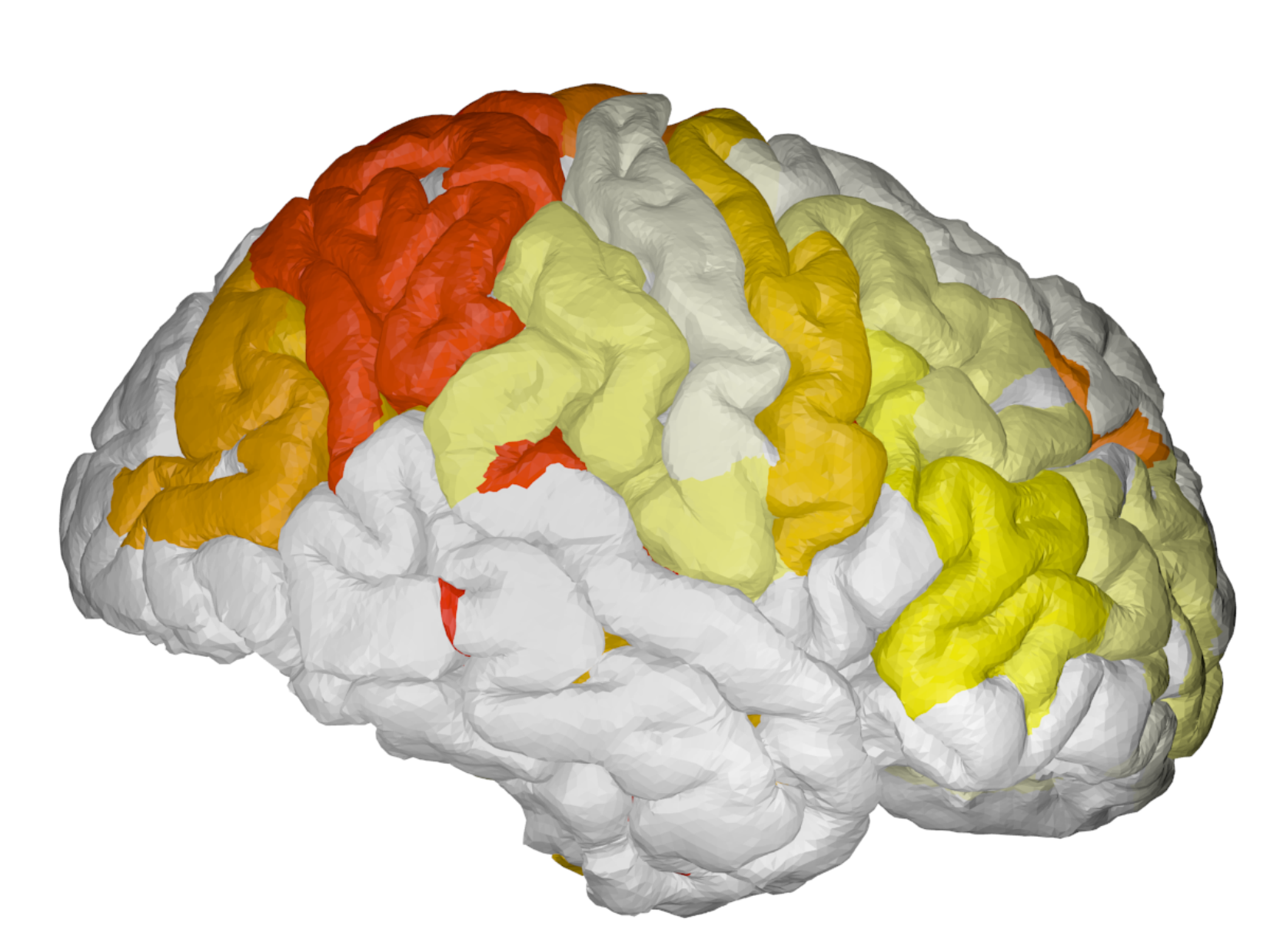} &
    \includegraphics[width=0.13\linewidth]{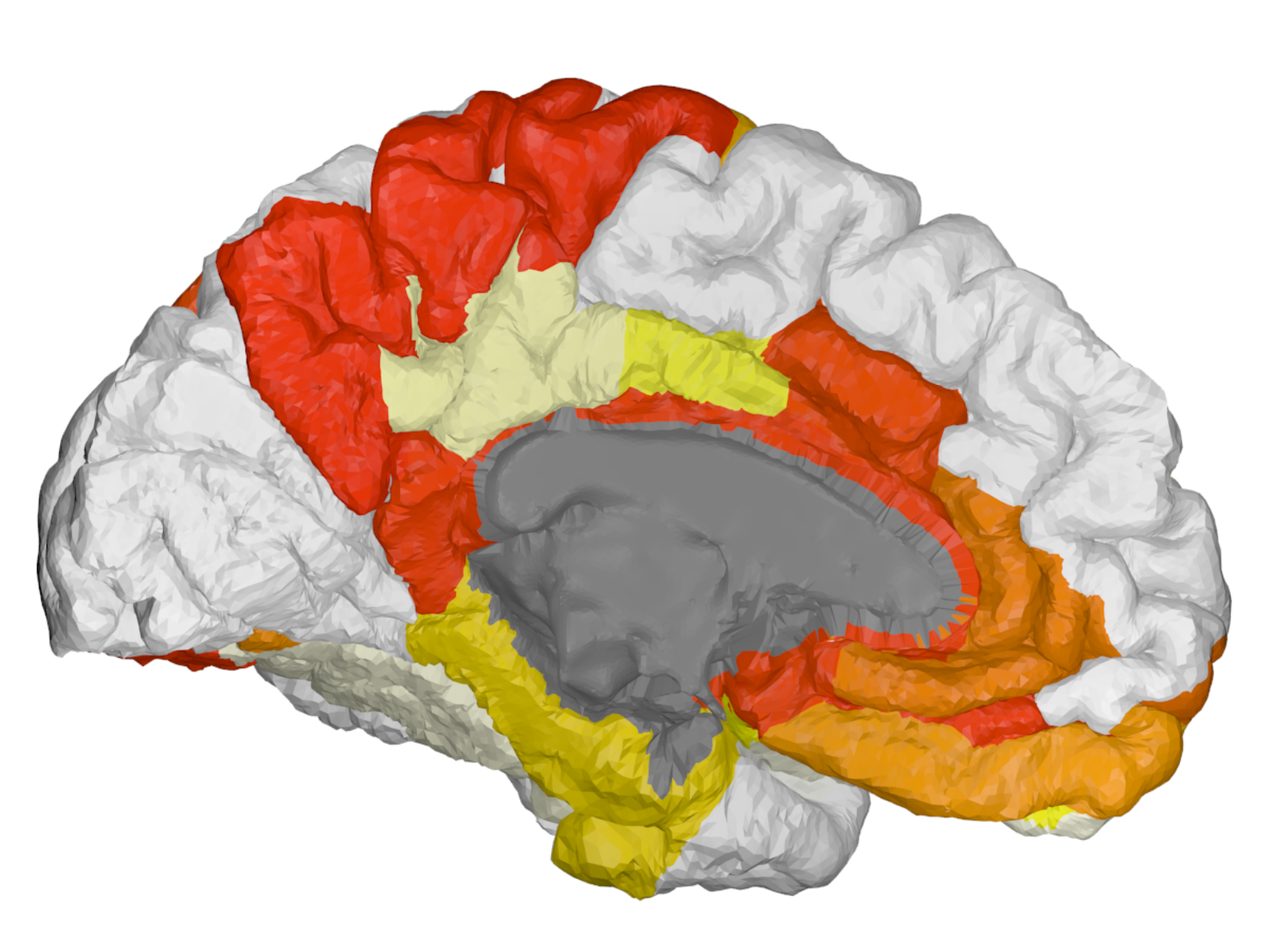} &
    \includegraphics[width=0.13\linewidth]{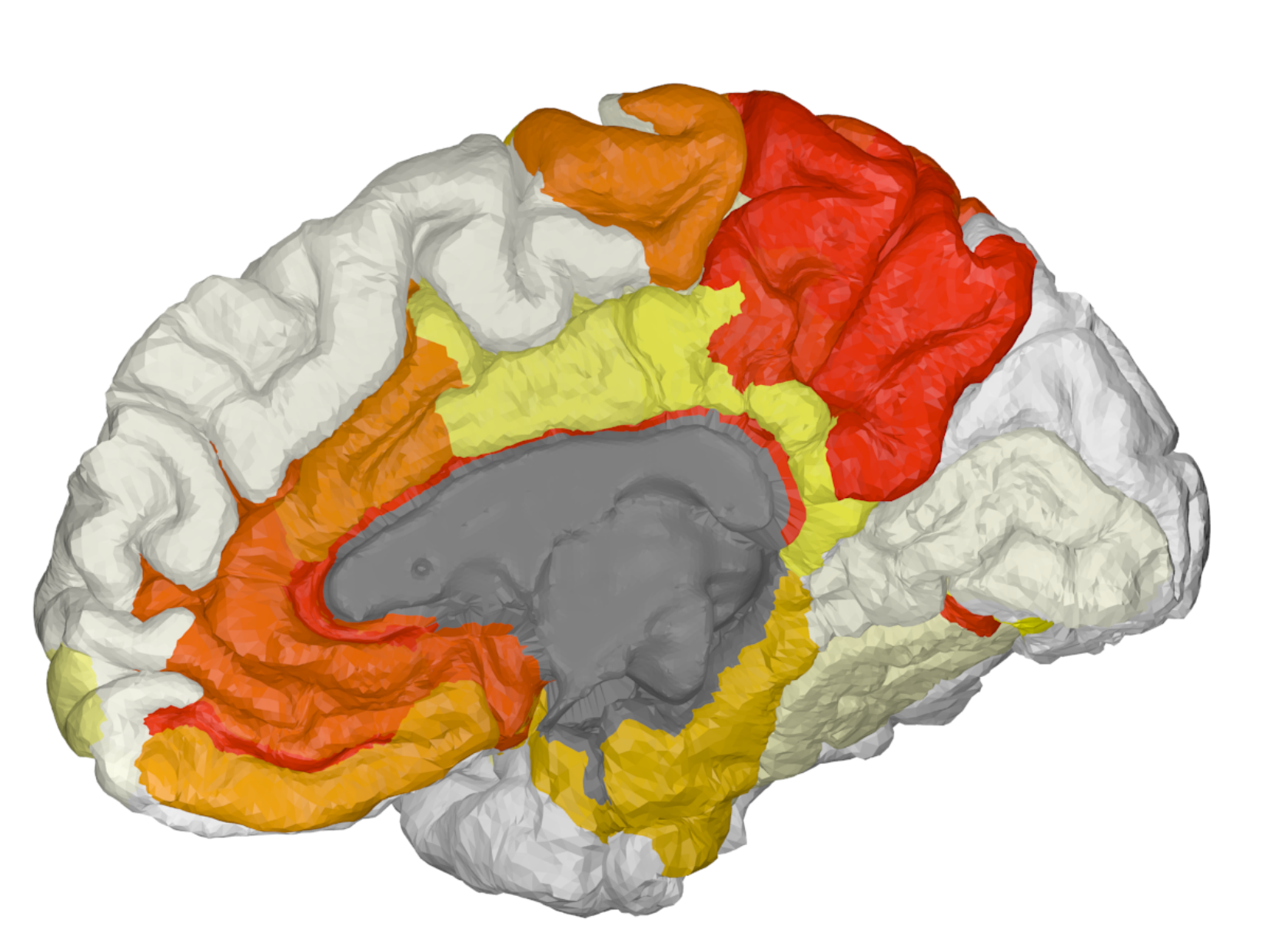} & 
    \includegraphics[width=0.13\linewidth]{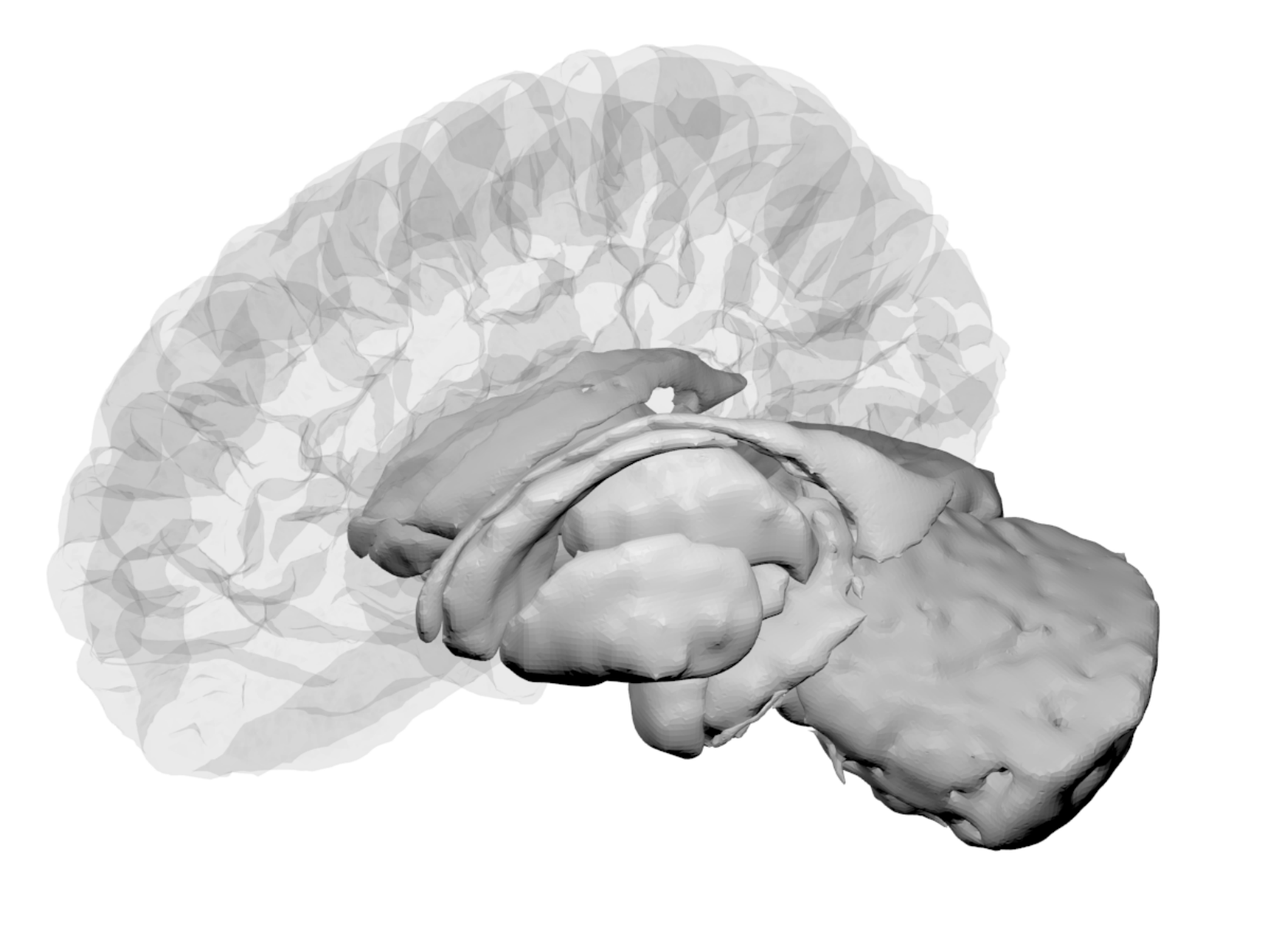} &\\
  \end{tabular}}}
  \vspace{5pt}
  \caption{Visualization of learned scales on the cortical and sub-cortical regions of a brain. 
  This visualization shows the scale of each RoI through the classification result using Cortical Thickness feature.}
\label{fig:Brain_Visualization_DT_supple}
\end{figure*}

\begin{figure*}[t!]
  \centering
  \renewcommand{\arraystretch}{1.0}
  \renewcommand{\tabcolsep}{0.05cm}
  \small{
  \scalebox{0.90}{
  \begin{tabular}{ccccccccl}
    & 
    \raisebox{1\height}[0pt][0pt]{\textbf{Top}} & \raisebox{1\height}[0pt][0pt]{\textbf{Bottom}} &
    \raisebox{1\height}[0pt][0pt]{\textbf{Outer-Left}} & \raisebox{1\height}[0pt][0pt]{\textbf{Outer-Right}} & \raisebox{1\height}[0pt][0pt]{\textbf{Inner-Left}} & \raisebox{1\height}[0pt][0pt]{\textbf{Inner-Right}} & \raisebox{1\height}[0pt][0pt]{\textbf{Sub-Cortical}} & \\ \vspace{-0.2cm}
    \raisebox{4\height}[0pt][0pt]{\textbf{Exact}} & 
    \includegraphics[width=0.13\linewidth]{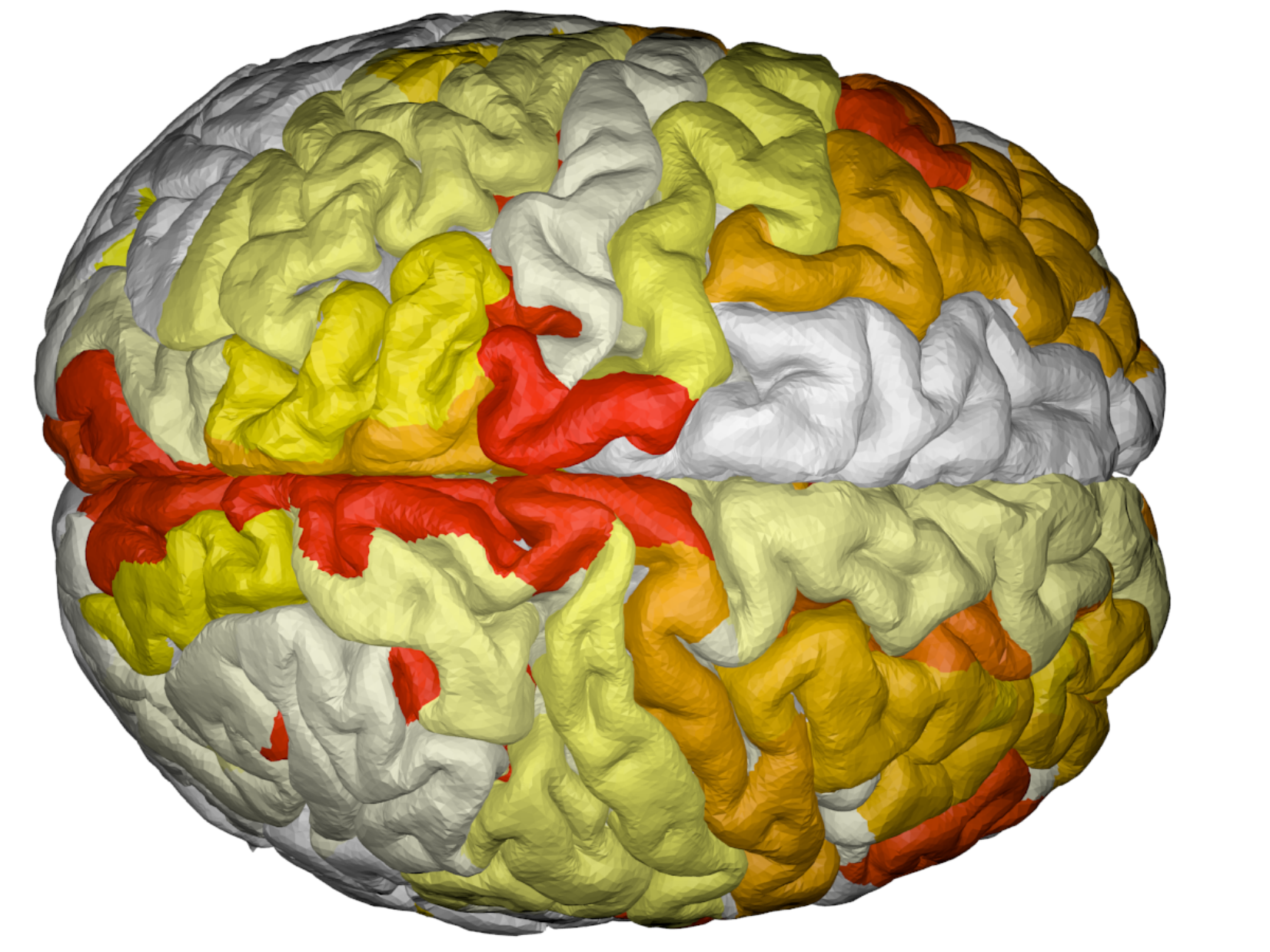} &
    \includegraphics[width=0.13\linewidth]{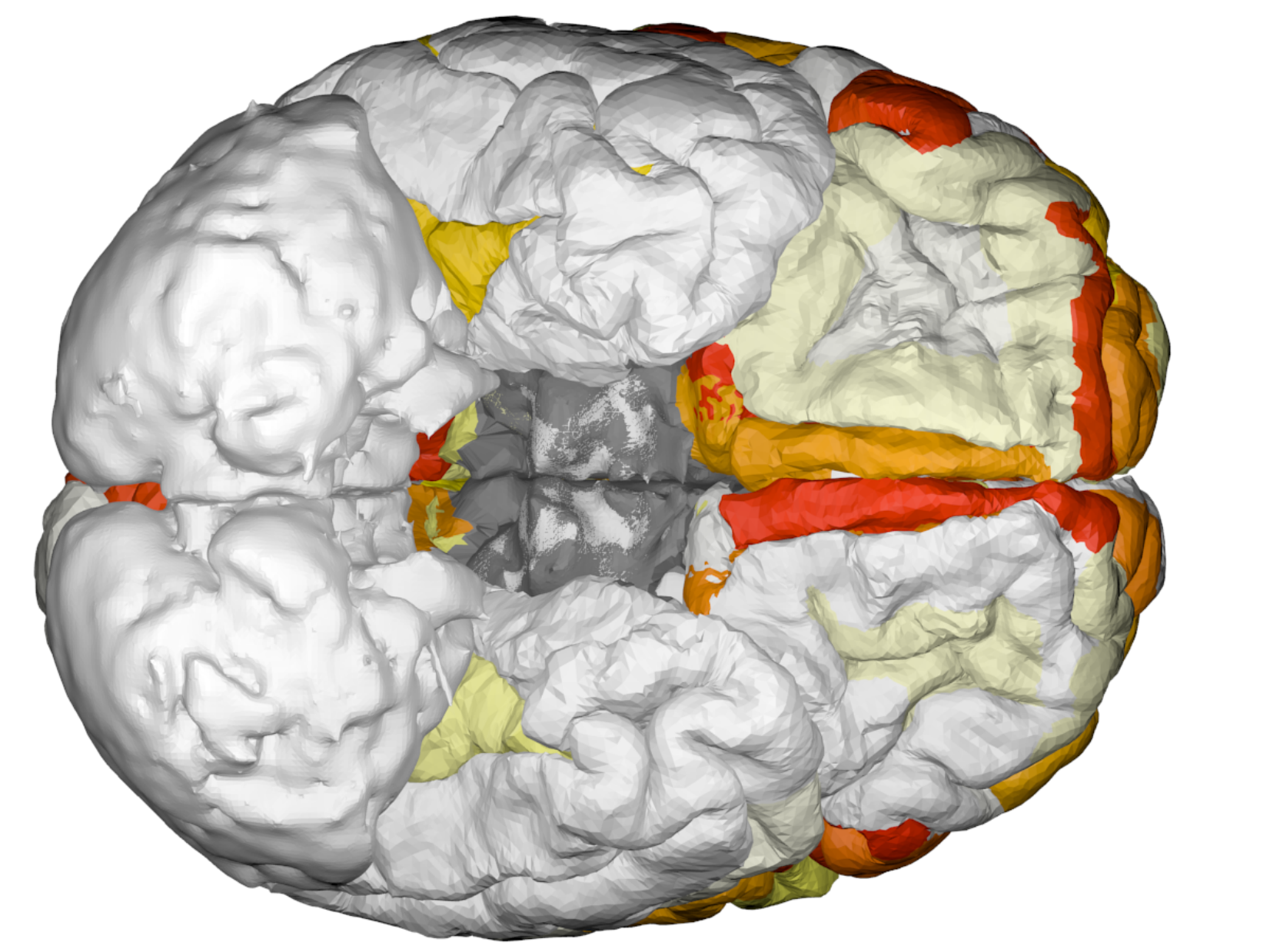} &
    \includegraphics[width=0.13\linewidth]{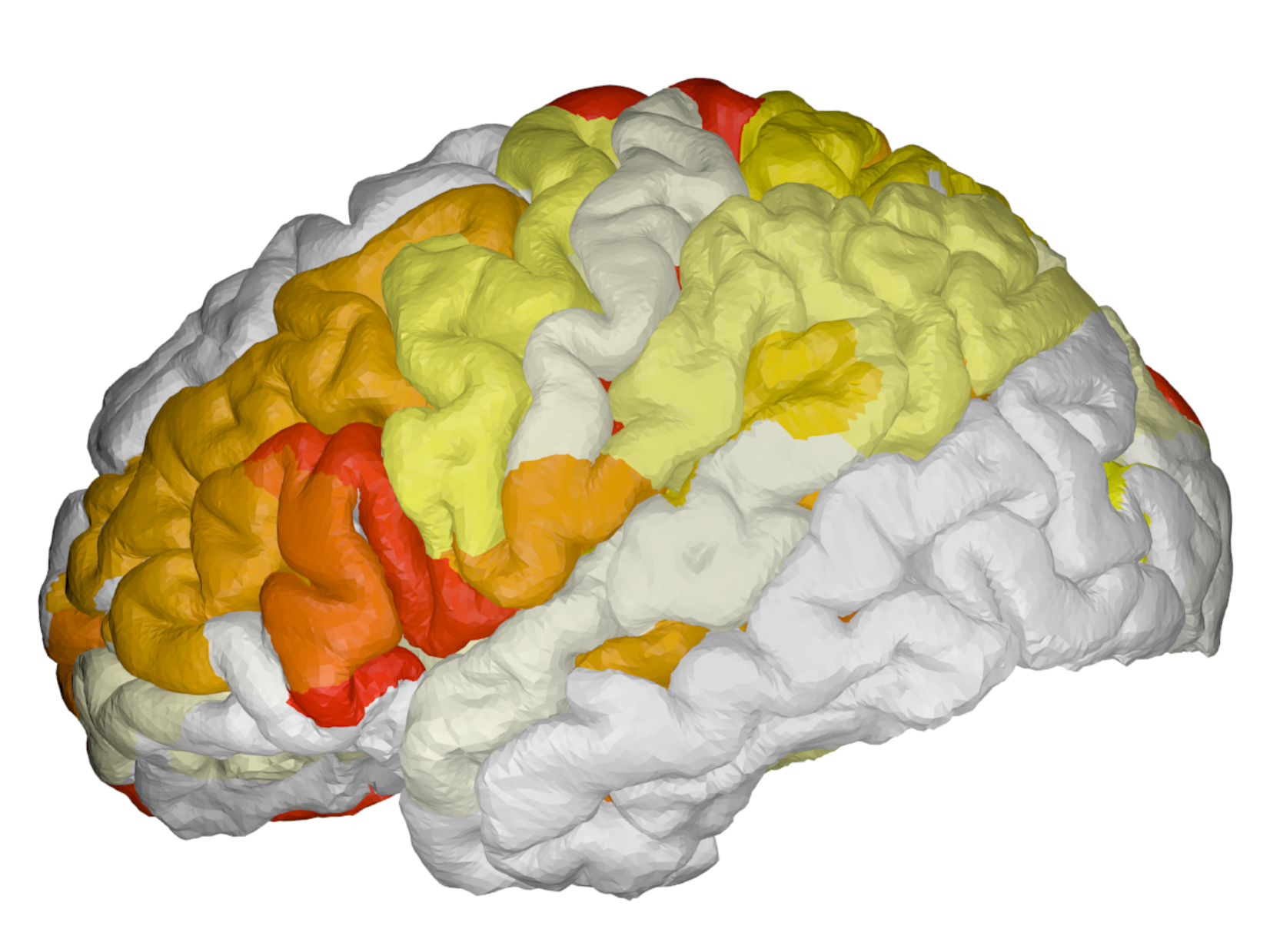} &
    \includegraphics[width=0.13\linewidth]{cortical-outer-right-hemisphere_fd_exa.pdf} &
    \includegraphics[width=0.13\linewidth]{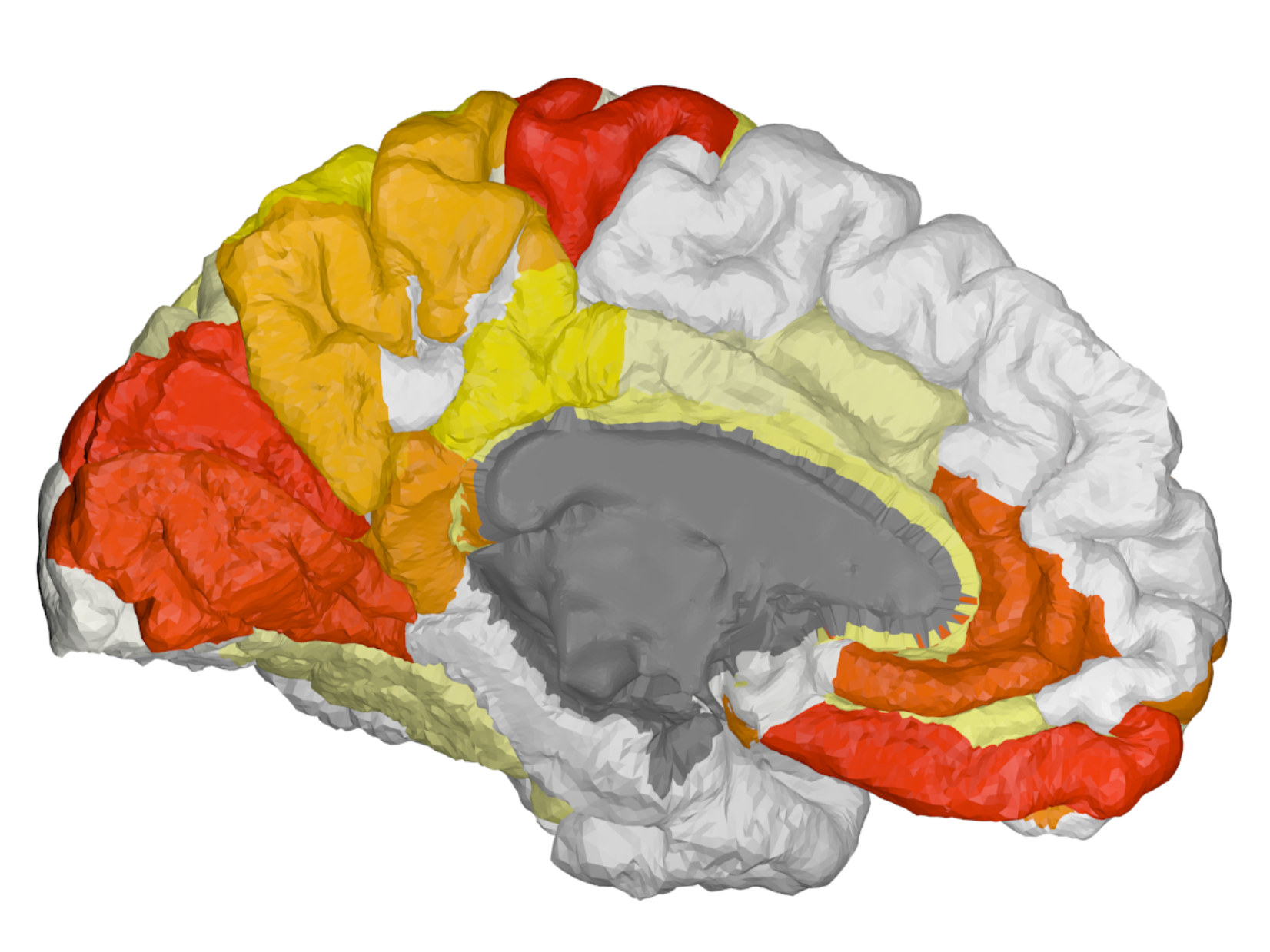} &
    \includegraphics[width=0.13\linewidth]{cortical-inner-right-hemisphere_fd_exa.pdf} & 
    \includegraphics[width=0.13\linewidth]{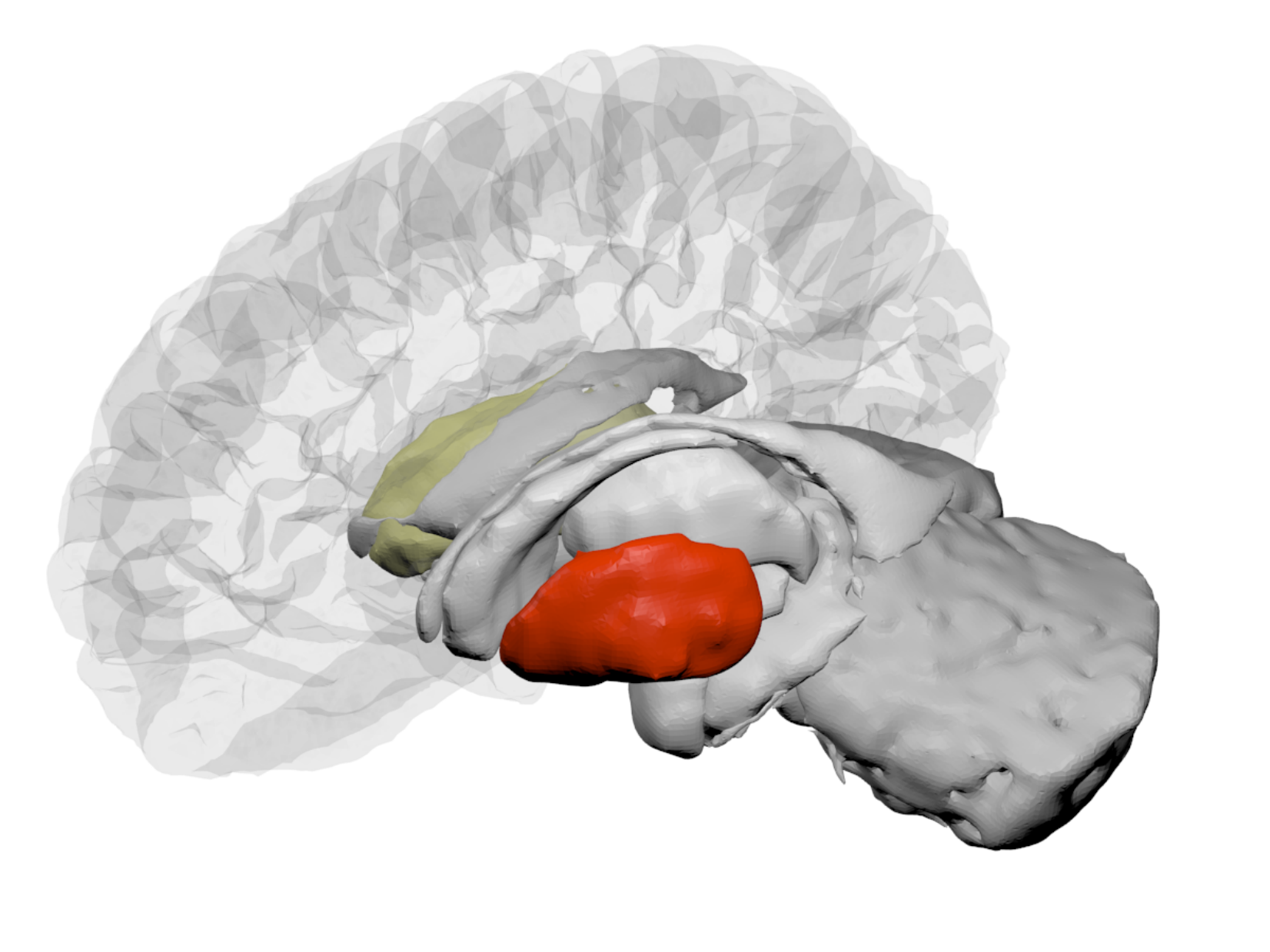} &\\ \vspace{-0.2cm}
    \raisebox{4\height}[0pt][0pt]{\textbf{LSAP-C}} &
    \includegraphics[width=0.13\linewidth]{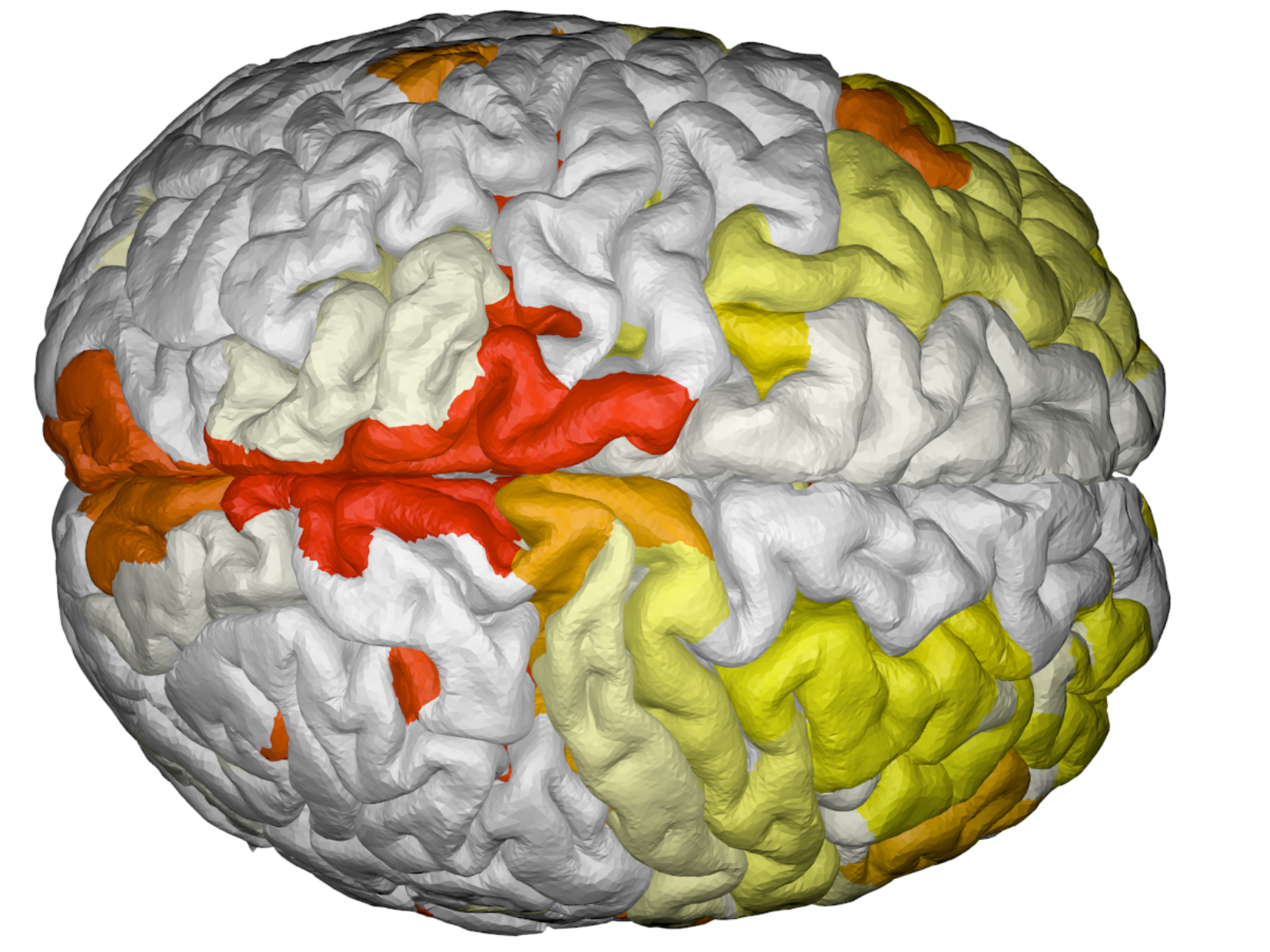} &
    \includegraphics[width=0.13\linewidth]{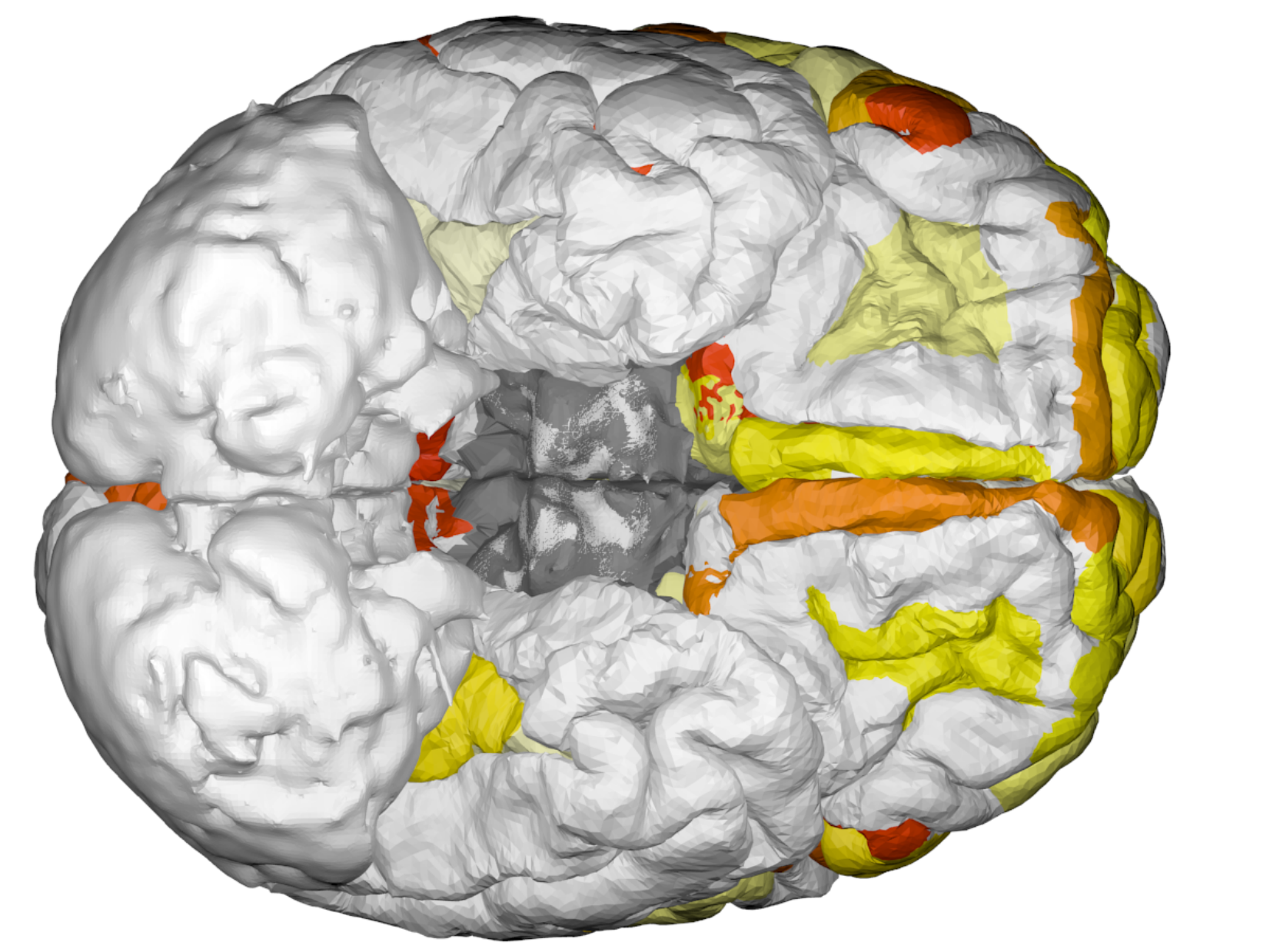} &
    \includegraphics[width=0.13\linewidth]{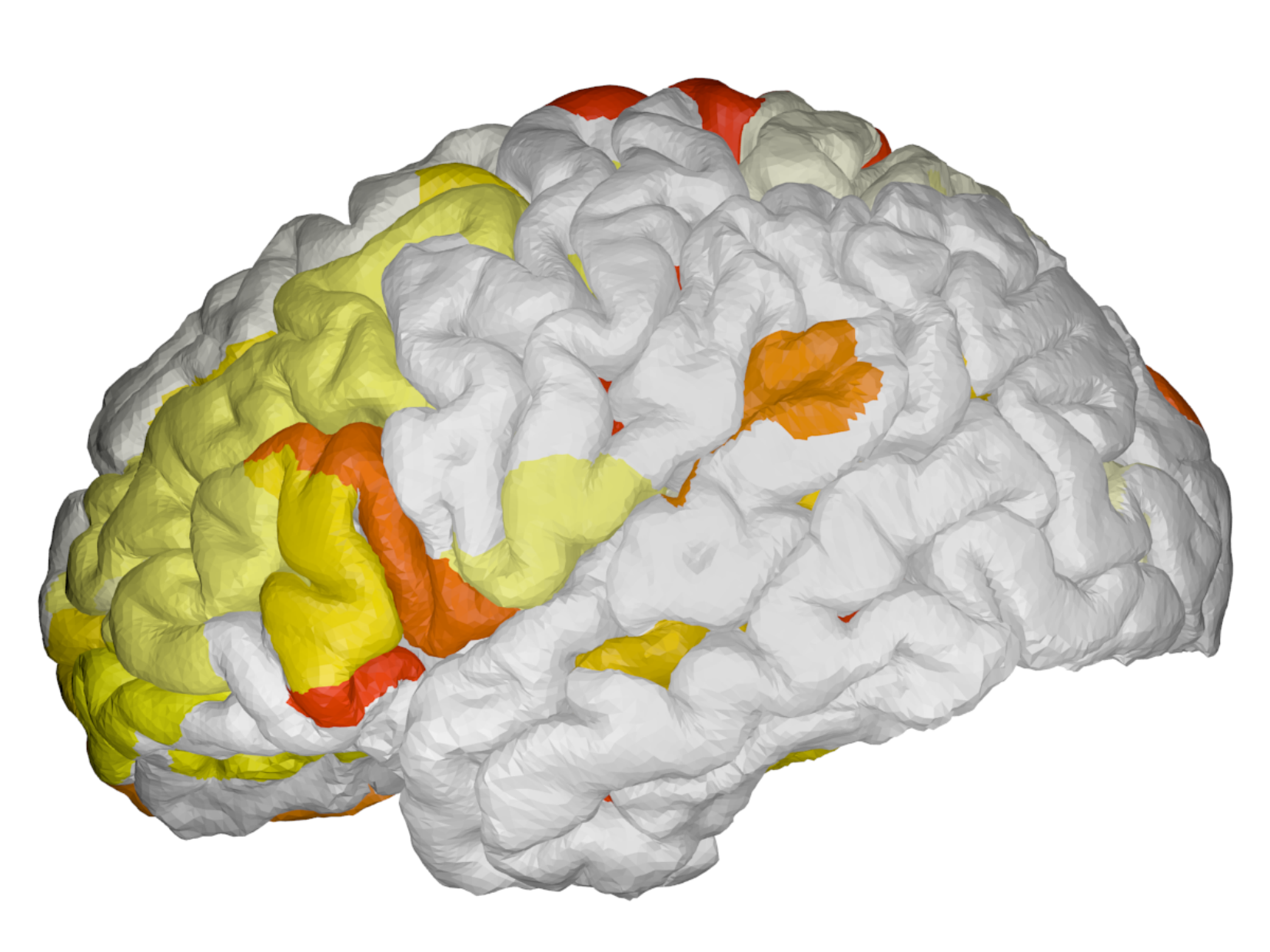} &
    \includegraphics[width=0.13\linewidth]{cortical-outer-right-hemisphere_fd_che.pdf} &
    \includegraphics[width=0.13\linewidth]{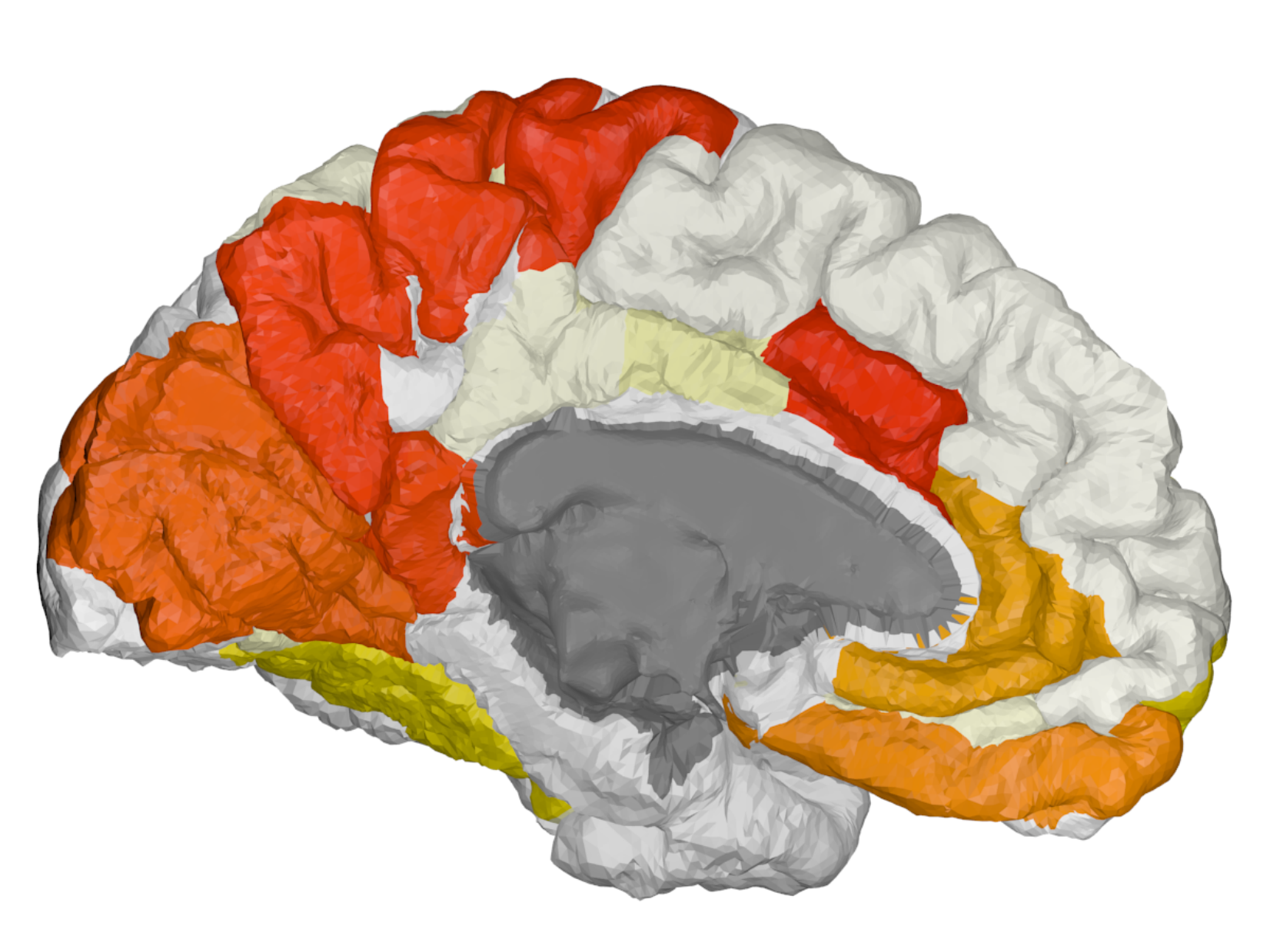} &
    \includegraphics[width=0.13\linewidth]{cortical-inner-right-hemisphere_fd_che.pdf} & 
    \includegraphics[width=0.13\linewidth]{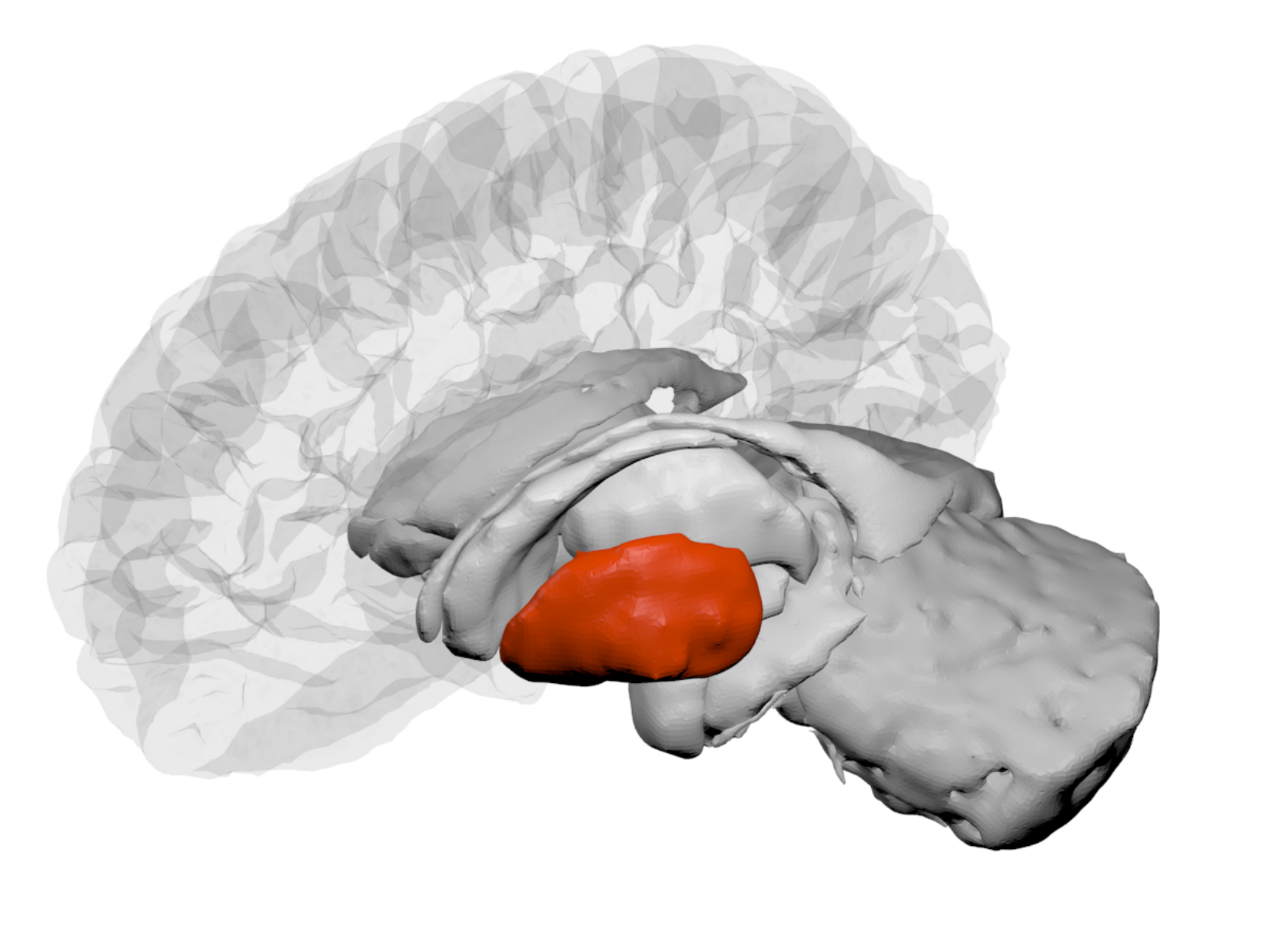} &
    \raisebox{-0.5\height}[0pt][0pt]{\includegraphics[width=0.06\textwidth]{colorbar.pdf}}\\ \vspace{-0.2cm}
    \raisebox{4\height}[0pt][0pt]{\textbf{LSAP-H}} &
    \includegraphics[width=0.13\linewidth]{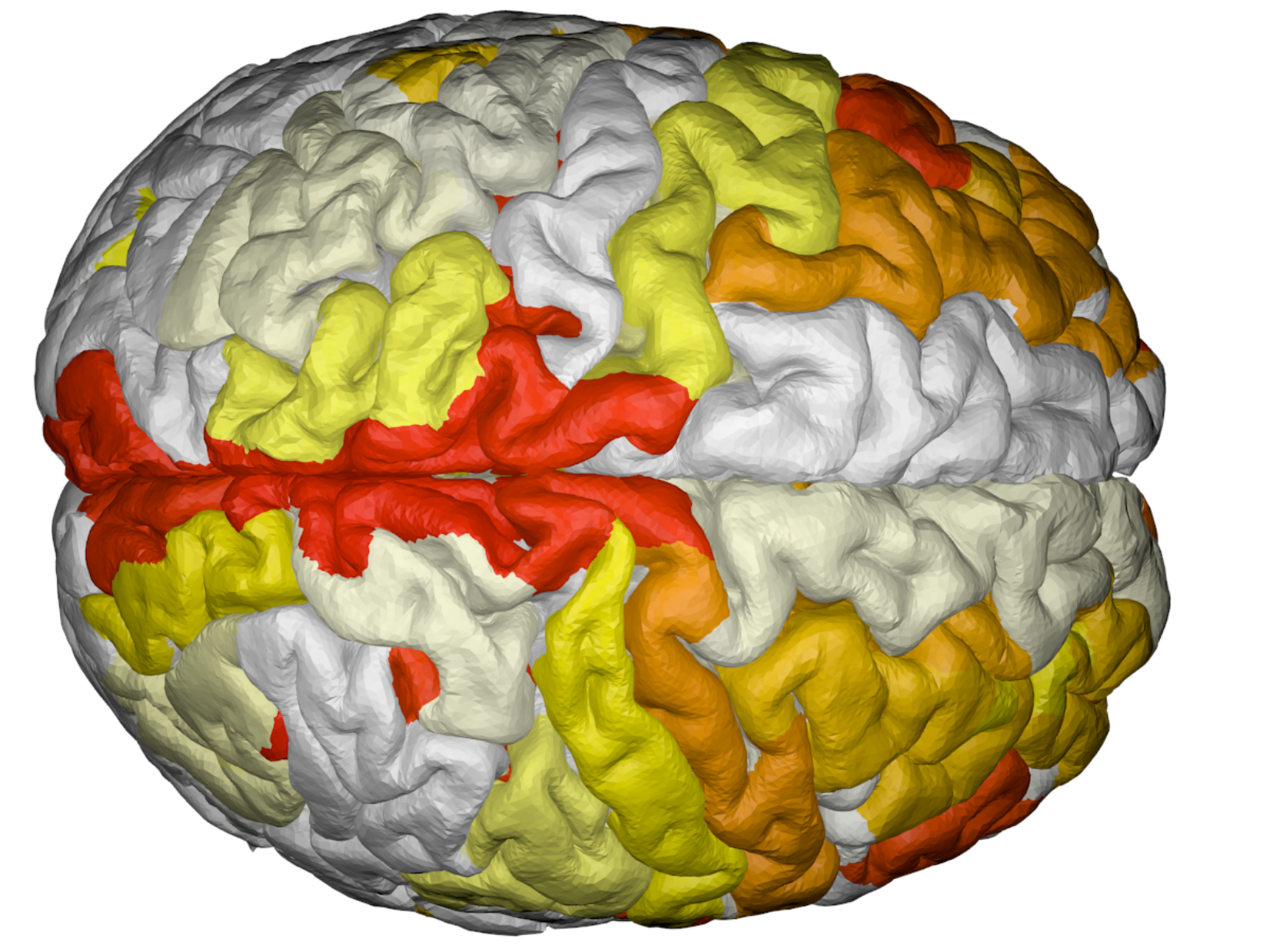} &
    \includegraphics[width=0.13\linewidth]{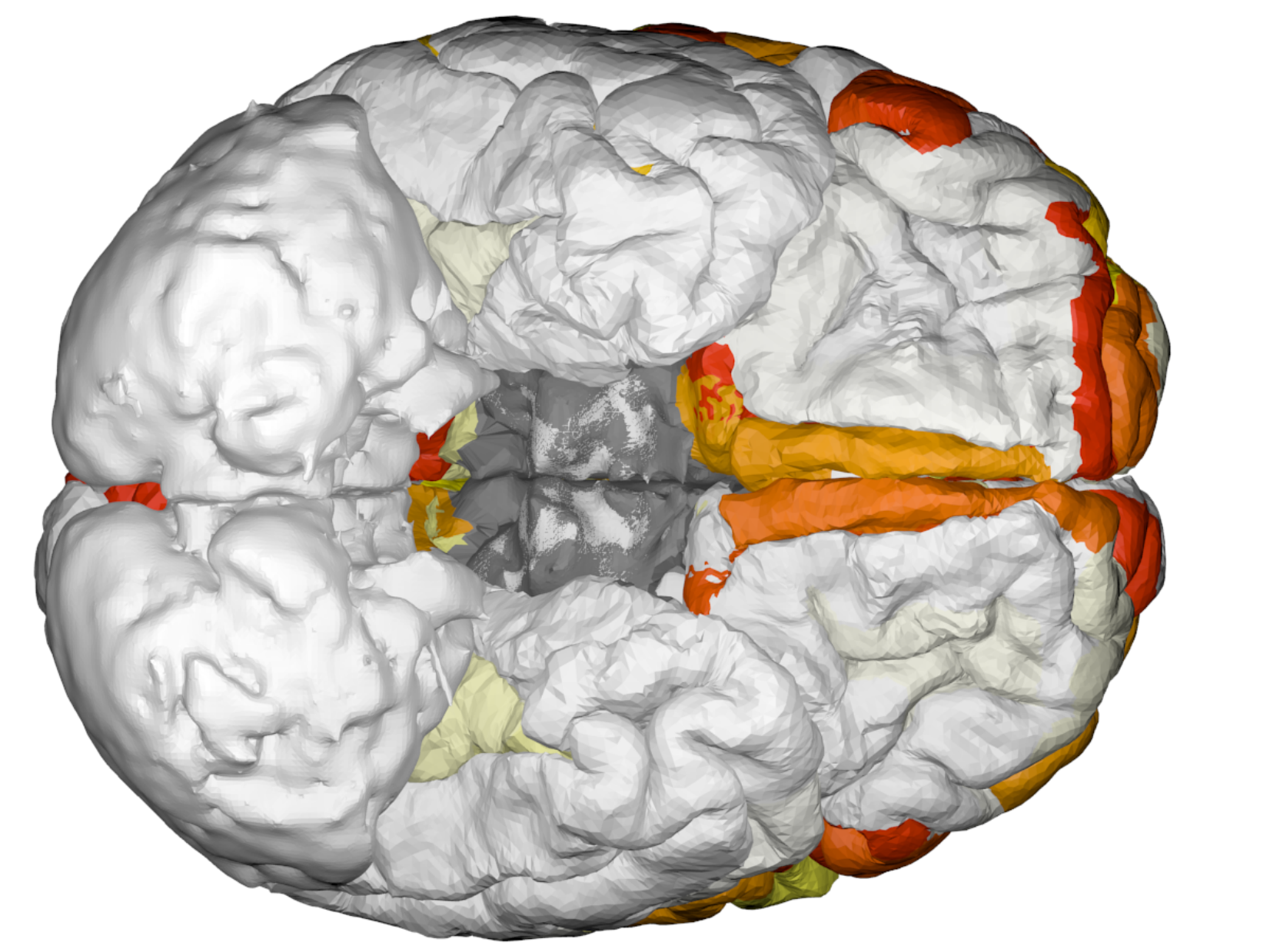} &
    \includegraphics[width=0.13\linewidth]{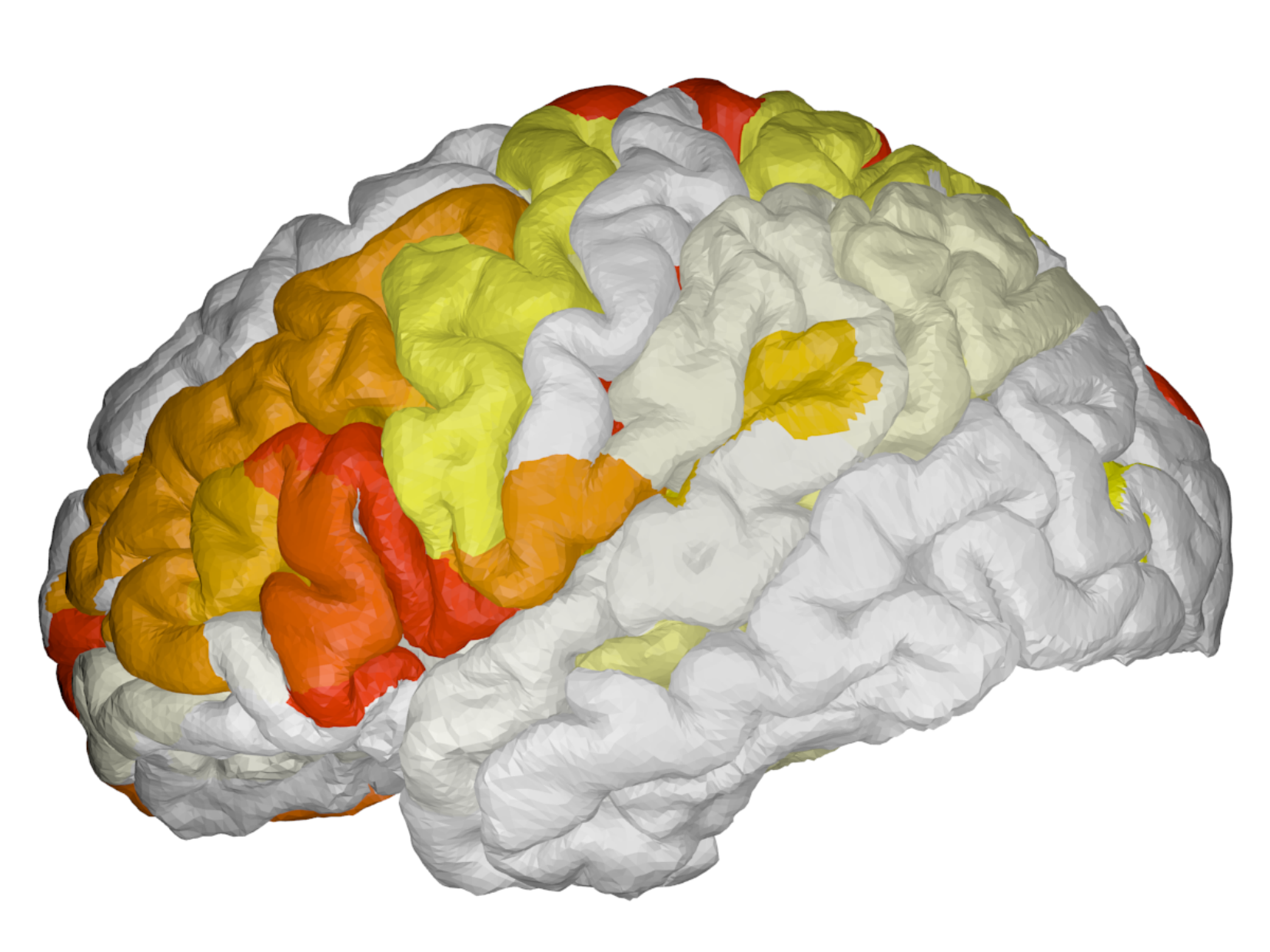} &
    \includegraphics[width=0.13\linewidth]{cortical-outer-right-hemisphere_fd_her.pdf} &
    \includegraphics[width=0.13\linewidth]{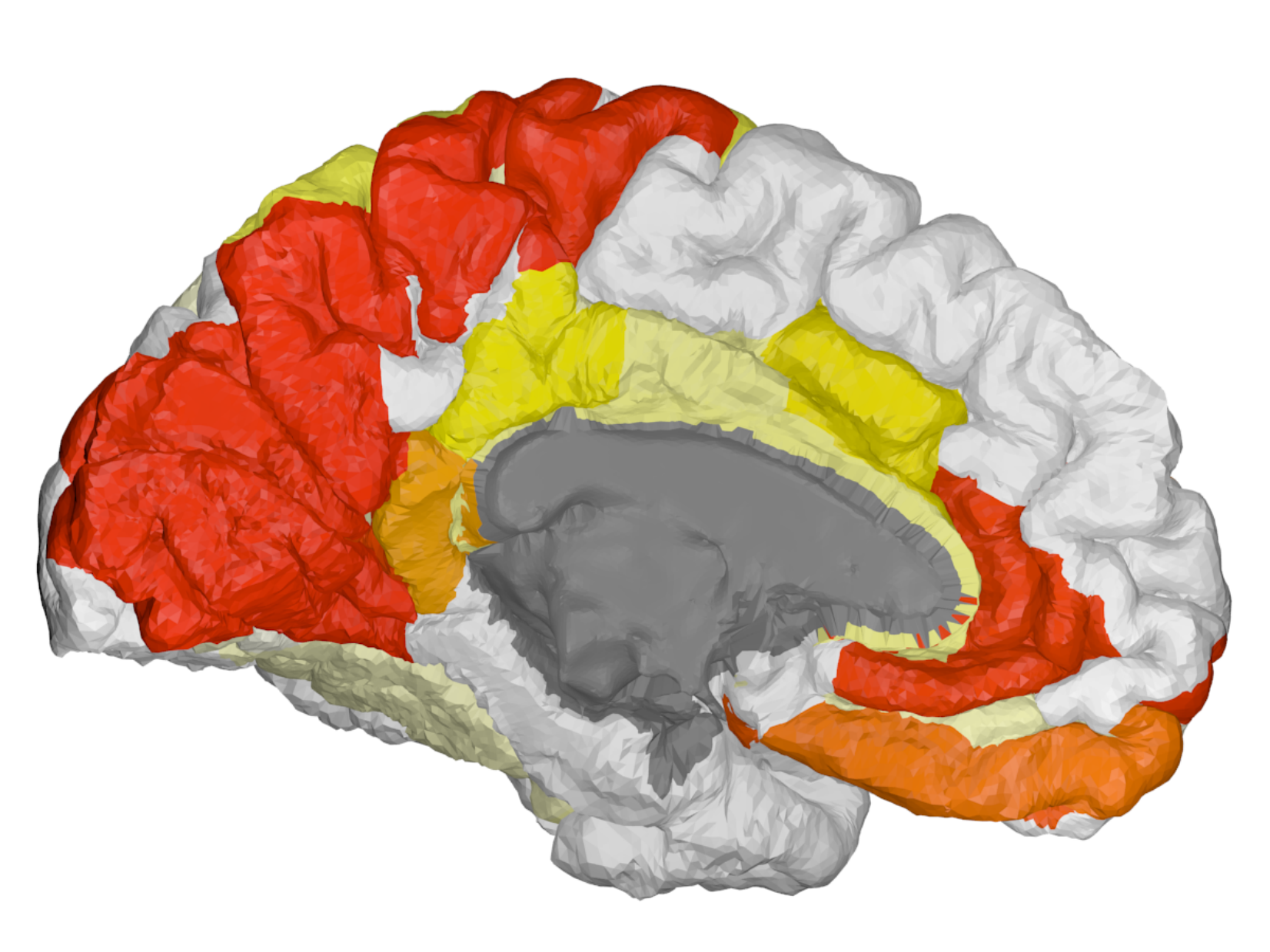} &
    \includegraphics[width=0.13\linewidth]{cortical-inner-right-hemisphere_fd_her.pdf} & 
    \includegraphics[width=0.13\linewidth]{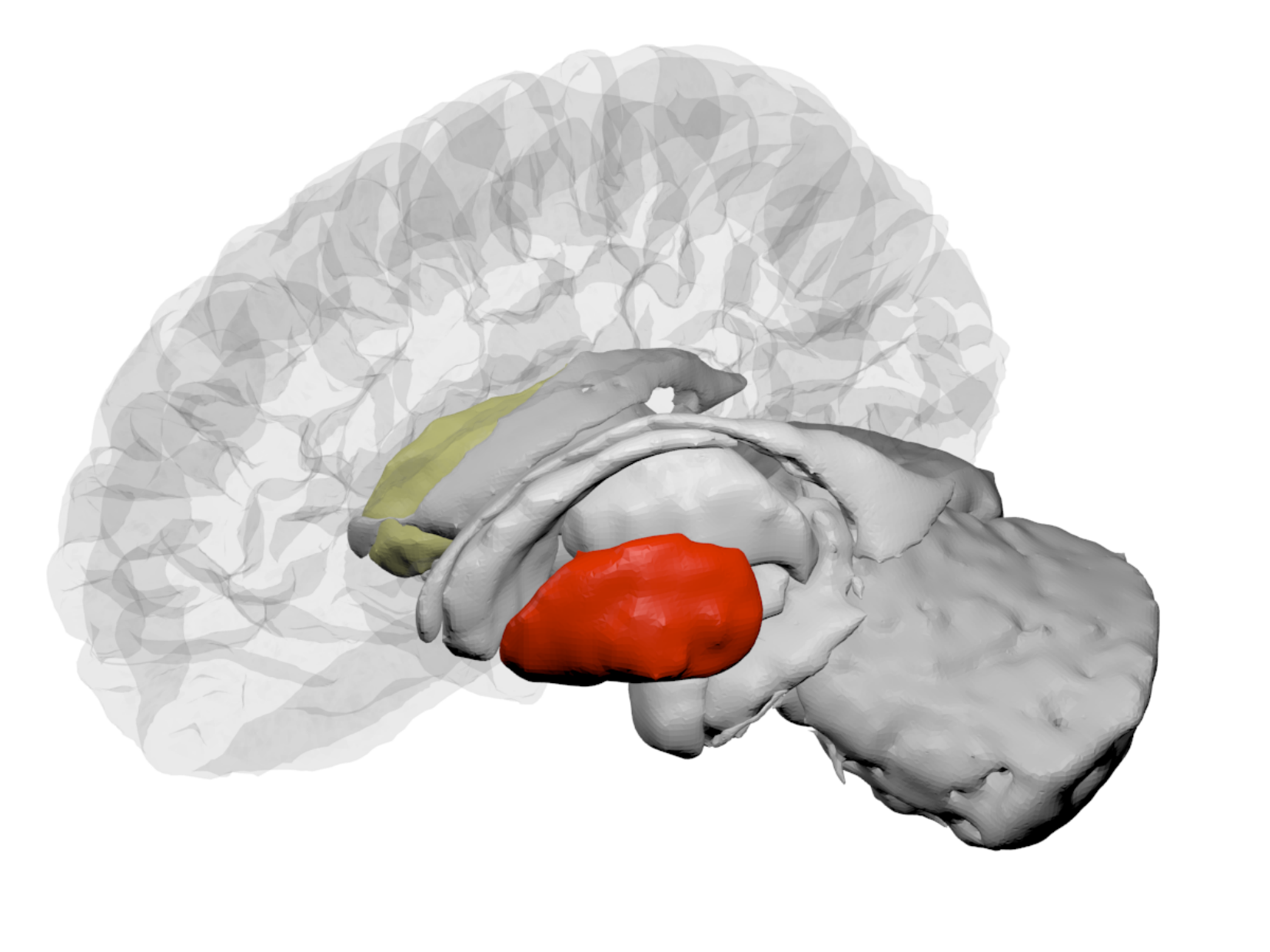} &
    \\ \vspace{-0.2cm}
    \raisebox{4\height}[0pt][0pt]{\textbf{LSAP-L}} &
    \includegraphics[width=0.13\linewidth]{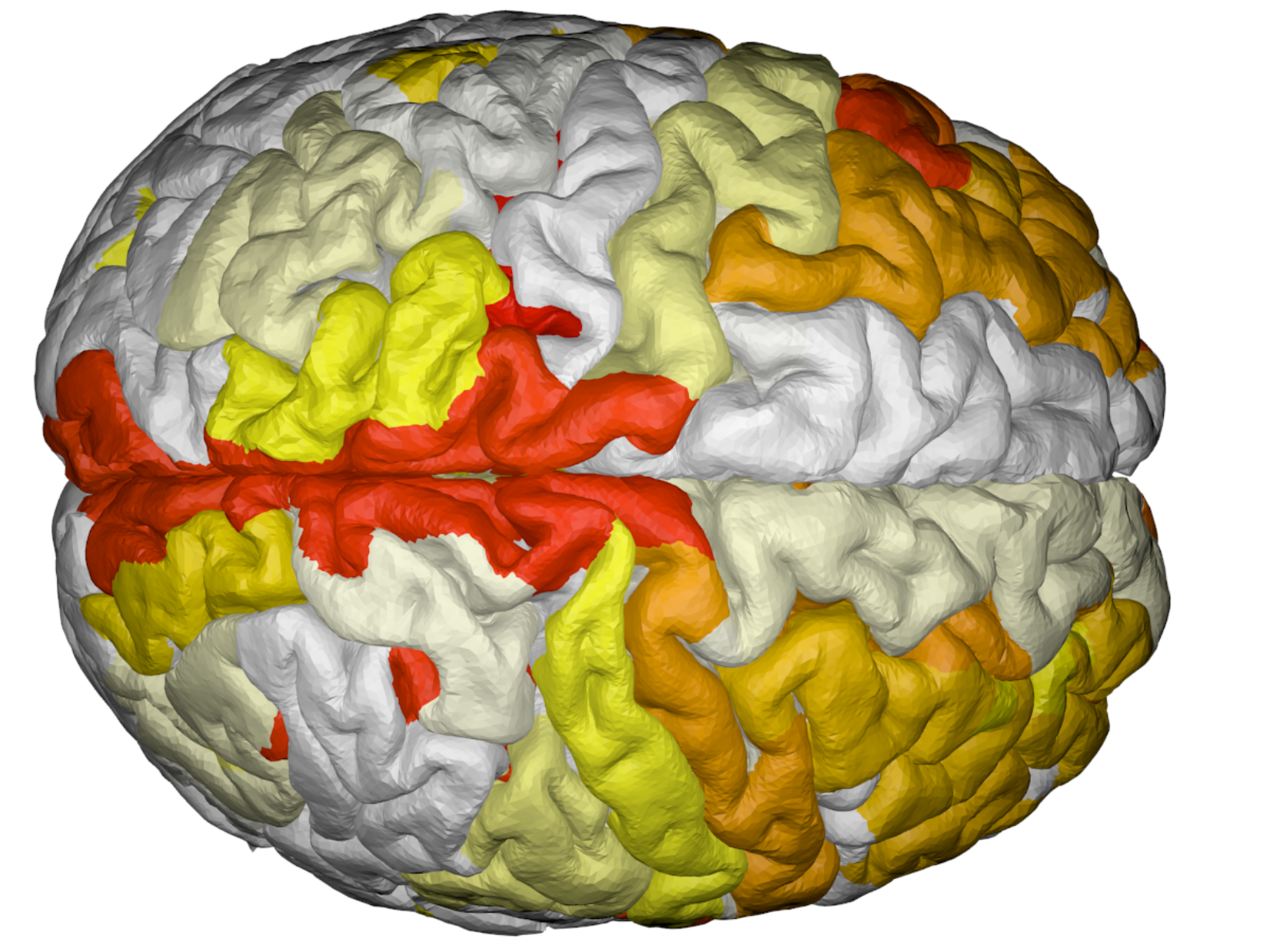} &
    \includegraphics[width=0.13\linewidth]{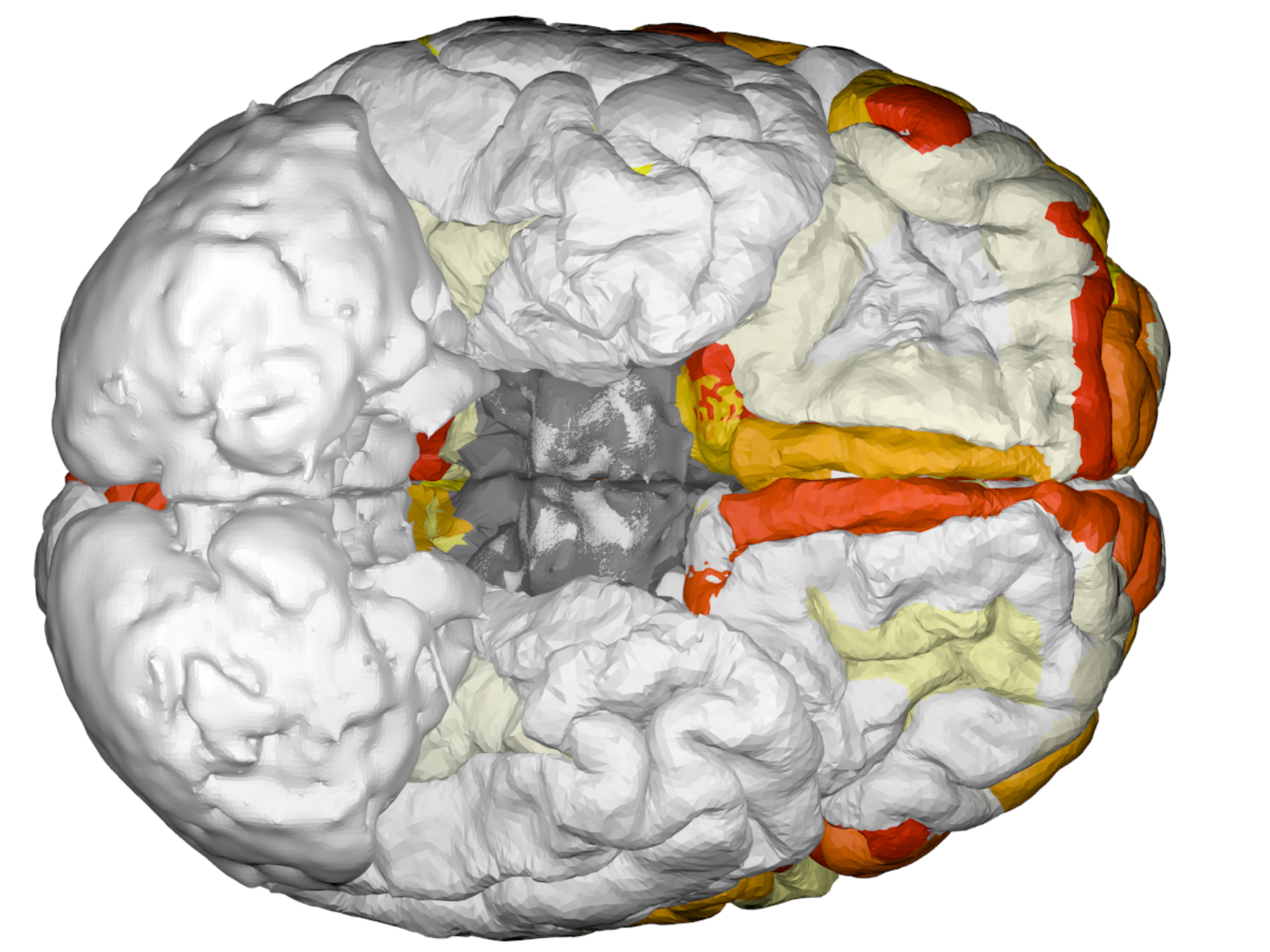} &
    \includegraphics[width=0.13\linewidth]{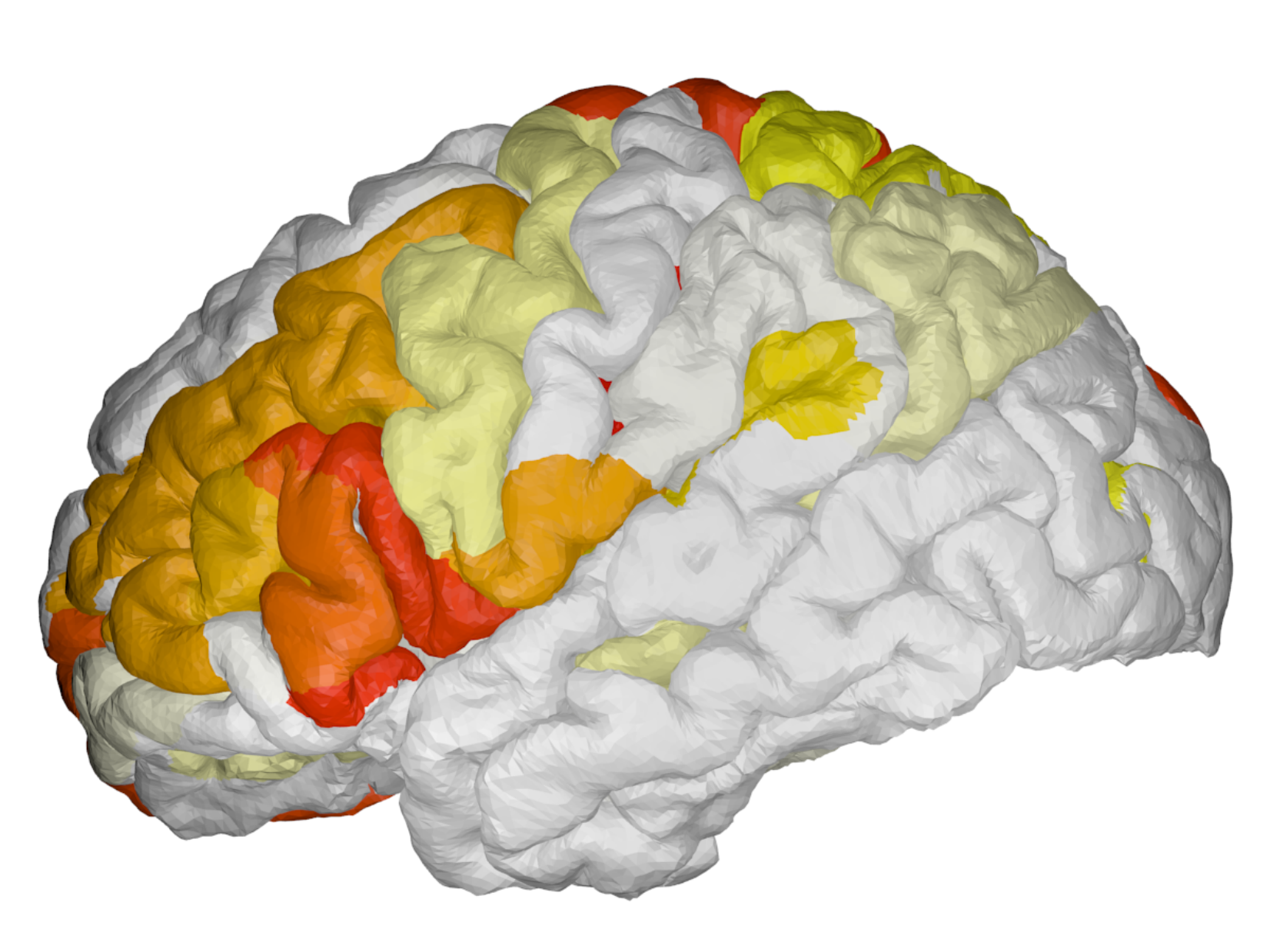} &
    \includegraphics[width=0.13\linewidth]{cortical-outer-right-hemisphere_fd_lag.pdf} &
    \includegraphics[width=0.13\linewidth]{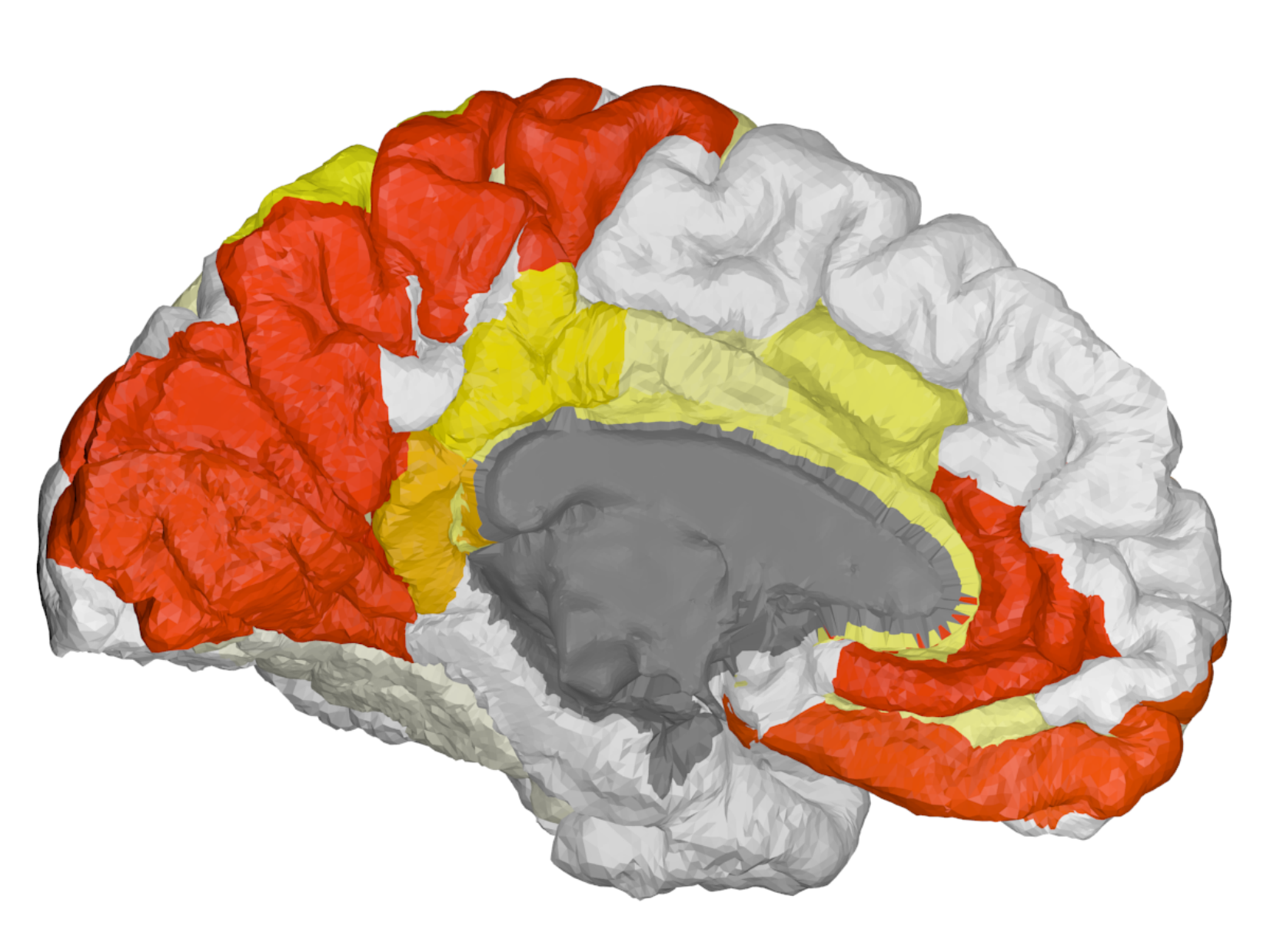} &
    \includegraphics[width=0.13\linewidth]{cortical-inner-right-hemisphere_fd_lag.pdf} & 
    \includegraphics[width=0.13\linewidth]{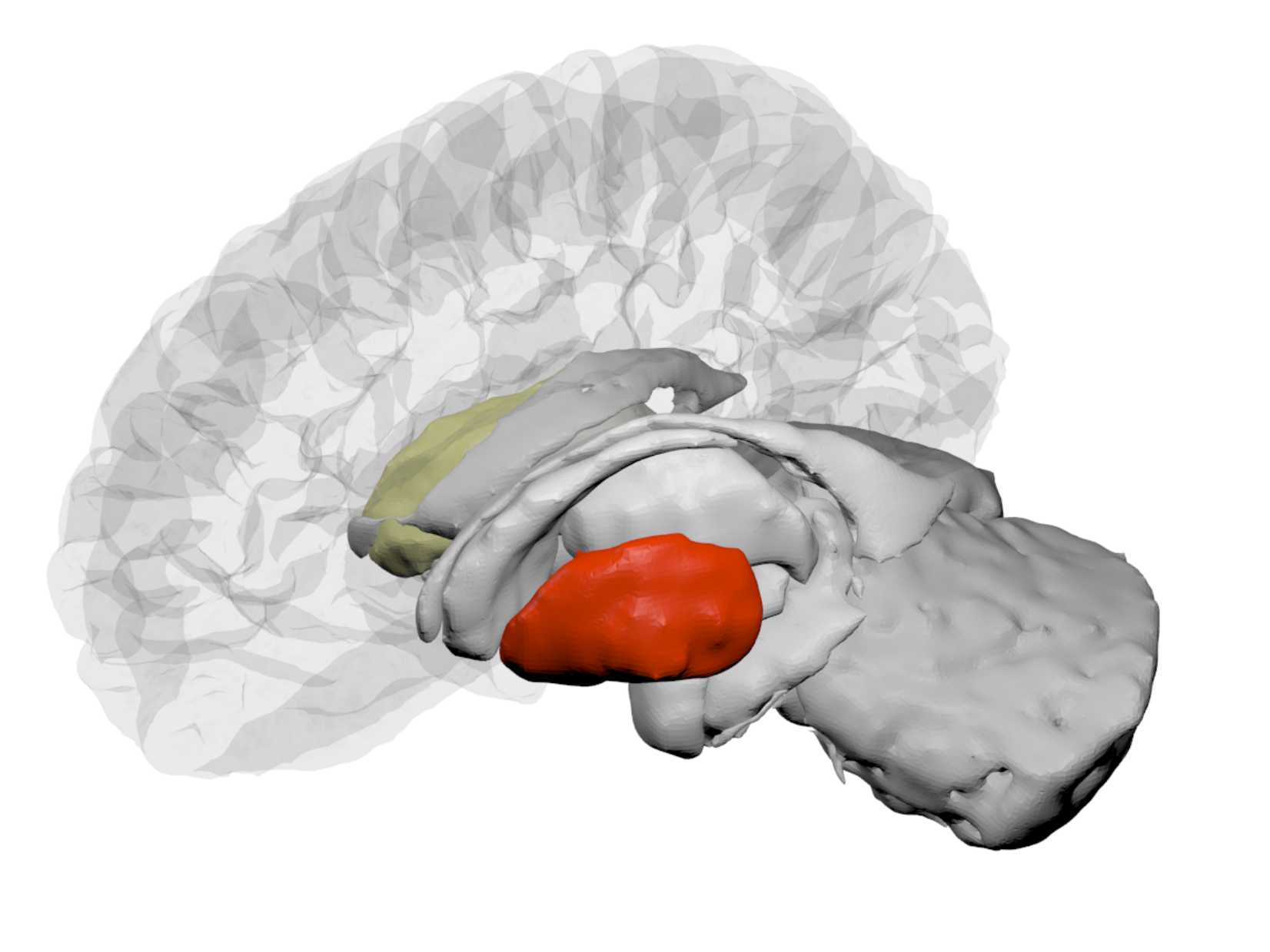} &\\
  \end{tabular}}}
  \vspace{5pt}
  \caption{Visualization of learned scales on the cortical and sub-cortical regions of a brain. 
  This visualization shows the scale of each RoI through the classification result using FDG feature.}
\label{fig:Brain_Visualization_FDG_supple}
\end{figure*}

\begin{table*}[t!]
 \centering
    \caption{ROIs with the smallest trained scales for AD classification. (L) and (R) denote left and right hemisphere.}
    \scalebox{0.6}{
    \renewcommand{\arraystretch}{1.2} 
    \renewcommand{\tabcolsep}{0.2cm}
    {
    \begin{tabular}{c|c||c|c|c|c||c|c||c|c|c|c}
    \Xhline{3\arrayrulewidth}
    \multicolumn{6}{c||}{\textbf{Cortical Thickness}} & \multicolumn{6}{c}{\textbf{FDG}}\\ \hline
    \textbf{\# of common models}&\textbf{ROI} & \textbf{Exact} & \textbf{LSAP-C} & \textbf{LSAP-H} & \textbf{LSAP-L} & \textbf{\# of common models}&\textbf{ROI} & \textbf{Exact} & \textbf{LSAP-C} & \textbf{LSAP-H} & \textbf{LSAP-L}\\
    \hline
    \multirow{7}{*}{\textbf{4}} &(R) S.pericallosal & 0.035 & 0.097 & 0.033 & 0.050 &\multirow{10}{*}{\textbf{4}} & (L) G\&S.paracentral & 0.034 & 0.052 & 0.049 & 0.066 \\ 
    &(R) Lat.Fis.ant.Horizont & 0.048 & 0.074 & 0.039 & 0.037&&(L) G.front.inf.Orbital & 0.036 & 0.071 & 0.060 & 0.043\\ 
    &(L) Lat.Fis.ant.Horizont & 0.049 & 0.083 & 0.038 & 0.051&&(R) G.precuneus & 0.041 & 0.044 & 0.034 & 0.051\\ 
    &(L) G\&S.paracentral & 0.059 & 0.034 & 0.046 & 0.044&&(R) S.ortibal.med.olfact & 0.047 & 0.078 & 0.054 & 0.059\\ 
    &(L) G\&S.occipital.inf & 0.062 & 0.071 & 0.043 & 0.055&&(R) G.cingul.Post.ventral & 0.055 & 0.056 & 0.055 & 0.051\\ 
    &(R) Lat.Fis.ant.Vertical & 0.064 & 0.108 & 0.046 & 0.041&&(R) S.oc.temp.lat& 0.055 & 0.065 & 0.045 & 0.063\\ 
    &(L) G.cingul.Post.ventral & 0.068 & 0.127 & 0.035 & 0.066&&(R) G.oc.temp.med.Lingual& 0.055 & 0.076 & 0.043 & 0.040 \\ \cline{1-6}
    \multirow{6}{*}{\textbf{3}} 
    &(L) S.suborbital & 0.036 & - & 0.054 & 0.060&&(L) Sub.put & 0.058 & 0.077 & 0.047 & 0.060 \\ 
    &(L) S.temporal.inf & 0.045 & 0.129 & 0.039 & 0.057&&(L) S.postcentral & 0.061 & 0.069 & 0.060 & 0.018 \\ 
    &(L) S.occipital.ant & 0.050 & 0.119 & - & 0.064&&(R) G.front.inf.Orbital& 0.063 & 0.093 & 0.069 & 0.050 \\ \cline{7-12}
    &(R) S.precentral.inf.part & 0.054 & 0.052 & - & 0.057 & \multirow{6}{*}{\textbf{3}} & (R) S.intrapariet\&P.trans & 0.046 & 0.095 & - & 0.054  \\ 
    &(R) S.oc.temp.med\&Lingual & 0.055 & - & 0.052 & 0.062&&(L) G.cuneus & 0.049 & - & 0.044 & 0.062\\ 
    &(L) G.temp.sup.G.T.transv & 0.068 & 0.104 & 0.049 & -&&(R) S.temporal.sup & 0.049 & - & 0.028 & 0.056 \\ 
    \cline{1-6}
    \multirow{4}{*}{\textbf{2}} 
    &(R) G.parietal.sup & 0.035 & - & 0.056 & -&&(R) S.calcarine & 0.053 & 0.059 & - & 0.047 \\ 
    &(R) Lat.Fis.post & 0.049 & 0.067 & - & -&&(R) G\&S.paracentral & 0.055 & - & 0.051 & 0.048 \\ 
    &(L) Lat.Fis.ant.Vertical & 0.059 & - & - & 0.040&&(R) G.cuneus & 0.057 & - & 0.057 & 0.066 \\ \cline{7-12}
    &(R) S.suborbital & 0.064 & - & - & 0.048&\multirow{2}{*}{\textbf{2}} 
    &(R) G\&S.cingul.Mid.Ant & 0.062 & 0.041 & - & - \\\cline{1-6}
    &&&&&&&(R) Sub.put & 0.065 & 0.037 & - & - \\
    \hline
    & \textbf{RMSE} for all ROIs & - & \textbf{0.5358} & \textbf{0.4193} & \textbf{0.4072} && \textbf{RMSE} for all ROIs & - & \textbf{0.5827} & \textbf{0.2904} & \textbf{0.2899} \\
    \Xhline{3\arrayrulewidth}
    \end{tabular}}}
\label{tab:RoIs_supple}
\end{table*}

\section{4. Experiments}
\label{sec:exp}
To understand our experiments, 
we describe datasets of our experiments in detail.
We conducted experiments on standard node classification datasets that provide connected and undirected graphs. 
Cora, Citeseer and Pubmed are constructed as citation networks. 
The nodes are papers, the edges are citations of one paper to another, the features are bag-of-words representation of papers, and the labels are specified as an academic topic. 
Amazon Computer and Amazon Photo define co-purchase networks. 
The nodes are goods, 
the edges indicate whether two goods are frequently purchased together, 
the features are bag-of-words representation of product reviews, 
and the labels are categories of goods. 
Coauthor CS \cite{shchur2018pitfalls} is a co-authorship network. 
The nodes are authors, the edges indicate whether two authors co-authored a paper, 
the features are keywords from papers, 
and the labels are each author's field of study.

In the main paper, specifically in Table 3, 
we reported the mean and standard deviation of the results of 10 replicates of the experiment and compared it with the results of {\em 3ference} \cite{luo2022inferring}.
However, since the result of {\em 3ference} is the maximum value among the 10 experimental results, 
we also wrote the maximum value in the supplementary and compared it, which is in Table~\ref{tab:Result_AAC_supple}.
Our model showed the best performance compared to the baselines.

In the main experiments, we compared the performances between LSAP and baselines on node classification and graph classification. 
For node classification, standard benchmarks (Cora, Citeseer and Pubmed) and additional datasets (Amazon Computers, Amazon Photo, and Coauthor CS) were used to validate the performance of LSAP.
For graph classification, the ADNI dataset containing a number of subjects with brain network was used to classify AD-specific labels.
Hyperparameters should be set so that the model can properly learn data in each experiment.
Table \ref{tab:hyperparameter_supple} shows the detailed parameters of LSAP and baselines for the main experiments.

The trained scales on the brain network classification with Exact and LSAP are visualized in Fig. \ref{fig:Brain_Visualization_DT_supple} and Fig. \ref{fig:Brain_Visualization_FDG_supple} containing additional interpretable results visualized from various views.
In Table \ref{tab:RoIs_supple}, based on Exact, smallest scales that appear in common across Exact and LSAPs are listed.
7 ROIs were detected in Exact and LSAPs, 6 ROIs were detected in 3 models, and only 2 ROIs were in 2 models for cortical thickness feature.
Using FDG, 10 ROIs were detected in every model, 6 ROIS were detected in 3 models, and 2 ROIs were detected in 2 models.
In addition, similar ROIs overall showed similar values of scale in all models.

\section{5. Implementation Details}

We stacked $K$$=$$2$ heat kernel convolution layers with rectified linear unit (ReLU) as the activation function and softmax function at the output to predict the node or graph classes.
The initial scale for every node was set to 2. 
The $n$$=$$20$ for LSAP-C and LSAP-L, and $n$$=$$30$ for LSAP-H as they empirically demonstrated the best convergence.
The same number of hidden nodes in $W_k$ was used within each the dataset, drop-out rate was 0.5, and other hyper-parameters such as learning rate for $W_k$ and $s$ were properly tuned to get the best results (given in the supplementary).
For the node classification, we used early stopping to stop training when the validation accuracy stopped increasing. 
For graph classification, the readout function $f_R(\cdot)$ was defined as MLP with 2 layers with 16 hidden nodes (with ReLU), 
and 5-fold cross validation (CV) was used.

\section{6. Model complexity}

Exact method requires full eigendecomposition of Graph Laplacian, which takes $O(N^3)$ where $N$ is the number of nodes.
However, our method does not have this computation and only use graph Laplacian directly to compute polynomial basis, 
which will be used to construct approximated heat kernel with $|{\bf s}| = N$ parameters.
Then, as we basically used GCN framework and redefined the convolution with approximated heat kernel,
time complexity of construction of heat kernel is simply added to time complexity of GCN.

The time complexity of $L$-layer GCN is $O(LND^2)$ where $D$ is dimension of features. 
When we use only first order approximation of heat kernel, time complexity of our process becomes $O(N)$ which is highly marginal compared to the GCN pipeline. 
The memory complexity of GCN is $O(LND+LD^2)$.
Memory complexity of our frameworks takes $O(N^2)$ for construction of polynomial basis, which will be added to memory complexity of GCN, 
which can be further reduced with polynomial approximation. 
At the cost of some time and memory in addition to the GCN, 
we gain significant improvements in both node and graph tasks.

\begin{figure*}[!t]
\centering
\renewcommand{\arraystretch}{1.3}
    \scalebox{0.95}{
    \begin{tabular}{cccccl}
    \includegraphics[width=1.0\linewidth]{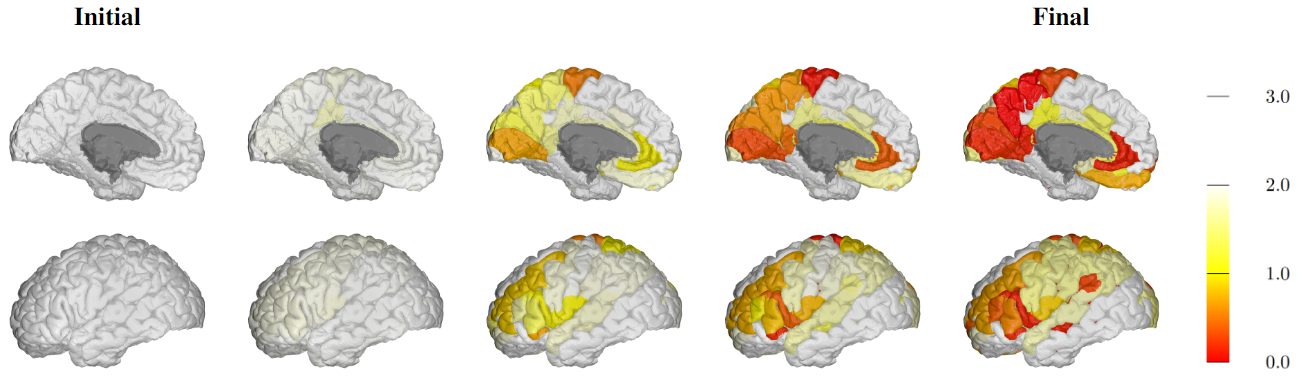}
    \end{tabular}
    }
\caption{Visualization of adjustment of the trained scales on the cortical regions of a brain at training phase using FDG biomarker.
Top: Inner part of left hemisphere, Bottom: Outer part of left hemisphere.
}
\label{fig:adjustment_supple}
\end{figure*}

\section{7. Adjustment of the learned scales}
\label{sec:adjustment}

Fig. \ref{fig:adjustment_supple} displays the outcomes pertaining to the scale of each ROI in the brain during every quarter interval throughout the complete iteration within node-wise method.
In contrast to prior approaches such as GraphHeat \cite{xu2020graph} and ADC \cite{zhao2021adaptive}, which employ a uniform global scale for all regions of interest (ROIs), 
our proposed method looks for multivariate scales for each individual ROI during the training stage.
The initial scales for all ROIs are uniformly set to the same value, 
and subsequently tuned to align with the corresponding data.
As the learning phase progresses from its initial state,
several ROIs associated with AD become spotlighted, which are colored by close to red in Fig. \ref{fig:adjustment_supple}.

\section{8. Discussions}
\label{sec:discussion}

\textbf{Reasons for evaluating LSAP on the ADNI data. }
ADNI dataset was a suitable bench mark for the following reasons: 
1) As the performance of node-wise scale training is sufficiently validated on node-classification with well-curated benchmarks, 
we wanted to run experiments on a real dataset with a practical application for scientific discovery such as Neuroimaging. 
2) Brain has a common structure that can be translated to registered ROIs (i.e., nodes), 
where the results from LSAP can be clinically interpreted and convey interesting results to the readers. 
3) ADNI study provide relatively large sample-size, and we have performed experiments with 5-fold cross validation to ensure unbiased results even though we are using only one dataset. 

\noindent\textbf{Limitations of LSAP. } 
As LSAP needs to localize and update the scale for each node, training takes longer compared to conventional GNN that perform uniform convolution across all nodes, 
but this issue is well-addressed with our training on the scales within approximation. 
Another limitation is that LSAP operates with fixed set of nodes for graph classification. 
However, such a setting has an advantage with LSAP as we learn node-wise scale which yield 
interpretable results that may lead to scientific discoveries.

\end{document}